\documentclass[preprint,12pt]{elsarticle}

\usepackage{epstopdf}
\usepackage{fullpage}
\usepackage{graphicx}
\usepackage{caption}

\usepackage{amsmath}
\usepackage{amsthm}
\usepackage{amssymb}

\usepackage{bbm}
\usepackage{verbatim}
\usepackage{graphicx}
\usepackage{subfigure}
\usepackage{afterpage}
\usepackage{etoolbox}
\usepackage{float}
\usepackage{rotating}
\usepackage{algorithmic}
\usepackage[algosection,algoruled,vlined]{algorithm2e}

\numberwithin{equation}{section}
\numberwithin{figure}{section}
\numberwithin{table}{section}

\makeatletter
\renewcommand{\p@subfigure}{\thefigure}
\makeatother

\newtheorem{theorem}{Theorem}[section]

\newcommand{\repeatable}[2]{\makeatletter \global\expandafter\def\csname repText@#1\endcsname {#2} \makeatother #2}
\newcommand{\repeatxt}[1]{\makeatletter \expandafter\csname repText@#1\endcsname \makeatother}

\newtoggle{citeInAbstract}
\toggletrue{citeInAbstract}

\newcommand{\usecrop}[2]
{
	\newlength{\cropwidth}
	\setlength{\cropwidth}{\the\textwidth}
	\addtolength{\cropwidth}{#1}
	\newlength{\cropheight}
	\setlength{\cropheight}{\the\textheight}
	\addtolength{\cropheight}{#2}
	\usepackage[width=\the\cropwidth,height=\the\cropheight,center]{crop}
}

\DeclareMathAlphabet{\mathpzc}{OT1}{pzc}{m}{it}

\DeclareMathOperator{\Vol}{Vol}

\newcommand{\Rn}[1]{{\mathbb{R}^{#1}}}

\newcommand {\myvec}[1] {{\mbox{\boldmath $#1$}}}
\newcommand {\mymat}[1]  {{\mbox{\boldmath $#1$}}}
\newtheorem{prop}{Proposition}[subsection]
\newtheorem{cor}{Corollary}[subsection]
\newtheorem{D1}{Defenition}[subsection]
\newtheorem{asump}{Assumption}

\newcommand {\defeq}{\triangleq}

\begin{document}
\begin{frontmatter}

\title{Gaussian Bandwidth Selection for Manifold Learning and Classification}
\author{Ofir~Lindenbaum\fnref{label1}}
\ead{ofirlin@gmail.com}
\author{      Moshe~Salhov\fnref{label2}}
\author{       Arie~Yeredor\fnref{label1}}
\author{      Amir~Averbuch\fnref{label2}}
\address[label1]{School of Electrical Engineering, Tel Aviv University, Israel}
\address[label2]{School of Computer Science, Tel Aviv University, Israel}


\begin{abstract}
Kernel methods play a critical role in many machine learning algorithms. They are useful in manifold learning, classification, clustering and other data analysis tasks. Setting the kernel's scale parameter, also referred to as the kernel's bandwidth, highly affects the performance of the task in hand. We propose to set a scale parameter that is tailored to one of two types of tasks: classification and manifold learning. For manifold learning, we seek a scale which is best at capturing the manifold's intrinsic dimension. For classification, we propose three methods for estimating the scale, which optimize the classification results in different senses. The proposed frameworks are simulated on artificial and on real datasets. The results show a high correlation between optimal classification rates and the estimated scales. Finally, we demonstrate the approach on a seismic event classification task.  
\end{abstract}


\end{frontmatter}

\section{Introduction}
Dimensionality reduction is an essential step in numerous machine learning tasks. Methods such as Principal
Component Analysis (PCA) \cite{jolliffe2002principal}, Multidimensional Scaling (MDS)
\cite{MDS}, Isomap~\cite{tenenbaum:Isomap} and Local Linear Embedding \cite{LLE} aim to extract essential information from high-dimensional data points based on their pairwise connectivities. Graph-based kernel methods such as Laplacian Eigenmaps \cite{belkin2001laplacian} and Diffusion Maps (DM) \cite{Lafon}, construct a positive semi-definite kernel based on the multidimensional data points to recover the underlying structure of the data. Such methods have been proven effective for tasks such as clustering \cite{Luo}, classification \cite{lindenbaum2015musical}, manifold learning \cite{lin2006riemannian} and many more.

Kernel methods rely on computing a distance function (usually Euclidean) between all pairs of data points $\myvec{x}_i,\myvec{x}_j\in \myvec{X} \in \mathbb{R}^{D \times N}$ and application of a data dependent kernel function. This kernel should encode the inherited relations between high-dimensional data points. An example for a kernel that encapsulates the Euclidean distance takes the form  
\begin{equation}\label{eq:KernelK1}
{\cal{K}}(\myvec{x}_i,\myvec{x}_j)\defeq{\cal{K}}\left( \frac{||\myvec{x}_i-\myvec{x}_j||^2}{\epsilon} \right)=K_{i,j},
\end{equation}
(where $\|\cdot\|$ denotes the Euclidean norm).
As shown, for example, in \cite{LLE, Lafon}, spectral analysis of such a kernel provides an efficient representation of the lower ($d$-) dimensional data (where $d\ll D$) embedded in the ambient space. Devising the kernel to work successfully in such contexts requires expert knowledge for setting two parameters, namely the scaling $\epsilon$ (Eq.\ (\ref{eq:KernelK1})) and the inferred dimension $d$ of the low-dimensional space. We focus in this paper on setting the scale parameter $\epsilon$, sometimes also called the {\em kernel bandwidth}. 

The scale parameter is related to the statistics and to the geometry of the data points. The Euclidean distance, which is often used for learning the geometry of the data, is meaningful only locally when applied to high-dimensional data points. Therefore, a proper choice of \begin{math} \epsilon \end{math} should preserve local connectivities and neglect large distances. If \begin{math} \epsilon \end{math} is too large, there is almost no preference for local connections and the kernel method is reduced essentially to PCA \cite{lindenbaum2015musical}. On the other hand, if $\epsilon$ is too small, the matrix \begin{math} \mymat{K} \end{math} (Eq.\ (\ref{eq:KernelK1})) has many small off-diagonal elements, which is an indication of a poor connectivity within the data.

Several studies have proposed different approaches for setting $\epsilon$. A study by Lafon {\em et al.} \cite{Coifman} suggests a method which enforces connectivity among most data points - a rather simple method, which is nonetheless sensitive to noise and to outliers. The sum of the kernel is used by Singer {\em et al.} \cite{Singer} to find a range of valid scales, this method provides a good starting value but is not fully automated. An approach by Zelnik-Manor and Perona \cite{zelnik} sets an adaptive scale for each point, which is applicable for spectral clustering but might deform the geometry. As a result, there is no guarantee that the rescaled kernel has real eigenvectors and eigenvalues. Others simply use the squared standard deviation (mean squared Euclidean distances from the mean) of the data as $\epsilon$, this again is very sensitive to noise and to outliers. 

Kernel methods are also used for Support Vector Machines (SVMs, \cite{SVM}), where the goal is to find a feature space that best separates given classes. Methods such as \cite{gaspar2012parameter,staelin2003parameter} use cross-validation to find the scale parameter which achieves peak classification results on a given training set. The study by Campbell {\em et al.} \cite{campbell1999dynamically} suggests an iterative approach that updates the scale until reaching maximal separation between classes. Chapelle {\em et al.} \cite{chapelle2002choosing} relate the scale parameter to the feature selection problem by using a different scale for each feature. This framework applies gradient descent to a designated error function to find the optimal scales. These methods are good for classification, but require to actually re-classify the points for testing each scale. 

In this paper we propose methods to estimate the scale which do not require the repeated application of a classifier to the data. Since the value of $\epsilon$ defines the connectivity of the resulted kernel matrix (Eq.\ (\ref{eq:KernelK1})), its value is clearly crucial for the performance of kernel based methods. Nonetheless, the performance of such methods depends on the training data and on the optimization problem in hand. Thus, in principle we cannot define an 'optimal' scaling parameter value independently of the data. We therefore focus on developing tools to estimate a scale parameter based on a given training set. We found that there are almost no simple methods focusing on finding element-wise (rather than one global) scaling parameters dedicated for manifold learning. Neither are there methods that try to maximize classification performance without directly applying a classifier. For these reasons we propose new methodologies for setting  $\epsilon$, dedicated either to manifold learning or to classification.

For the manifold learning task, we start by estimating the manifold's intrinsic dimension. Then, we introduce a vector of scaling parameters $\myvec{\epsilon}=[\epsilon_1,...,\epsilon_D]$, such that each value $\epsilon_i,i=1,...,D$, rescales each feature. We propose a special greedy algorithm to find the scaling parameters which best capture the estimated intrinsic dimension. This approach is analyzed and simulated to demonstrate its advantage.

For the classification task, we propose three methods for finding a scale parameter. In the first, by extending \cite{lindenbaum2016bandwidth} we seek a scale which provides the maximal separation between the classes in the extracted low-dimensional space. The second is based on the eigengap of the kernel. It is justified based on the analysis of a perturbed kernel. The third method sets the scale which maximizes the within-class transition probability. This approach does not require to compute an eigendecomposition. 

Additionally, we provide new theoretical justifications for the eigengap-based method, as well as new simulations to support all methods. Interestingly, we also show empirically that all the three methods converge to a similar scale parameter $\epsilon$. 

The structure of the paper is as follows: Preliminaries are given in
section \ref{sec:Background}. Section \ref{sec:EpsEst} presents and analyzes two frameworks for setting the scale parameter: the first is dedicated to a manifold learning task while the second fits a classification task. Section
\ref{sec:Exp} presents experimental results. Finally, in section \ref{sec:application} we demonstrate the applicability of the proposed methods for the task of learning seismic parameters from raw seismic signals.

\section{Preliminaries}
\label{sec:Background} 

We begin by providing a brief description of two methods used in this study: A kernel-based method for dimensionality reduction called {\em Diffusion Maps} \cite{Lafon}; and {\em Dimensionality from Angle and Norm Concentration} (DANCo, \cite{danco}), which estimates the intrinsic dimension of a manifold based on the ambient high-dimensional data.
In the following, vectors and matrices are denoted by bold letters, and their components are denoted by the respective plain letters, indexed using subscripts or parentheses.

\subsection{Diffusion Maps (DM)}
\label{SecDiff} DM \cite{Lafon} is a nonlinear dimensionality reduction framework that extracts the intrinsic geometry from a high-dimensional dataset. This framework is based on the construction of a stochastic matrix from the graph of the data. The eigendecomposition of the stochastic matrix provides an efficient representation of the data.
Given a high-dimensional dataset \begin{math} \myvec{X} \in \mathbb{R}^{D \times N}\end{math}, the DM framework construction consists of the following steps:
\begin{enumerate}

\item A kernel function  \begin{math} {{\cal{K}} : \myvec{X}\times{\myvec{X}}\longrightarrow{\mathbb{R}}  }
\end{math} is chosen, so as to compute a matrix $\myvec{K} \in {\mathbb{R}^{N \times N}}$ with elements $K_{i,j}={\cal{K}}(\myvec{x}_i,\myvec{x}_j)$, satisfying the following properties: (i) Symmetry: $\mymat{K}=\mymat{K}^T$; (ii) Positive semi-definiteness: $\mymat{K}\succeq\mymat{0}$, namely $\myvec{v}^T  \myvec{K}  \myvec{v} \geq 0$ for all $\myvec{v} \in
\mathbb{R}^N$; and (iii) Non-negativity: $\mymat{K}\ge\mymat{0}$, namely $K_{i,j}\ge 0$ $\forall i,j\in\{1\ldots N\}$. These properties guarantee that $\myvec{K}$ has real-valued eigenvectors and non-negative real-valued eigenvalues. 

In this study, we focus on the common choice of a {\em Gaussian kernel} (see Eq.\ (\ref{eq:KernelK1}))
	\begin{equation}\label{GKernel}
	{\cal{K}}(\myvec{x}_i,\myvec{x}_j)\defeq K_{i,j}=\exp\left( {-\frac{||\myvec{x}_i-\myvec{x}_j||^2}{2
			\epsilon}  }\right),i,j\in\{1\ldots N\},
	\end{equation} as the affinity measure between
two multidimensional data vectors $\myvec{x}_i$ and $\myvec{x}_j$;

Obviously, choosing the kernel function entails the selection of an appropriate scale $\epsilon$, which determines the degrees of connectivities expressed by the kernel.
\item By normalizing the rows of $\mymat{K}$, the row-stochastic matrix
    \begin{equation}
    \label{eq:DefP}
    \mymat{P}\defeq\mymat{D}^{-1}\mymat{K}\in\mathbb{R}^{N\times N}
    \end{equation}
    is computed, where $\mymat{D}\in\mathbb{R}^{N\times N}$ is a diagonal matrix with $D_{i,i}=\sum_jK_{i,j}$. $\mymat{P}$ can be interpreted as the transition probabilities of a (fictitious) Markov chain on $\mymat{X}$, such that $\left[(\mymat{P})^t\right]_{i,j}\defeq p_t(\myvec{x}_i,\myvec{x}_j)$ (where $t$ is an integer power) describes the implied probability of transition from point $\myvec{x}_i$ to point $\myvec{x}_j$ in $t$ steps.
\item{Spectral decomposition is applied to $\mymat{P}$, yielding a set of $N$ eigenvalues \begin{math}{\lbrace {\lambda_n}\rbrace }
	\end{math} \\(in descending order) and associated normalized eigenvectors \begin{math}{\lbrace{{{\myvec{\psi}}}_n}\rbrace }
	\end{math}\\ satisfying ${ {\myvec{P}}  {{\myvec{\psi}}_n} =\lambda_n{{\myvec{\psi}}}_n, n\in\{0\ldots N-1\}}
	$; }
\item{A new representation for the dataset $\myvec{X}$ is defined by
	\begin{equation}\label{eq:Psi}{ \myvec{\Psi}_{\epsilon}{(\myvec{x}_i)}:   \myvec{x}_i
		\longmapsto \begin{bmatrix} { \lambda_1\psi_1(i)} , {
			\lambda_2\psi_2(i)} , { \lambda_3\psi_3(i)} , {.} {.} {.}
		,
		
		{\lambda_{N-1}\psi_{N-1}(i)}\\
		
		\end{bmatrix}^T \in{\mathbb{R}^{N-1}} },
	\end{equation}
	where $\epsilon$ is the scale parameter of the Gaussian kernel (Eq.\  (\ref{GKernel})) and $\psi_m(i)$ denotes the $i^{\rm{th}}$ element of ${{\myvec{\psi}}_m}$. Note that $\lambda_0=1$ and $\myvec{\psi}_0=\myvec{1}$ were excluded as the constant eigenvector $\myvec{\psi}_0=\myvec{1}$ doesn't carry information about the data.
	
	{The main idea behind this representation is that the Euclidean distance between two multidimensional data points in the new representation is equal to the weighted $L_2$ distance between the conditional probabilities ${{p}(\myvec{x}_i,:)}$ and ${{p}(\myvec{x}_j,:)}$, $i,j=1,...,N$, where $i$ and $j$ are the $i$-th and $j$-th rows of $\myvec{P}$.
		The diffusion distance is defined by
		\begin{equation}{ \label{EqDist} { {\cal{D}}^2_{\epsilon}( {x}_i,{x}_j)=||{{\Psi}_{\epsilon}{({x}_i)}}-{{\Psi}_{\epsilon}{({x}_j)}}||^2={  \sum_{m\geq{1}} {\lambda}_m (\psi_m(i)-\psi_m(j))^2 }}=||{p}({x}_i,:)-{p}({x}_j,:)||^2_{\tiny{W}^{-1}}},
		\end{equation}
		where ${W}$ is a diagonal matrix with
		$W_{i,i}=\frac{D_{i,i}}{\sum_{i=1}^M D_{i,i}}$. This
		equality is proved in \cite{Lafon}.}}
\item{A low-dimensional mapping $\myvec{\Psi}^{d}_{\epsilon}(\myvec{x}_i),i=1,...,N$ is defined by \\ \begin{equation} \label{eq:PsiLow} {\myvec{\Psi}^{d}_{\epsilon}(\myvec{x}_i) :  X \rightarrow \begin{bmatrix}
		{ \lambda_1\psi_1(i)} , { \lambda_2\psi_2(i)} , {
			\lambda_3\psi_3(i)} , {.} {.} {.}   ,
		
		{\lambda_{d}\psi_{d}(i)}\\
		
		\end{bmatrix}^T \in \mathbb{R}^{d}}, \end{equation}  such that $d\ll D$, where $\lambda_{d+1},...,\lambda_{N-1}\longrightarrow 0$. }
\end{enumerate}
We chose to use DM in our analysis as it provides an intuitive interpretation based on its Markovian construction. Nonetheless, the methods in this manuscript could also be adapted to Laplacian Eigenmaps \cite{belkin2001laplacian} and to other kernel methods.
\subsection{Intrinsic Dimension Estimation}
\label{sec:ID}
Given a high-dimensional dataset $\myvec{X}=\{\myvec{x}_1,\myvec{x}_2,...,\myvec{x}_N\} \subseteq \mathbb{R}^{D \times N} $, which describes an ambient space with a manifold $\cal{M}$ containing the data points $\myvec{x}_1,\myvec{x}_2,...,\myvec{x}_N$, the {\em intrinsic dimension} $\bar{d}$ is the minimum number of parameters needed to represent the manifold. 

\begin{D1}
\label{def:ID}
Let $\cal{M}$ be a manifold. The intrinsic dimension $\bar{d}$ of the manifold is a positive integer determined by how many independent ``coordinates'' are needed to describe $\cal{M}$. By using a parametrization to describe the manifold, the intrinsic dimension is the smallest integer $\bar{d}$ such that there exists a smooth map $f({\xi})$ for all data points on the manifold ${\cal{M}} ={f}({\xi})$, ${\xi} \subseteq {\cal{R}}^{\bar{d}}$.\\
\end{D1} 
Methods proposed by Fukunaga \& Olsen \cite{fukunaga1971algorithm} or by Verveer \& Duin \cite{verveer1995evaluation} use local or global PCA to estimate the intrinsic dimension $\bar{d}$. The dimension is set as the number of eigenvalues greater than some threshold. Others, such as Trunk \cite{trunk1976stastical} or Pettis {\em et al.} \cite{pettis1979intrinsic}, use $k$-neaserst-neighbors (k-NN) distances to find a subspace around each point and based on some statistical assumption estimate $\bar{d}$. A survey of different approaches is provided in \cite{camastra2003data}. In this study we use Dimensionality from Angle and Norm Concentration (DANCo, by Ceruti {\em et al.} \cite{danco}) based algorithm (which we observed to be the most robust approach in our experiments) to estimate $\bar{d}$.

DANCo jointly uses the normalized distances and mutual angles to extract a robust estimate of $\bar{d}$. This is done by finding the dimension that minimizes the Kullback–Leibler divergence between the estimated probability distribution functions (pdf-s) of artificially-generated data and the observed data. A full description of DANCo is presented in the Appendix of this manuscript. In section \ref{sec:OptMani}, we propose a framework which exploits the resulting estimate $\hat{d}$ of $\bar{d}$ for choosing the scale parameter $\epsilon$.

\section{Setting the Scale Parameter $\epsilon$}
\label{sec:EpsEst}
DM as described in Section \ref{sec:Background} is an efficient method for dimensionality reduction. The method is almost completely automated, and does not require tuning many hyper-parameters. Nonetheless, its performance is highly dependent on proper choice of $\epsilon$ (Eq.\ (\ref{eq:KernelK1})), which, along with the decaying property of Gaussian affinity kernel $\myvec{K}$, defines the affinity between all points in $\myvec{X}\in \mathbb{R}^{D \times N} $. If the ambient dimension $D$ is high, the Euclidean distance becomes meaningless once it takes large values. Thus, a proper choice of $\epsilon$ should preserve local connectivities and neglect large distances. We argue that there is no one 'optimal' way for setting the scale parameter; rather, one should define the scale based on the data and on the task in hand.

In the following subsection we describe several existing method for setting $\epsilon$. In subsection \ref{sec:OptMani}, we propose a novel algorithm for setting $\epsilon$ in the context of manifold learning. Finally in subsection \ref{sec:OptClass}, we present three methods for setting $\epsilon$ in the context of classification tasks, so as to optimize the classification performance (in certain senses) in the low-dimensional space. Our goal is to optimize the scale prior to the application of the classifier.
\subsection{Existing Methods}
\label{sec:exist}
Several studies propose methods for setting the scale parameter $\epsilon$. Some choose \begin{math} \epsilon \end{math} as the empirical squared standard deviation (mean squared Euclidean deviations from the empirical mean) of the data. This approach is reasonable when the data is sampled from a uniform distribution. 

A max-min measure is suggested in \cite{Keller} where the scale is set to
\begin{equation} \label{eq:MaxMin}
\epsilon_{\text{MaxMin}}={\cal{C}}\cdot \underset{j}{\max} [ \underset{i,i\neq j}{\min} (||\myvec{x}_i-\myvec{x}_j||^2)],i,j=1,...N,
\end{equation}
and ${\cal{C}} \in [2,3]$. This approach attempts to set a small scale to maintain local connectivities.

Another scheme \cite{Singer} aims to find a range of values for $\epsilon$. The idea is to compute the kernel $\mymat{K}$ from Eq.\ (\ref{GKernel}) at various values of $\epsilon$. Then, search for the range of values which give rise to a well-pronounced Gaussian bell shape. The scheme in \cite{Singer} is implemented using Algorithm \ref{alg:Singer}.\\
\begin{algorithm}[H] 
\caption{$\epsilon$ range selection} \textbf{Input:} dataset
$\myvec{X} = \{\myvec{x}_1, \myvec{x}_2, \ldots, \myvec{x}_N\}, \myvec{x}_i \in \mathbb{R}^{D}$.\\ \textbf{Output:} Range of values for the scale $\epsilon$, $\bar{\epsilon}=[\epsilon_0,\epsilon_1]$.
\begin{algorithmic}[1]
	\STATE Compute Gaussian kernels        \begin{math}{\myvec{K}(\epsilon) }
	\end{math}      for several values of \begin{math} {\epsilon} \end{math}.
	\STATE Compute:    \begin{math}{L(\epsilon)=\underset{i}{{\sum}}\underset{j}{{\sum}}K_{i,j}(\epsilon) }
	\end{math} (Eq.\ (\ref{GKernel})).
	\STATE Plot a logarithmic plot of \begin{math} {L(\epsilon)} \end{math} (vs. $\epsilon$).
	\STATE Set \begin{math} \bar{\epsilon}=[\epsilon_0,\epsilon_1] \end{math} as the maximal linear range of $L(\epsilon)$.
\end{algorithmic}
\label{alg:Singer}
\end{algorithm}
Note that $L(\epsilon)$ consists of two asymptotes, $L(\epsilon)\overset{\epsilon\rightarrow 0}{\longrightarrow} \text{log}(N)$ and $L(\epsilon)\overset{\epsilon\rightarrow \infty}{\longrightarrow} \text{log}(N^2)=2\text{log}(N)$, since when $\epsilon\rightarrow 0$, $\myvec{K}$ (Eq.\ (\ref{GKernel})) approaches the Identity matrix, whereas for $\epsilon\rightarrow \infty$, $\myvec{K}$ approaches an all-ones matrix. We denote by $\epsilon_0$ the minimal value within the range $\bar{\epsilon}$ (defined in Alg.\ \ref{alg:Singer}). This value is used in the simulations presented in Section \ref{sec:Exp}. 

A dynamic scale is proposed in \cite{zelnik}, suggesting to calculate a local-scale $\sigma_i$ for each data point $\myvec{x}_i,i=1,...,N$. The scale is chosen using the $L_1$ distance from the $r$-th nearest neighbor of the point $\myvec{x}_i$. Explicitly, the calculation for each point is 
\begin{equation}
\sigma_i=||\myvec{x}_i-\myvec{x}_r||,i=1,...,N,
\end{equation} where $\myvec{x}_r$ is the $r$-th nearest (Euclidean) neighbor of the point $\myvec{x}_i$. The value of the kernel for points $\myvec{x}_i \text{ and } \myvec{x}_j$ is
\begin{equation}\label{GKernelZelnik}
{\cal{K}}(\myvec{x}_i,\myvec{x}_j)\defeq K_{i,j}=\exp\left( {-\frac{||\myvec{x}_i-\myvec{x}_j||^2}{
	\sigma_i \sigma_j}  }\right),i,j\in\{1\ldots N\}.
\end{equation} This dynamic scale guarantees that at least half of the points are connected to $r$ neighbors.

All the methods mentioned above treat $\epsilon$ as a scalar. Thus, when data is sampled from various types of sensors these methods may be dominated by the features (vector elements) with highest energy (or variance). In such cases, each feature $\ell=1,..,D,$ in a data vector $x_i[\ell]$ may require a different scale. In order to re-scale the vector, a diagonal $D\times D$ positive-definite (PD) scaling matrix $\myvec{A} \succ \mymat{0}$ is introduced. The rescaling of the feature vector $\myvec{x}_i$ is set as $\widehat{\myvec{x}}_i = \myvec{A} \myvec{x}_i, 1 \leq i \leq N$. The kernel matrix is rewritten as
\begin{equation}
\label{eq:KernelK}
K_{i,j}={\cal K} \left(\hat{\myvec{x}}_i,\hat{\myvec{x}}_j \right)  = 
\exp \left(-\frac{1}{2\epsilon}\|\hat{\myvec{x}}_i-\hat{\myvec{x}}_j\|^2\right)=
\exp \left( 
- \frac{1}{2\epsilon}
\left(\myvec{x}_i - \myvec{x}_j\right)^T \myvec{A}^T\mymat{A} \left(\myvec{x}_i - \myvec{x}_j\right)
\right).
\end{equation} 
A standard way to set the scaling elements $A_{\ell,\ell}$ is to use the empirical standard deviation of the respective elements and then set $\epsilon=\epsilon_{\rm std}\triangleq 1$. More specifically,
\begin{equation} \label{eq:StdN}
A_{\ell,\ell}=\sqrt{\frac{1}{N}\sum^N_{i=1}
\left(x_i(\ell)-\mu_\ell\right)^2}, \;\;\; \mu_\ell\defeq\frac{1}{N}\sum_{i=1}^Nx_i(\ell)  \;\;\; \ell=1,...,D, \;\;\; \epsilon=\epsilon_{\rm std}=1.
\end{equation}

\subsection{Setting $\epsilon$ for Manifold Learning}
\label{sec:OptMani}

In this subsection we propose a framework for setting the scale parameter $\epsilon$ when the dataset $\myvec{X}$ has some low-dimensional manifold structure $\cal{M}$.
We start by revisiting an analysis from \cite{YoelTomo2008,hein2005intrinsic}, which relate the scale parameter \begin{math} {\epsilon} \end{math} to the intrinsic dimension $\bar{d}$ (Definition \ref{def:ID}) of the manifold.  In \cite{YoelTomo2008} a range of valid values is suggested for $\epsilon$, here we expand the results from \cite{YoelTomo2008,hein2005intrinsic} by introducing a diagonal PD scaling matrix $\myvec{A} $ (as used in Eq.\ (\ref{eq:KernelK})). This diagonal matrix enables a feature selection procedure which emphasizes the latent structure of the manifold. 

Let $\myvec{K}(\epsilon)\in\mathbb{R}^{N\times N}$, $\myvec{A}\in\mathbb{R}^{D\times D}$ be the kernel matrix and diagonal PD matrix (resp.) from Eq.\ (\ref{eq:KernelK}). By taking the double sum of all elements in Eq.\ (\ref{eq:KernelK}), we get
\begin{equation}
\label{eq:sum_w}
S(\epsilon)\triangleq\sum_{i,j} K_{i,j}(\epsilon)  = \sum_{i,j} \exp \left( -\frac{1}{2\epsilon}
\left(\myvec{x}_i - \myvec{x}_j\right)^T \myvec{A}^T\myvec{A} \left(\myvec{x}_i - \myvec{x}_j\right)\right).
\end{equation}
By assuming that the data points in $\myvec{X}$ are independently uniformly distributed over the manifold $\mathcal{M}$, this sum can be approximated using the mean value theorem as
\begin{equation}
\label{eq:sum_w_mv}
S(\epsilon)  \approx \frac{N^2}{\Vol^2\left( \mathcal{M} \right)} \int_ {\mathcal{M}}  \int_{ \mathcal{M} } 
\exp \left( -\frac{1}{2\epsilon} \lVert \myvec{y}' - \myvec{y} \rVert^2 \right) {\rm d}\myvec{y}' {\rm d}\myvec{y},
\end{equation} 
where $\Vol \left(\mathcal{M} \right)\triangleq  \int_{ \mathcal{M} }{\rm d} \myvec{y} $ is the (weighted) volume of the $\bar{d}$-dimensional manifold $\mathcal{M}$,
with ${\rm d} \myvec{y} $, ${\rm d} \myvec{y'} $ being infinitesimal parallelograms on the manifold, carrying the dependence on $\mymat{A}$ (see \cite{moser1965volume} for a more detailed discussion).
When $\epsilon$ is sufficiently small, the integrand in the internal integral in Eq.\ (\ref{eq:sum_w_mv}) takes non-negligible values only when $\myvec{y}'$ is close to the hyperplane tangent to the manifold at $\myvec{y}$. Thus, the integration over $\myvec{y}'$ within a small patch around each $\myvec{y}$ can be approximated by integration in $\mathbb{R}^{\bar{d}}$, so that
\begin{equation}
\label{eq:sum_w_mv2}
S(\epsilon)  \approx
\frac{N^2}{\Vol^2\left( \mathcal{M} \right)} \int_ {\mathcal{M}}  \int_{ \Rn{\bar{d}} } \exp \left( 
-\frac{1}{2\epsilon} \lVert\myvec{y} - \myvec{t} \rVert^2 \right) {\rm d}\myvec{t} {\rm d}\myvec{y},
\end{equation} where $\bar{d}$ is the intrinsic dimension of $\mathcal{M}$ and $\myvec{t}$ is a $\bar{d}$-dimensional vector of coordinates on the tangent plane.

The integral in Eq.\ (\ref{eq:sum_w_mv2}) has a closed-form solution (the internal integral yields $(2\pi \epsilon)^{\bar{d}/2}$, so that the outer integral yields $(2\pi\epsilon)^ {\bar{d}/2}\Vol(\mathcal{M})$ ), and we therefore obtain the relation
\begin{equation}
\label{eq:sum_w_mv4}
S \left( \epsilon \right) = \sum_{i,j} \exp\left(-{r_{i,j}(\mymat{A})}/{2\epsilon}\right) \approx \frac{N^2\left(2 \pi \epsilon \right)^{\bar{d}/2}}{\Vol\left( \mathcal{M} \right)},
\end{equation}
where we have used $ r_{i,j}\left(\myvec{A} \right) \defeq \left(\myvec{x}_i - \myvec{x}_j\right)^T \myvec{A}^T\myvec{A} \left(\myvec{x}_i - \myvec{x}_j\right)$ for shorthand.

The key observation behind our suggested selection of $\mymat{A}$ and $\epsilon$, is that due to the approximation in Eq.\ (\ref{eq:sum_w_mv4}), an ``implied" intrinsic dimension $d_\epsilon(\mymat{A})$ can be obtained for different selections of $\mymat{A}$ and $\epsilon$ as follows. Taking the $\log$ of Eq.\ (\ref{eq:sum_w_mv4}) we have
\begin{equation}
\label{eq:sum_w_log}
\log S \left( \epsilon \right)=
\log \left( \sum_{i,j} \exp\left(-{r_{i,j}(\mymat{A}) }/{2\epsilon}\right) \right) \approx  \frac{\bar{d}}{2} \log\left( \epsilon \right)  + \log\left( \frac{N^2\left(2 \pi \right)^{\bar{d}/2}}{\Vol\left( \mathcal{M} \right)}\right).
\end{equation}
Differentiating w.r.t.\ $\epsilon$ we obtain
\begin{equation}
\frac{\partial \log S(\epsilon)}{\partial \epsilon}
=\frac{\sum_{i,j}r_{i,j}(\mymat{A})\exp\left(-r_{i,j}(\mymat{A})/{2\epsilon}\right)}{2\epsilon^2\sum_{i,j}\exp\left(-r_{i,j}(\mymat{A})/\frac{1}{2\epsilon}\right)}
\approx\frac{\bar{d}}{2\epsilon},
\end{equation}
leading to the ``implied" dimension
\begin{equation}
\label{eq:deA}
    d_\epsilon(\mymat{A})\approx
    \frac{\sum_{i,j}r_{i,j}(\mymat{A})\exp\left(-r_{i,j}(\mymat{A})/{2\epsilon}\right)}{\epsilon\sum_{i,j}\exp\left(-r_{i,j}(\mymat{A})/{2\epsilon}\right)}.
\end{equation}
We propose to choose the scaling so as to minimize the difference between the estimated dimension $\hat{d}$ (see section \ref{sec:ID}) and the implied dimension $d_{\epsilon}(\mymat{A})$.
We therefore set $\myvec{A}$ (and $\epsilon$) based on solving the following optimization problem
\begin{equation}
\label{eq:sum_w_optimization}
\myvec{A} = \arg \min_{\footnotesize{\mymat{A}},\epsilon}  |d_{\epsilon}(\mymat{A})-\hat{d}|\;\;\;\text{s.t. }
\mymat{A} \text{ is diagonal and PD}, \epsilon>0.
\end{equation}
We note that this minimization problem has one degree of freedom, which can be resolved, e.g., by arbitrarily setting $A_{1,1}=1$ (see Algorithm \ref{alg:solv_global} below). When working in a sufficently small ambient dimension $D$, the minimizaion can be solved using an exhaustive search (e.g., on some pre-defined grid of scaling values). However, for large $D$ an exhaustive search may become unfeasible in practice, so we propose a greedy algorithm, outlined below as Algorithm~\ref{alg:solv_global}, for computing both the scaling matrix $\mymat{A}$ and the scale parameter $\epsilon$.

\begin{algorithm}[h!]
\caption{Manifold Based Vector Scaling}
\label{alg:solv_global}
\KwIn{dataset: $\myvec{X}=\{\myvec{x}_1,\myvec{x}_2,...,\myvec{x}_N\} \in \mathbb{R}^{D \times N} $.\\ Intrinsic dimension estimate: $\hat{d}$ (optional).}
\KwOut{Normalized dataset $\widehat{\myvec{X}}$
}
\begin{algorithmic}[1]
	\IF {isempty($\hat{d}$)} 
	\STATE Apply DANCo \cite{danco} to $\myvec{X}$ to estimate $\hat{d}$ (description in Appendix).
    \ENDIF
	\STATE Set $\widehat{\myvec{X}}^{(\hat{d})} = \left(\myvec{X}_{1:\hat{d},:} - {\rm mean}(\myvec{X}_{1:\hat{d},:}) \right)./{\rm std}(\myvec{X}_{1:\hat{d},:})$.  \\
	\FOR{$\ell=\hat{d}+1$ \TO $D$}
		
		\STATE Construct $\widehat{\myvec{X}}^{(\ell)}\triangleq[\widehat{\myvec{X}}^{(\ell-1)};\myvec{X}_{\ell,:}]$\\
		\STATE Find $A_{\ell,\ell}>0$ and $\epsilon>0$ minimizing $|d_\epsilon(\mymat{A})-\hat{d}|$, where $\mymat{A}\in\mathbb{R}^{\ell\times\ell}$ is an identity matrix with its $(\ell,\ell)$-th element (only) replaced by $A_{\ell,\ell}$, and where $r_{i,j}(\mymat{A})$ in Eq.\ (\ref{eq:deA}) operates on $\widehat{\myvec{X}}^{(\ell)}$\\
		\STATE Update $\widehat{\myvec{X}}^{(\ell)} = [\widehat{\myvec{X}}^{(\ell-1)};\myvec{X}_{\ell,:}\cdot A_{\ell,\ell}]/\sqrt{\epsilon}$ \\
	\ENDFOR
\end{algorithmic}
\end{algorithm}

The proposed algorithm (Algorithm \ref{alg:solv_global}) operates by iteratively constructing the normalized dataset $\widehat{\mymat{X}}\defeq \mymat{A}\mymat{X}/\sqrt{\epsilon}$ row by row. To this end, the algorithm is first initialized by normalizing the first $\hat{d}$ rows (coordinates) using their empirical standard deviations (note that the estimated (or known) intrinsic dimension $\hat{d}$ is either provided as an input or estimated using DANCo \cite{danco}). Then, in the $\ell$-th iteration ($\ell=\hat{d}+1,\ldots D$) only the scaling factor $A_{\ell,\ell}$ for the next ($\ell$-th) row of $\mymat{X}$ and a new overall scaling $\epsilon$ are found, using a (two-dimensional) exhaustive search. The resulting $A_{\ell,\ell}$ is then applied to the $\ell$-th row, and the entire $\ell\times N$ data block is normalized by $\sqrt{\epsilon}$ before the next iteration. 

The computational complexity of this algorithm with $k$ hypotheses of $\epsilon$ and $A_{l,l}$ for each iteration is $O \left( N^2 k D \right)$, since $N^2$ operations are required in the computation of a single scaling hypothesis, and this is required for each coordinate $d=1,...,D$. 

Due to the greedy nature of the algorithm, its performance depends on the order of the $D$ features. We further propose to reorder the features using a soft feature-selection procedure. The studies in \cite{cohen2002feature,lu2007feature,song2010feature} suggest an unsupervised feature selection procedure based on PCA. The idea is to use the features which are most correlated with the top principle components. We propose an algorithm for reordering the features based on their correlation with the leading coordinates of the DM embedding. The algorithm (Algorithm \ref{alg:fear_perm} below) is called Correlation Based Feature Permutation (CBFP), and uses the correlation between the $D$ features and $\hat{d}$ embedding coordinates. This correlation value provides a natural measure for the influence of each feature on the extracted embedding.
\begin{algorithm}[h!]
\caption{Correlation Based Feature Permutation (CBFP)}
\label{alg:fear_perm}
\KwIn{dataset: $\myvec{X}=\{\myvec{x}_1,\myvec{x}_2,...,\myvec{x}_N\} \in \mathbb{R}^{D \times N} $.\\ Intrinsic dimension estimate: $\hat{d}$ (optional).}
\KwOut{Feature permutation vector $\myvec{j}$, such that $\myvec{j}= {\sigma} ([1,...,D])$ and $\sigma()$ is a permutation operation. }
\begin{algorithmic}[1]
\IF {isempty($\hat{d}$)} 
	\STATE Apply DANCo \cite{danco} to $\myvec{X}$ to estimate $\hat{d}$ (described in the Appendix).
    \ENDIF
	\STATE Compute $\epsilon_{\text{MaxMin}}$ based on Eq.\ (\ref{eq:MaxMin}).
	\STATE Using $\epsilon_{\text{MaxMin}}$ for the kernel scaling, compute DM representation $\myvec{\Psi}^{\hat{d}}$ using Eq.\ (\ref{eq:PsiLow}).
	\STATE Compute a vector $\myvec{c}\in\mathbb{R}^D$ of the feature-embedding correlation scores, defined as
	\begin{equation}
	c_i \defeq \sum^{\hat{d}}_{\ell=1} |{\rm corr}(\myvec{X}_{i,:},\myvec{\Psi}^{\hat{d}}_{\ell,:})|,i=1,...,D,
	\end{equation}
	where ${\rm corr}(\cdot,\cdot)$ denotes the correlation coefficient between its two vector agruments.
	\STATE Set $[\myvec{v},\myvec{j}]=\text{sort}(\myvec{c})$, where $\myvec{v}$, $\myvec{j}$ are the sorted values and corresponding indices of $\myvec{c}$.
	\STATE Reorder the $D$ features by $\myvec{\widetilde{X}}=\myvec{X}(\myvec{j},:)$.
	
\end{algorithmic}
\end{algorithm}

In subsection \ref{sec:ExpMani} below, we evaluate the performance of Algorithms \ref{alg:solv_global} and \ref{alg:fear_perm} when applied to synthetic data embedded in artificial manifolds.

\subsection{Setting $\epsilon$ for Classification}
\label{sec:OptClass}
Classification algorithms use a metric space and an induced distance to compute the category of unlabeled data points. Dimensionality reduction is effective for capturing the essential intrinsic geometry of the data and neglecting the undesired information (such as noise). Therefore, dimensionality reduction can drastically improve classification results \cite{lindenbaum2015musical}. As previously mentioned, $\epsilon$ playes a crucial role in the performance of kernel methods for dimensionality reduction. Various methods have been proposed for finding a scale $\epsilon$ that would potentially optimize classification performance. 

Studies such as by Gaspar {\em et al.} \cite{gaspar2012parameter} and by Staelin \cite{staelin2003parameter} use a cross-validation procedure and select the scale $\epsilon$ that maximizes the performance on the validation data. In \cite{chapelle2002choosing}, Chapelle {\em et al.} apply gradient descent to a classification-error function to find an 'optimal' scale $\epsilon$. Although these methods share our goal, they require performing classification on a validation set for selecting the scale parameter. To the best of our knowledge, the only method that estimates the scale without using a validation set was proposed by Campbell {\em et al.} in \cite{campbell1999dynamically}, where for binary classificationit was suggested to use the scale that maximizes the margin between the support vectors. The authors show empirically that their suggested value correlates with peak classification performance on a validation set. 

In this subsection we focus on DM for dimensionality reduction and demonstrate the influence of the scale parameter $\epsilon$ on classification performance in the low-dimensional space. The contribution in this section is threefold:
\begin{itemize}
    \item We use the Davis-Kahan theorem to analyze a perturbed version of ideally separated classes. This allows us to optimize the choice of $\epsilon$ merely based on the eigengap of the perturbed kernel.  
    \item Based on our study in \cite{lindenbaum2015musical}, we present an intuitive geometric metric to evaluate the separation in a multi-class setting. We show empirically that the scale which maximizes the ratio between class separation and the average class spread also optimizes classification performance.
    \item Finally, to reduce the computational complexity involved in the spectral decomposition of the affinity kernel, we present a heuristic that allows to estimate the scale parameter based on the stochastic version of the affinity kernel.
\end{itemize}

Next, we develop tools to estimate a scale parameter based on a given training set. The training set denoted as $\myvec{T}\subset \mathbb{R}^{D\times N}$ consists of $N_C$ classes. The classes are denoted by $\myvec{C}_1,...,\myvec{C}_{N_C}$. In this study we focus on the balanced setting, where the number of samples in each class is $N_P$, thus the total number of data points is $N=N_PN_C$.  We use a scalar scaling factor $\epsilon$. However, the analysis provided in this subsection could be expanded to a vector scaling (namely, to the use of a diagonal PD scaling matrix) in a straightforward way.

\begin{figure}[H]
	\centering
	\includegraphics[scale = 1.5]{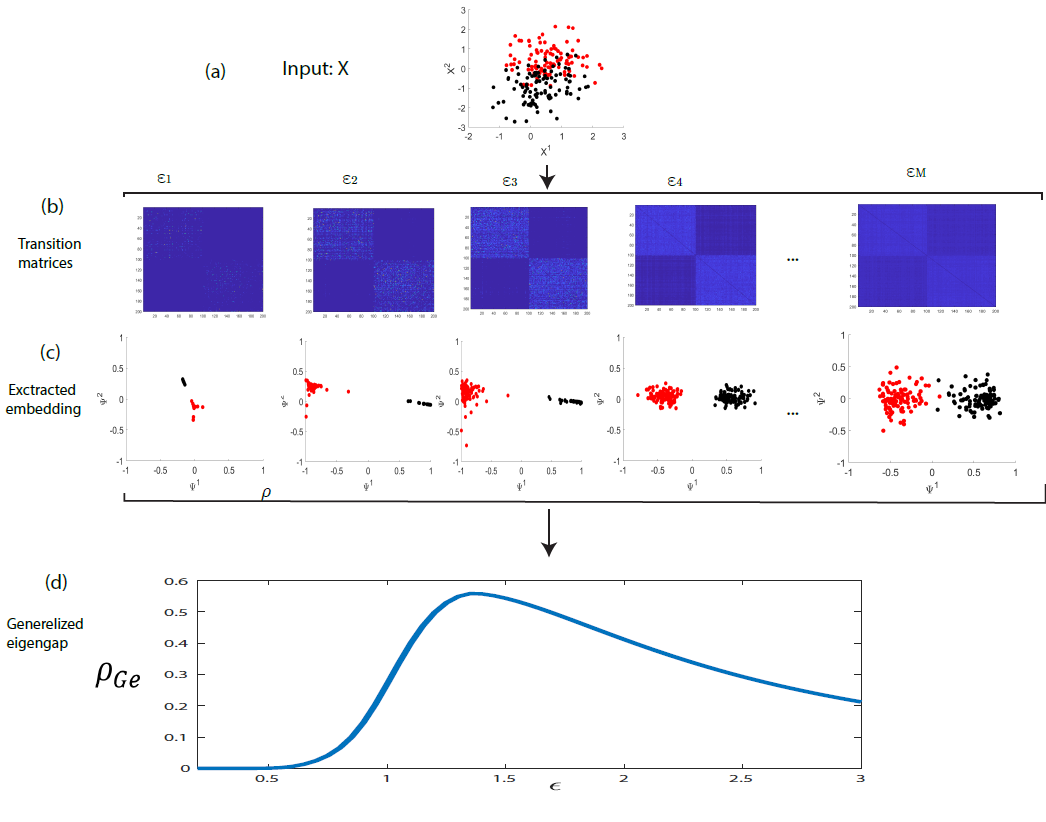} 
	
		\caption{A schematic of the three proposed methods for estimated $\epsilon$ dedicated for classification. Given an input \myvec{X} we use range of hypothesis $\bar{\myvec{\epsilon}}$. (a) The first two dimensions of high-dimensional Gaussian classes. (b) The probabilistic approach- the scale is estimated based on modified transition matrices. (c) The geometric approach- the scale is estimated based on the geometry of the embedding. (d) The spectral approach- the scale is estimated based on a generalized eigengap computed for each scale within the hypothesis range. }
		\label{Cartoon}
\end{figure}

{\subsubsection{The Geometric Approach}}
The following approach for setting $\epsilon$ is based on the geometry of the extracted embedding. The idea is to extract low-dimensional representations for a range of candidate $\epsilon$'s. Then choose $\epsilon$ which maximizes the among-classes to within-class variances ratio. In other words, choose a scale such that the classes are dense and far apart from each other in the resulting low-dimensional embedding space. This is done by maximizing the ratio between the inter-class variance and the sum of  intera-class variances. We have explored a similar approach for an audio based classification task in \cite{lindenbaum2015musical}. This geometric approach is implemented using the following steps:
\begin{enumerate}
\item{Compute DM-based embeddings $\myvec{\Psi}^d_{\epsilon}(\myvec{x}_n),n=1,...,N$, (Eq.\ (\ref{eq:PsiLow})) for various candidate-values of $\epsilon$.}
\item{Denote by \begin{math} {{{\myvec{\mu}}}}_i \end{math} the center of mass for class $\myvec{C}_i,i=1,...,N_C$, and by \begin{math} {{\myvec{\mu}}}_a \end{math} the center of mass for all the data points. All $\myvec{\mu}_i$ and $\myvec{\mu}_a$ are computed in the low-dimensional DM-based embedding $\myvec{\Psi}^d_{\epsilon}(\myvec{x}_n)$ - see step 1.}
\item{For each class $\myvec{C}_i$, the average square distance (in the embedding space) is computed for the $N_P$ data points from the center of mass $\myvec{\mu}_i$ such that
	\begin{equation} D_{c_i}=\frac{1}{N_P}{\sum_{{{\myvec{x}}_n}\in \myvec{C}_i}}{||{{{{\myvec{\Psi}^d_{\epsilon}({{\myvec{x}}_n})}}}}-{{{\myvec{\mu}}}}_i||^2},i=1,...,N_C.  \end{equation}}
\item{The same measure is computed for all data points such that
	\begin{equation} D_a=\frac{1}{N}{\sum_{{{\myvec{x}}_n} \in X}}{||{{{{\myvec{\Psi}^d_{\epsilon}({{\myvec{x}}_n})}}}}-{{{\myvec{\mu}}}}_a||^2} .  \end{equation}}
\item{Define 
	\begin{equation}
	\label{SigOpt}
	{\rho_{\Psi}\defeq \frac{D_a}{\overset{N_C}{\underset{i=1}\sum }{D_{c_i}}}} .
	\end{equation} }
\item{Choose \begin{math} \epsilon \end{math} which maximizes $\rho_{\Psi}$
	\begin{equation}
	\label{eq:rhopsi}
	{\epsilon}_{\rho_{\Psi}}=\arg\max_\epsilon \rho_{\Psi}.
	\end{equation}}
\end{enumerate}
The idea is that \begin{math} {\epsilon}_{\rho_{\Psi}} \end{math} (Eq.\ (\ref{eq:rhopsi})) inherits the inner structure of the classes and neglects the mutual structure. In subsection \ref{sec:ExpClass}, we describe experiments that empirically evaluate the influence of $\epsilon$ on the performance of classification algorithms. We note, however, that this approach requires an eigendecomposition computation for each $\epsilon$, thus, its computational complexity is of order of ${\cal{O}}(N^2 d)$ ($d$ being the number of required eigenvectors).
\subsubsection{The Spectral Approach}
\label{sec:spectral}
In this subsection, we analyze the relation between the spectral properties of the kernel and its corresponding low-dimensional representation. We start the analysis by constructing an ideal training set, with well separated classes. Then, we add a small perturbation to the training set and compute the perturbed affinity matrix $\myvec{K}$. Based on the spectral properties of the perturbed kernel we suggest a scaling $\epsilon$ to capture the essential information for class separation.\\
\begin{description}
\item[The Ideal Case:] 
We begin the discussion by considering an ideal classification setting, in which the $N_C$ classes are assumed to be well-separated in the ambient space $\myvec{X}$ (a similar setting for spectral clustering was described in \cite{ng2002spectral}). The separation is formulated using the following definitions:
\begin{enumerate}
	\item The {\em Euclidean gap} is defined as
	\begin{equation}\label{eq:Dgap}{
		D_{\rm Gap}(\myvec{X})\defeq \min_{\stackrel{\myvec{x}_i\in \myvec{C}_\ell, \myvec{x}_j \in \myvec{C}_m}{\ell,m=1,...,N_C,\;\;\ell\neq m}}||\myvec{x}_i-\myvec{x}_j||^2}.
	\end{equation}
	This is the Euclidean distance between the two closest data points belonging to two different classes.
	\item The {\em maximal class width} is defined as
	\begin{equation}\label{eq:Dclass}{
		D_{\rm Class}(\myvec{X})\defeq\max_{\stackrel{\myvec{x}_i,\myvec{x}_j\in \myvec{C}_\ell}{\ell=1,...,N_C}}||\myvec{x}_i-\myvec{x}_j||^2}. 
	\end{equation} This is the maximal Euclidean distance between two data points belonging to the same class.
\end{enumerate}   We assume that $D_{\rm Class} \ll D_{\rm Gap}$ such that the classes are well separated. Using this assumption and the decaying property of the Gaussian kernel, the matrix $\myvec{K}$ (Eq.\ (\ref{GKernel})) converges to the following block form
\begin{equation} \label{eq:KMAT} {{\bar{\myvec{K}}}}= \begin{bmatrix}
\myvec{K}^{(1)}    &    0    &    . . .    &    0        \\
0    &    \myvec{K}^{(2)}    &    . . .    &    0        \\
:    &    :    &    :    &    :    \\
0    &    0    &    . . .    &    \myvec{K}^{(N_C)}     
\end{bmatrix},\text{ } \bar{\myvec{P}}=\bar{\myvec{D}}^{-1}\bar{\myvec{K}},\text{ } \bar{\myvec{P}}=\begin{bmatrix}
\myvec{P}^{(1)}    &    0    &    . . .    &    0        \\
0    &    \myvec{P}^{(2)}    &    . . .    &    0        \\
:    &    :    &    :    &    :    \\
0    &    0    &    . . .    &    \myvec{P}^{(N_C)}     
\end{bmatrix} , \end{equation} where $\bar{D}_{i,i}=\underset{j}{\sum}{\bar{K}_{i,j}}$. For the ideal case, we further assume that the elements of ${\myvec{K}}^{(i)},i=1,...,N_P$, are non-zeros because $\epsilon \sim D_{\rm Class}$ and the classes are connected.
\begin{prop}
	\label{prop:ideal}
	Assume that  $D_{\rm Class} \ll D_{\rm Gap}$, then, the matrix $\bar{\myvec{P}}$ (Eq.\ (\ref{eq:KMAT})) has an eigenvalue $\lambda=1$ with multiplicity $N_C$. Furthermore, the first $N_C$ coordinates of the DM mapping (Eq.\ (\ref{eq:Psi})) are piecewise constant. The explicit form of the first nontrivial eigenvector $\myvec{\psi}_1$ is given by  \[ \myvec{\psi}_1=[    \underbrace{1,...,1}_{\text{$N_P$ 1's}},\underbrace{0,...,0}_{\text{$N-N_P$ 0's}}]^T/\sqrt{N_P}.\] The eigenvectors $\myvec{\psi}_i,i=2,...,N_C$ have the same structure but cyclically shifted to the right by $(i-1) \cdot N_P$ bins.
\end{prop}
\begin{proof}
	Recall that ${\myvec{\bar{P}}=\myvec{\bar{D}}^{-1}\myvec{\bar{K}}}$ (row-stochastic). Due to the special block structure of $\myvec{\bar{K}}$ (Eq.\ (\ref{eq:KMAT})), each block $\myvec{P}^{(i)}, i=1,...,N_P$, is row stochastic. Thus, \[ \myvec{\psi}_i=[ \underbrace{0,...,0}_{\text{$(i-1)\cdot N_P$ 0's}}   \underbrace{1,...,1}_{\text{$N_P$ 1's}},\underbrace{0,...,0}_{\text{$N-i \cdot N_P$ 0's}}]^T/\sqrt{N_P}.\] Each eigenvector $\myvec{\psi}_i,i=1,..,N_p$ consists of a block of $1$-s at the row indices that correspond to $\myvec{P}^{(i)}$ (Eq.\ (\ref{eq:KMAT})), padded with zeros. $\myvec{\psi}_i$ is the right eigenvector (${1\cdot \myvec{\psi}_i=\bar{\myvec{P}}\cdot \myvec{\psi}_i }$), with the eigenvalue $\lambda=1$. We now have an eigenvalue $\lambda=1$ with multiplicity $N_P$ and piecewise constant eigenvectors denoted as $\myvec{\psi}_i,i=1,...,N_P$. 
\end{proof}
Each data point $\myvec{x}_i \in \myvec{C}_\ell,\;\ell=1,...,N_C$, corresponds to a row within the respective sub-matrix ${\mymat{P}}^{(\ell)}$. Therefore, using $\Psi^{N_C}_{\epsilon}(\myvec{T})=[1\cdot \myvec{\psi}_1,...,1\cdot \myvec{\psi}_{N_C}]^T$ as the low-dimensional representation of $\myvec{T}$, all the data points from within a class are mapped to a point in the embedding space.  
\begin{cor}\label{cor1}
	Using the first $N_C$ eigenvectors of $\bar{\myvec{P}}$ (Eq.\ (\ref{eq:KMAT})) as a representation for $\myvec{T}$ such that $ \Psi^{N_C}_{\epsilon}(\myvec{T})=[1\cdot \myvec{\psi}_1,...,1\cdot \myvec{\psi}_{N_C}]^T$ yields that the distances $D_{\rm Gap}(\myvec{\Psi}^{N_C}_{\epsilon})=2\text{ and }D_{\rm Class}(\myvec{\Psi}^{N_C}_{\epsilon})=0$ (defined in Eqs.\ (\ref{eq:Dgap}) and (\ref{eq:Dclass}), respectively).
\end{cor}


\begin{proof}
	Based on the representation in Proposition \ref{prop:ideal} along with Eq.\ (\ref{eq:Dgap}), we get
	\begin{multline}
	D_{\rm Gap}(\myvec{\Psi}^{N_C}_{\epsilon})=\min_{\stackrel{\myvec{x}_i\in \myvec{C}_\ell, \myvec{x}_j \in \myvec{C}_m}{\ell,m=1,...,N_C,\;\;\ell\neq m}}||\myvec{\Psi}^{N_C}_{\epsilon}(\myvec{x}_i)-\myvec{\Psi}^{N_C}_{\epsilon}(\myvec{x}_j)||^2=\sum^{N_C}_{r=1}\lambda_r\cdot(\psi_r(i)-\psi_r(j))^2\\ =1\cdot(1-0)^2+1\cdot(0-1)^2+\sum^{N_C}_{r=3}1\cdot(0-0)^2=2.
	\end{multline} 
	In a similar manner, we get by Eq.\ (\ref{eq:Dclass}) 
	\begin{multline}
	D_{\rm Class}(\myvec{\Psi}^{N_C}_{\epsilon})=\max_{\stackrel{\myvec{x}_i,\myvec{x}_j \in \myvec{C}_\ell}{\ell=1,...,N_C}}||\myvec{\Psi}^{N_C}_{\epsilon}(\myvec{x}_i)-\myvec{\Psi}^{N_C}_{\epsilon}(\myvec{x}_j)||^2=\sum^{N_C}_{r=1}\lambda_r\cdot(\psi_r(i)-\psi_r(j))^2\\
	=1\cdot(0-0)^2+1\cdot(1-1)^2+\sum^{N_C}_{r=3}1\cdot(0-0)^2=0.
	\end{multline}
\end{proof}
Corollary \ref{cor1} implies that we can compute an efficient representation for the $N_C$ classes. We denote this representation by $\myvec{\bar{\Psi}}_{\epsilon}^{N_C}=[\myvec{\psi}_1,\myvec{\psi}_2,...,\myvec{\psi}_{N_C}]$.\\
\item[The Perturbed Case:] 

In real datasets, we cannot expect that the off block-diagonal elements of the affinity matrix $\myvec{K}$ would be zero. The data points from different classes in real datasets are not completely disconnected, and we can assume they are weakly connected. This low connectivity implies that off-block-diagonal values of $\myvec{K}$ are non-zeros. We analyze this more realistic scenario by assuming that $\mymat{K}$ is a perturbed version of the ``Ideal'' block form of $\myvec{\bar{K}}$.
Perturbation theory addresses the question of how a small change in a matrix relates to a change in its eigenvalues and eigenvectors. 
In the perturbed case, the off-block-diagonal terms are non-zeros and the obtained (perturbed) matrix $\widetilde{\myvec{K}}$ takes the form 
\begin{equation} \label{eq:Pert}
\widetilde{\myvec{K}}=\bar{\myvec{K}}+\widehat{\myvec{W}},
\end{equation} where $\widehat{\myvec{W}}$ is assumed to be a symmetrical small perturbation of the form \begin{equation} \label{EQWMAT} {\widehat{\myvec{W}}}= \begin{bmatrix}
-\myvec{W}^{(1,1)}    &    \myvec{W}^{(1,2)}    &    . . .    &    \myvec{W}^{(1,N_C)}        \\
\myvec{W}^{(2,1)}   &     -\myvec{W}^{(2,2)}   &    . . .    &    \myvec{W}^{(2,N_C)}       \\
:    &    :    &    :    &    :    \\
\myvec{W}^{(N_C,1)}   &    \myvec{W}^{(N_C,2)}    &    . . .    &    - \myvec{W}^{(N_C,N_C)}     
\end{bmatrix}, \myvec{W}^{(\ell,m)}=\myvec{W}^{(m,\ell)},\;\;\;\ell,m=1,...,N_C. \end{equation} 

The analysis of the ``Ideal case'' has provided an efficient representation for classification tasks as described in Proposition \ref{prop:ideal}. We propose to choose the scale parameter $\epsilon$ such that the extracted representation based on ${\bar{\myvec{K}}}$ (Eq.\ (\ref{eq:KMAT})) is similar to the extracted representation using $\tilde{\myvec{K}}$ (Eq.\ (\ref{eq:Pert})). For this purpose we use the following theorem.

\begin{theorem} {(\bf{Davis-Kahan}) \cite{stewart1990matrix}} \label{theo:DK}
	Let $\bar{\myvec{A}}$ and $\widehat{\myvec{B}}$ be Hermitian matrices of the same dimensions, and let $\widetilde{\myvec{A}}\triangleq\bar{\myvec{A}}+\widehat{\myvec{\myvec{B}}}$ be a perturbed version of $\bar{\myvec{A}}$. Set an interval $S$, denote the eigenvalues within $S$ as $\lambda_S(\bar{\myvec{A}})$ and $\lambda_S(\widetilde{\myvec{A}})$ with a corresponding set of eigenvectors $\myvec{\bar{V}}_1$ and $\widetilde{\myvec{V}_1}$ for $\myvec{\bar{A}}$ and $\widetilde{\myvec{A}}$, respectively. Define $\delta$ as 
	\begin{equation}
	\label{eq:DavisKahan}
	\delta\triangleq\min \{ |\lambda(\widetilde{\myvec{A}})-s|; \lambda(\widetilde{\myvec{A}}) \notin S, s\in S\}.
	\end{equation} 
	Then the distance 
	\begin{equation}
	d(\bar{\myvec{V}}_1,\widetilde{\myvec{V}_1})\triangleq\left\|\sin\Theta \left(\bar{\myvec{V}}_1,\widetilde{\myvec{V}_1}\right)\right\|_F
	\leq \frac{1}{\delta} \left\|\widehat{\myvec{B}}\right\|_F,
    \end{equation}
where $\Theta \left(\bar{\myvec{V}}_1,\widetilde{\myvec{V}_1}\right)$ is a diagonal matrix with the principal angles on the diagonal, and $\|\cdot\|_F$ denotes the Frobenius norm.
\end{theorem}
In other words, the theorem states that the eigenspace spanned by the perturbed kernel $\widetilde{\mymat{K}}$ is similar, to some extent, to the eigenspace spanned by the ideal kernel $\bar{\mymat{K}}$. The distance between these eigenspaces is bounded by $\frac{1}{\delta}\|\widehat{\myvec{W}}\|_F$. Theorem \ref{Theorem:32} provides a measure which helps to minimize the distance between the ideal representation $\bar{\myvec{\Psi}}^{N_C}_{\epsilon}$ (proposition \ref{prop:ideal}) and the realistic (perturbed) representation $\widetilde{\myvec{\Psi}}^{N_C}_{\epsilon}$. 
\begin{theorem} \label{Theorem:32}
	The distance between $\bar{\myvec{\Psi}}^{N_C}_{\epsilon}\in \mathbb{R}^{N_C}$ and $\widetilde{\myvec{\Psi}}^{N_C}_{\epsilon}\in \mathbb{R}^{N_C}$ in the DM representations based on the matrices $\bar{\myvec{P}}$ and $\widetilde{\myvec{P}}$, respectively, is bounded such that
	\begin{equation}
	d\left(\bar{\myvec{\Psi}}^{N_C}_{\epsilon},\widetilde{\myvec{\Psi}}^{N_C}_{\epsilon}\right) =\left\|\sin\Theta \left(\bar{\myvec{\Psi}}^{N_C}_{\epsilon},\widetilde{\myvec{\Psi}}^{N_C}_{\epsilon}\right)\right\|_F\leq \frac{||\widehat{\myvec{W}}||_F||\bar{\myvec{D}}^{-1/2}||^2_F}{\widetilde{\lambda}_{N_{C}}-\widetilde{\lambda}_{N_{C}+1}},
	\end{equation} where $\widehat{\myvec{W}}$ is the perturbation matrix defined in Eq.\ (\ref{eq:Pert}) and $\myvec{\bar{D}}$ is a diagonal matrix whose elements are the sums of rows ${\bar{D}}_{i,i}=\sum_j{\bar{K}}_{i,j}$.
\end{theorem}
\begin{proof}
	Define $\bar{\myvec{A}}\defeq\myvec{\bar{D}}^{-1/2}\myvec{\bar{K}}\myvec{\bar{D}}^{-1/2}=\myvec{\bar{D}}^{1/2}\myvec{\bar{P}}\myvec{\bar{D}}^{-1/2}$. Based on Eq.\ \ref{eq:Pert}, we have 
	\begin{equation} \label{eq:PertB}
	\widetilde{\myvec{A}}=\bar{\myvec{A}}+\myvec{\bar{D}}^{-1/2}\widehat{\myvec{W}}\myvec{\bar{D}}^{-1/2}.
	\end{equation} 
	We are now ready to use Theorem \ref{theo:DK}. Assume that the eigenvalues $\{\bar{\lambda}_i\}$ of $\bar{\mymat{A}}$ and $\{\widetilde{\lambda}_i\}$ of $\widetilde{\mymat{A}}$ are ordered in descending order, and set $S=[{\lambda}(\widetilde{\myvec{A}})_{N_C},1]$, denoting the first $N_C$ eigenvectors of $\bar{\myvec{A}}$ and of $\widetilde{\myvec{A}}$ as $\bar{\myvec{V}}_1$ and $\widetilde{\myvec{V}}_1$, respectively. Obviously, by construction we have $\widetilde{\lambda}_i\in S$, $i=1,...,N_C$. Based on the analysis of the ``ideal'' matrix $\bar{\myvec{P}}$, we know that its first $N_C$ eigenvalues are equal to 1. Noting that $\bar{\mymat{A}}$ is algebraically similar to $\bar{\mymat{P}}$, they have the same eigenvalues, implying that $\bar{\lambda}_i\in S$, $i=1...,N_C$, as well. Using the definition of $\delta$ from Eq.\ (\ref{eq:DavisKahan}), we conclude that $\delta=\widetilde{\lambda}_{N_{C}}-\widetilde{\lambda}_{N_{C}+1}$. Setting $\bar{\myvec{A}}\defeq\myvec{\bar{D}}^{-1/2}\myvec{\bar{K}}\myvec{\bar{D}}^{-1/2}$ and $\myvec{\widehat{B}}=\myvec{\bar{D}}^{-1/2}\widehat{\myvec{W}}\myvec{\bar{D}}^{-1/2}$, the Davis-Kahan Theorem \ref{theo:DK} asserts that the distance between the eigenspaces $\myvec{\bar{V}}_1$ and $\myvec{\widetilde{V}}_1$ is bounded such that 
		\begin{equation}
	d\left(\bar{\myvec{\Psi}}^{N_C}_{\epsilon},\widetilde{\myvec{\Psi}}^{N_C}_{\epsilon}\right) =\left\|\sin\Theta \left(\bar{\myvec{\Psi}}^{N_C}_{\epsilon},\widetilde{\myvec{\Psi}}^{N_C}_{\epsilon}\right)\right\|_F\leq \frac{||\widehat{\myvec{W}}||_F||\bar{\myvec{D}}^{-1/2}||^2_F}{\widetilde{\lambda}_{N_{C}}-\widetilde{\lambda}_{N_{C}+1}},
	\end{equation}
	The eigen-decomposition of $\myvec{\bar{A}}$ is written as $\myvec{\bar{A}}=\myvec{\bar{V} \bar{\Sigma} \bar{V}}^T$. Note that $\myvec{\bar{P}}=\myvec{\bar{D}}^{-1/2}\myvec{\bar{A}}\myvec{\bar{D}}^{1/2}$ which means that the eigen-decomposition of $\myvec{\bar{P}}$ could be written as $\myvec{\bar{P}}=\myvec{\bar{D}}^{-1/2}\myvec{\bar{V} \bar{\Sigma} \bar{V}}^T\myvec{\bar{D}}^{1/2}$ and the right eigenvectors of $\myvec{\bar{P}}$ are $\myvec{\bar{\Psi}}=\myvec{\bar{D}}^{-1/2}\myvec{\bar{V}}$. Using the same argument for $\myvec{\widetilde{A}}$ and choosing the eigenspaces using the first $N_C$ eigenvectors, we get that decreasing the term $d(\bar{\myvec{V}}_1,\widetilde{\myvec{V}}_1)$ also decreases $d(\bar{\myvec{\Psi}}^{N_C}_{\epsilon},\widetilde{\myvec{\Psi}}^{N_C}_{\epsilon})$.
\end{proof}
\begin{asump}
	\label{assump1}
	The perturbation matrix $\widehat{\myvec{W}}$ (Eq.\ (\ref{eq:Pert})) changes only slightly over the range of values of $\epsilon\in(D_{\rm Class},D_{\rm Gap})$.\\
	{\bf{Explanation:}} For two data points $\myvec{x}_i,\myvec{x}_j$ from different classes $\myvec{x}_i\in C_\ell$, $\myvec{x}_j\in C_m,\ell\neq m$ and $\epsilon \sim D_{\rm Class}$, the values of  $\widehat{\myvec{W}}_{i,j} \ll 1$. The decaying property of the Gaussian kernel provides a range of values for $\epsilon\sim D_{\rm Class}$ such that the perturbation matrix $\widehat{\myvec{W}}$ is indeed small.
	In subsection \ref{sec:ExpClass} below we evaluate this assumption using a mixture of Gaussians.
\end{asump}
\begin{cor}
	\label{cor}
	Given $N_C$ classes under the perturbation assumption and assumption \ref{assump1}, the generalized eigengap is defined as $Ge=|\widetilde{\lambda}_{N_C}-\widetilde{\lambda}_{N_C+1}|$. The scale parameter $\epsilon$, which maximizes $Ge$ 
	\begin{equation} \label{eq:GE} \epsilon_{Ge}=\arg\max_\epsilon(Ge)=\arg\max_\epsilon(\widetilde{\lambda}_{N_C}-\widetilde{\lambda}_{N_C+1}) \end{equation} provides the best class separation using an $N_C$ coordinates embedding ($\widetilde{\myvec{\Psi}}^{N_C}_{\epsilon}$).
\end{cor}

\end{description}
This approach also requires computing an eigendecomposition for each $\epsilon$ value, thus, its computational complexity is of the order of ${\cal{O}}(N^2 N_C)$.

{\subsubsection{The Probabilistic Approach}}

We introduce here notations from graph theory to compute a measure of the class separation based on the stochastic matrix ${\myvec{P}}$ (Eq.\ (\ref{eq:DefP})). Based on the values of the matrix ${\myvec{P}}$, a {\em Cut} \cite{shi2000normalized} is defined for any two subsets $\myvec{A},\myvec{B} \subset \myvec{T}$
\begin{equation}
{\rm cut}(\myvec{A},\myvec{B})=\underset{\myvec{x}_i\in \myvec{A},\myvec{x}_j\in \myvec{B}}{\sum}P_{i,j}.
\end{equation}
Given $N_C$ classes $\myvec{C}_1,\myvec{C}_2,\myvec{C}_3,...,\myvec{C}_{N_C}\subset \myvec{T}$, we define the {\em Classification Cut} by the following measure
\begin{equation}
{\rm Ccut}(\myvec{C}_1,...,\myvec{C}_{N_C})=\frac{\sum_{\ell=1}^{N_C}{\rm cut}(\myvec{C}_\ell,{\myvec{C}_\ell})}{N}.
\end{equation}
In clustering problems, a partition is searched such that the normalized version of the cut is minimized \cite{dhillon2004kernel,ding2005equivalence}. We use this intuition for a more relaxed classification problem.

We first define a {\em Generalized cut} using the following matrix
\begin{equation}
\widehat{P}_{i,j}\defeq\begin{cases}
{P}_{i,j},& \text{if } i\neq j\\
0,              & \text{otherwise}
\end{cases}.
\end{equation}
Based on $\widehat{\mymat{P}}$ (which carries the dependence on $\epsilon$), the Generalized cut is then defined as 
\begin{equation}
    {\rm Gcut}(\myvec{A},\myvec{B})\defeq\underset{\myvec{x}_i\in \myvec{A},\myvec{x}_j\in \myvec{B}}{\sum}\widehat{P}_{i,j}.
\end{equation} 
The idea is to remove the probability of ``staying'' at a specific node from the within-class transition probability.
Now let
\begin{equation}
\rho_{P}\defeq {\rm GCcut}(\myvec{C}_1,...,\myvec{C}_{N_C})\defeq\frac{1}{N}\sum_{\ell=1}^{N_C}{\rm Gcut}(\myvec{C}_\ell,{\myvec{C}_\ell}).
\end{equation} We search for $\epsilon$ which maximizes $\rho_P$, namely
\begin{equation}
{\epsilon}_{\rho_P}=\arg\max_\epsilon(\rho_P).
\label{eq:rhop}
\end{equation} 

By the stochastic model, the implied probability of transition between point $\myvec{x}_i $ and point $\myvec{x}_j$ is equal to ${{p}(\myvec{x}_i,\myvec{x}_j)}=P_{i,j}$, therefore by maximizing $\rho_P$, the sum of within-class transition probabilities is maximized. Based on the definition of the diffusion distance (Eq.\ \ref{EqDist}), this implies that the within-class diffusion distance would be small, followed by a small Euclidean distance in the DM space. 
The heuristic approach entailed in Eq.\ (\ref{eq:rhop}) provides yet another criterion for setting a scale parameter which captures the geometry of the given classes. This approach does not require computing an eigendecomposition for each candidate $\epsilon$ value, thus, its computational complexity is of order of ${\cal{O}}(N^2)$.
\\
\section{Experimental Results}
\label{sec:Exp}
In this section we provide some experimental results, showing and comparing the different approaches (outlined in the previous section) in their respective contexts. We begin by demonstrating our proposed manifold-based scaling in subsection \ref{sec:ExpMani}, and then demonstrate the classification-based scaling approaches in subsection \ref{sec:ExpClass}.
\subsection{Manifold Learning}
\label{sec:ExpMani}
In this subsection we evaluate the performance of the proposed manifold-based approach by embedding a low-dimensional manifold which lies in a high-dimensional space. We consider two datasets: A synthetic set and a set based on an image taken from the MNIST database.
\subsubsection{A synthetic dataset}
The first experiment is constructed by projecting a $3$-dimensional synthetic manifold into a high-dimensional space, then concatenating it with Gaussian noise. Data generation is done according to the following steps: \begin{itemize}
\item First, a 3-dimensional Swiss Roll is constructed based on the following function \begin{equation}{
	\label{EQSwissYML}
	\myvec{y}_i=
	\begin{bmatrix}
	{y_i}(1)\\
	{y_i}(2)\\
	{y_i}(3)\\	
	\end{bmatrix}
	=
	\begin{bmatrix}
	{6\theta_i\cos(\theta_i)}\\
	{h_i}\\
	{6\theta_i\sin(\theta_i)}\\
	
	\end{bmatrix},i=1,...,N,
}
\end{equation} where $\theta_i,h_i$ ($i=1,...,N$), are drawn from Uniform distributions within the intervals $[\frac{3\cdot \pi}{2},\frac{9\cdot \pi}{2}], [0,100]$, respectively. In our experiment we chose $N=2000$.
\item We project the Swiss roll into a high-dimensional space by multiplying the data by a random matrix ${\myvec{N}_T} \in {{R}}^{D_1\times 3}, D_1>3$. The elements of $\myvec{N}_T$ are drawn from a Gaussian distribution with zero mean and variance of $\sigma^2_T$. 
\item Finally, we augment the projected Swiss Roll with a vector of Gaussian noise, obtaining
\begin{equation}
\label{EQSwissXML}
\myvec{x}_i=
\begin{bmatrix}
{\myvec{N}_T \cdot \myvec{y}}_i\\
\myvec{n}^1_i
\end{bmatrix},\;\;\;i=1,...,N,
\end{equation}  
where each component of $\myvec{n}^1_i \in \mathbb{R}^{D_2},i=1,...,2000,$ is an independent Gaussian variable with zero mean and variance of $\sigma^2_N$. 
\end{itemize}
We define the datasets $\mymat{Y}=[\myvec{y}_1,...,\myvec{y}_N]\in\mathbb{R}^{3\times N}$ and $\mymat{X}=[\myvec{x}_1,...,\myvec{x}_N]\in\mathbb{R}^{(D_1+D_2)\times N}$

\begin{figure}[H]
	\includegraphics[scale = 0.25]{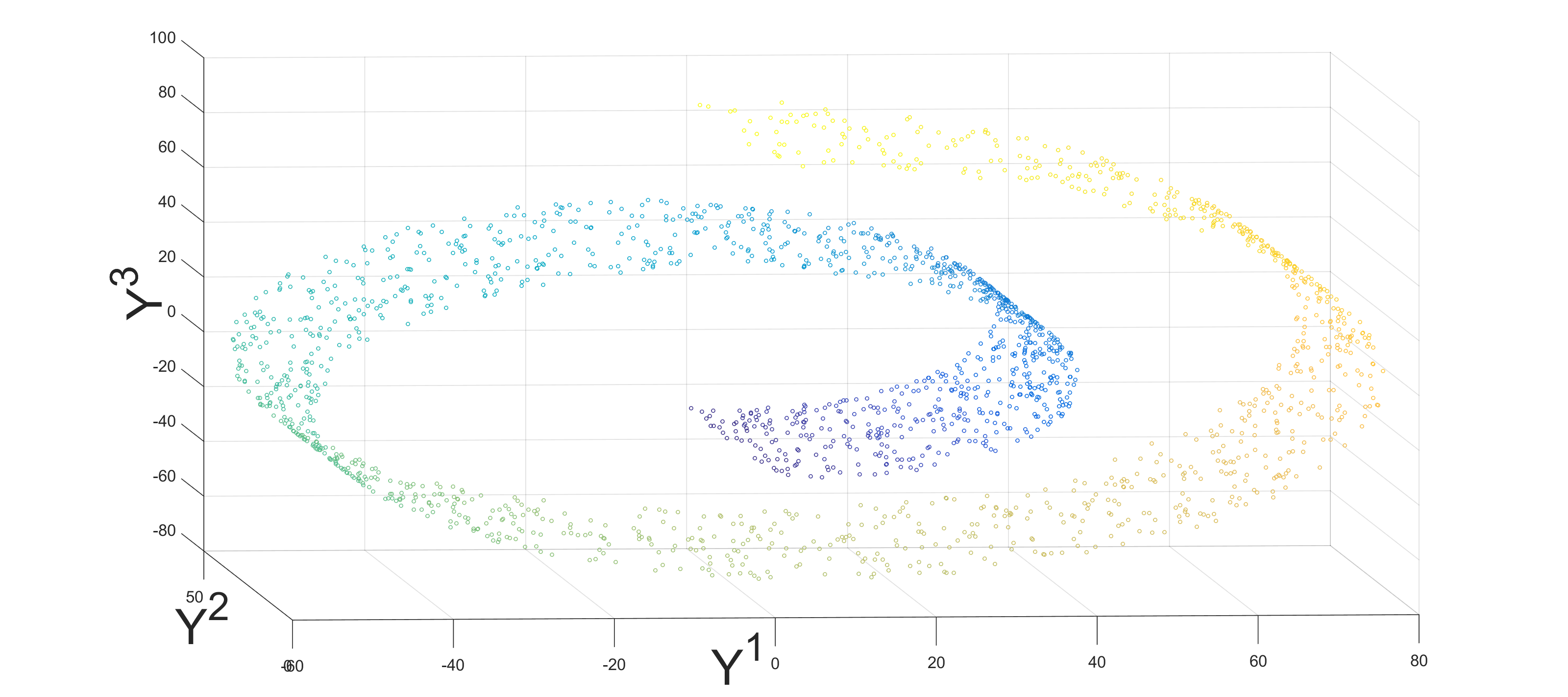} 
		\includegraphics[scale = 0.25]{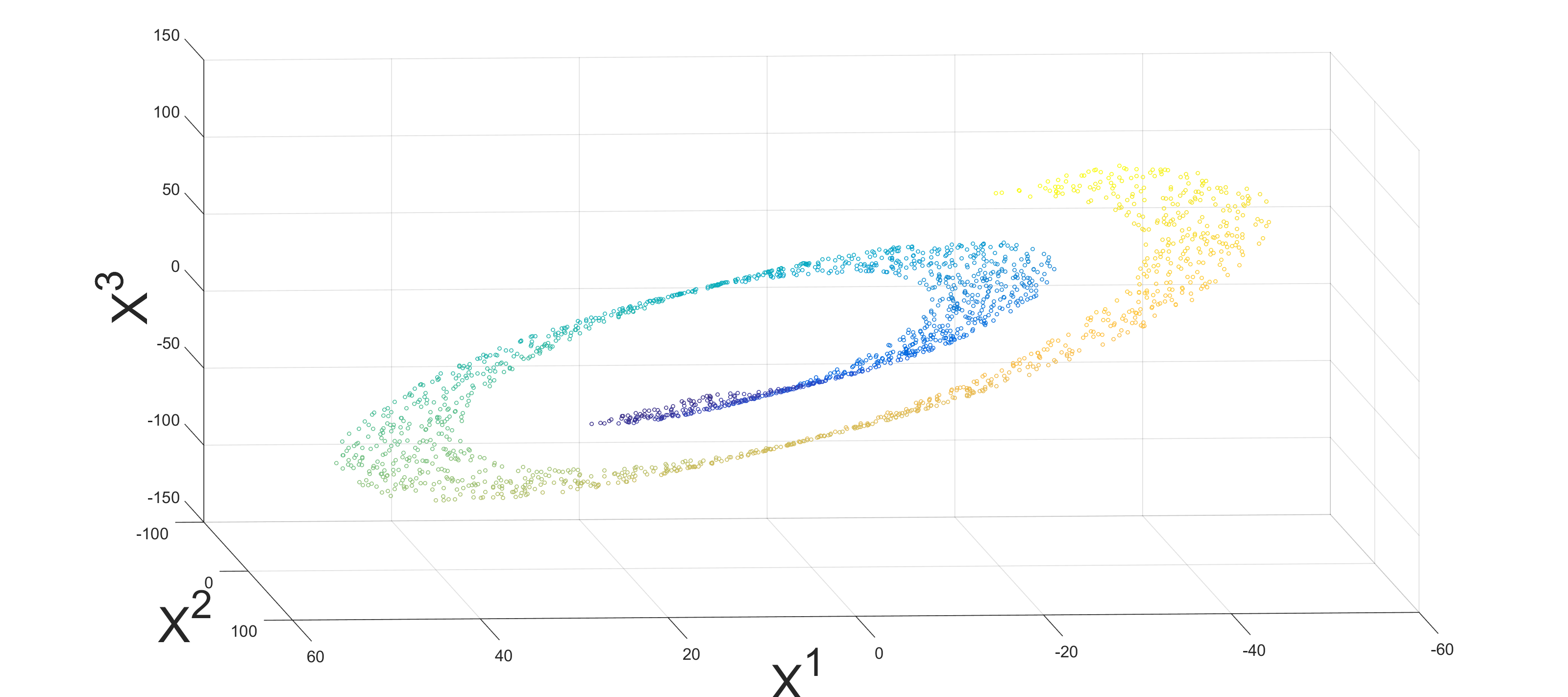} 
		\caption{Left: ``clean'' Swiss Roll ($\myvec{Y}$ in Eq.\ (\ref{EQSwissYML})). Right: 3 coordinates of the projected Swiss Roll ($\myvec{X}$ in Eq.\ (\ref{EQSwissXML})). Both figures are colored by the value of the underlying parameter $\theta_i,i=1,...,2000$ (Eq.\ (\ref{EQSwissYML})).}
		\label{SwissXY}
\end{figure}	
To evaluate the proposed framework, we apply Algorithm \ref{alg:Singer} followed by Algorithm \ref{alg:solv_global}, and extract a low-dimensional embedding. 
\begin{figure}[H]

	\includegraphics[scale = 0.25]{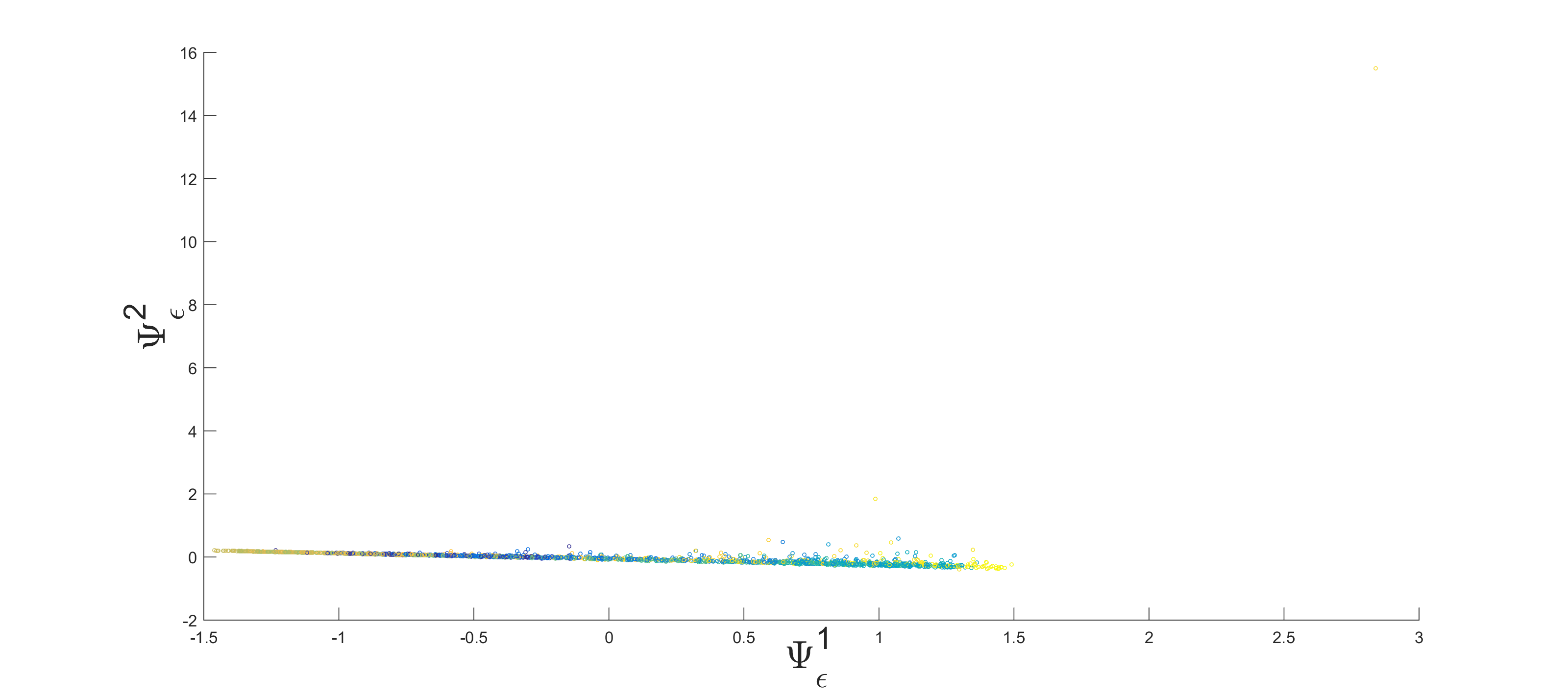}
	\includegraphics[scale = 0.25]{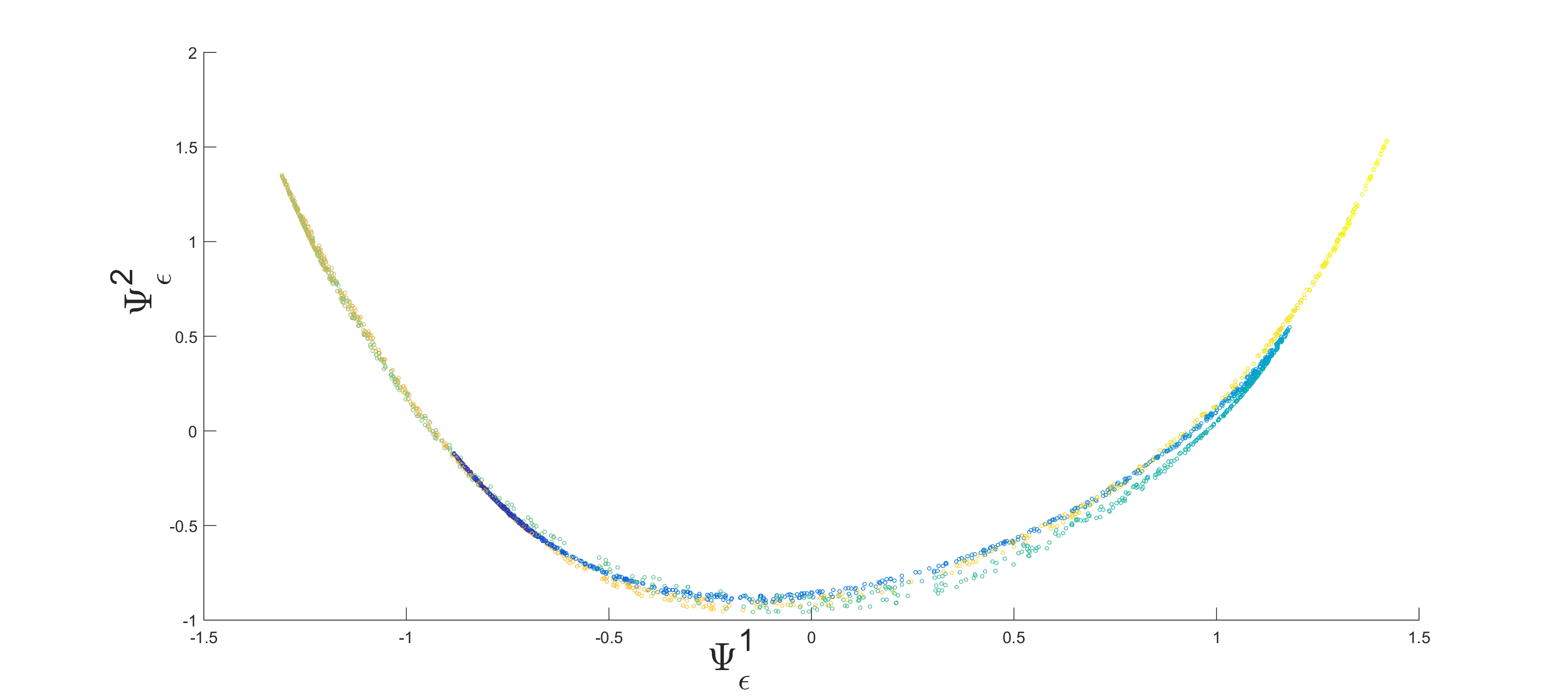}
	\includegraphics[scale = 0.25]{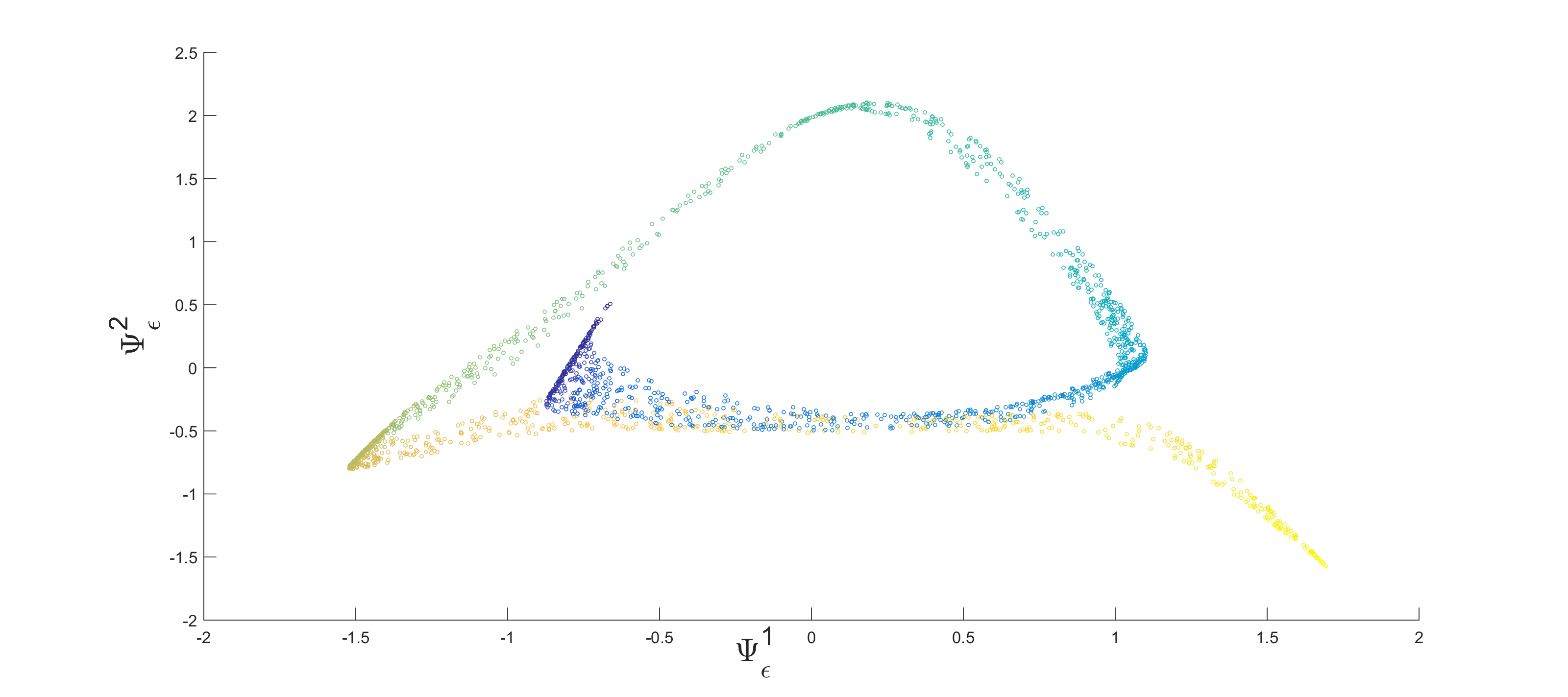}
		\includegraphics[scale = 0.25]{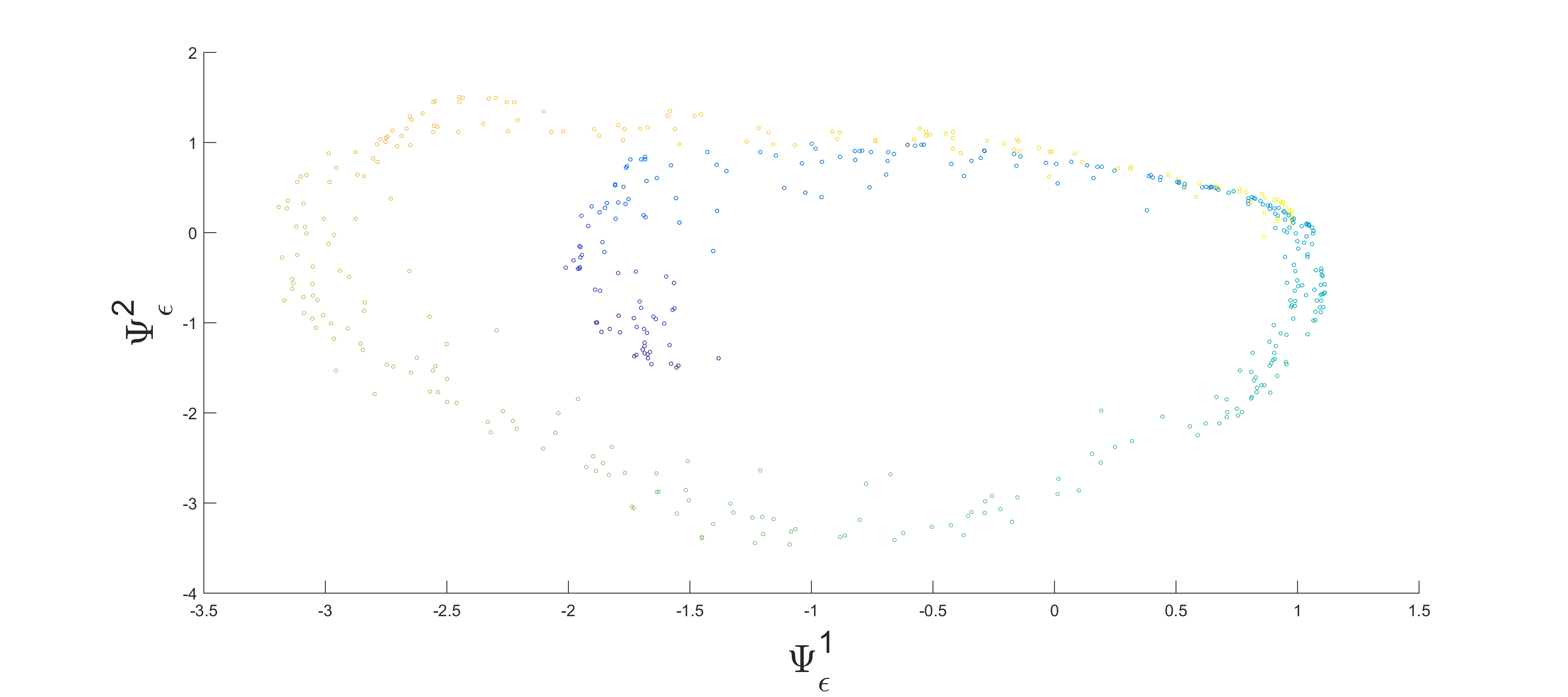}
		\includegraphics[scale = 0.25]{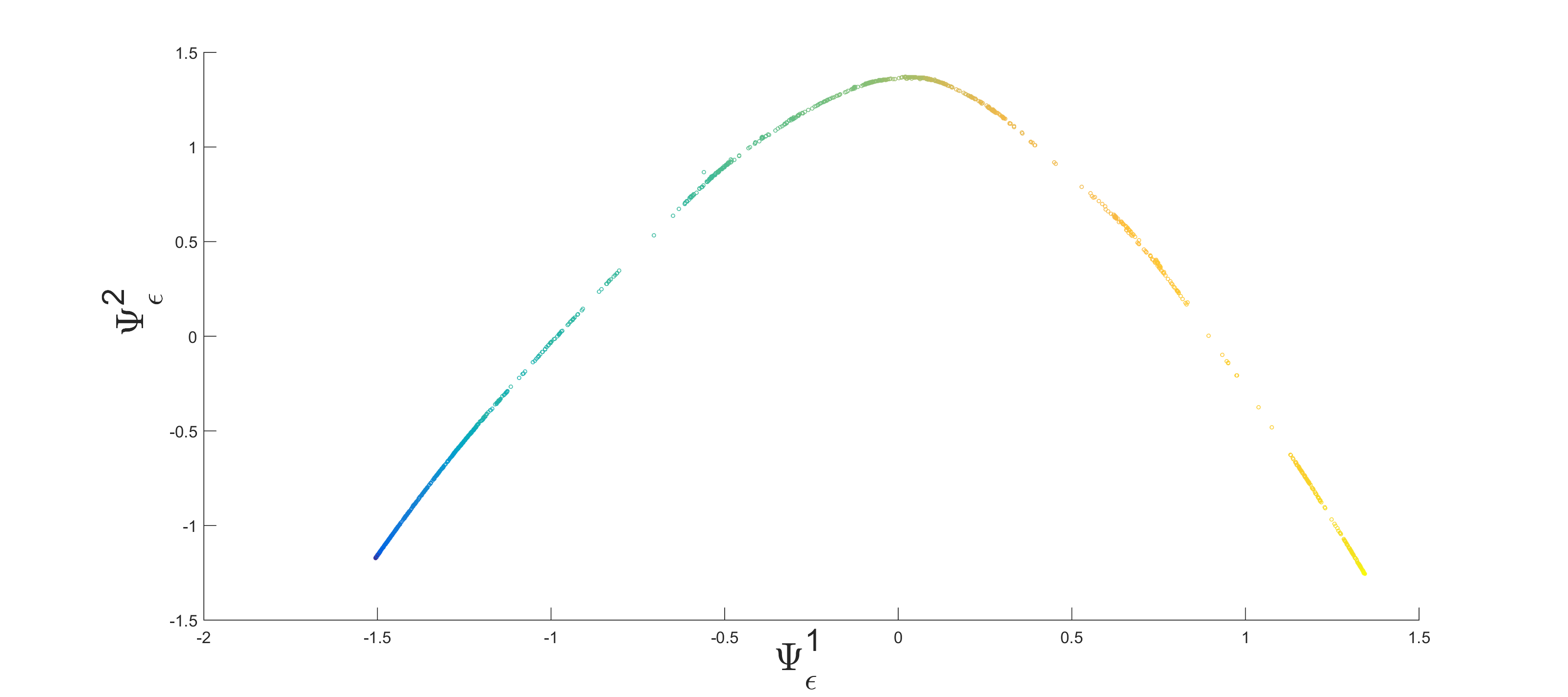}
		\includegraphics[scale = 0.25]{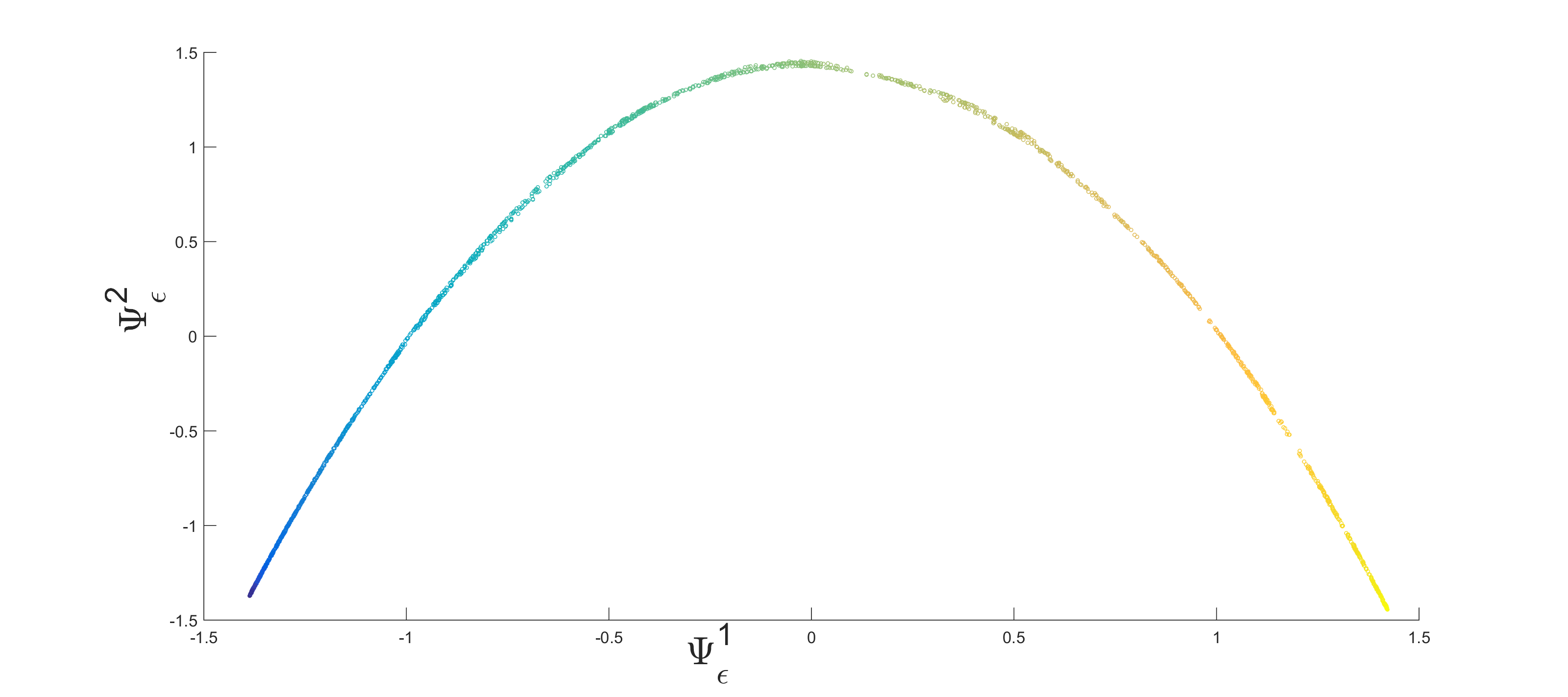}
	\caption{Extracted DM-based embedding of the ``noisy'' Swiss roll using different methods for choosing the scale parameter $\epsilon$. Top left: standard deviation scalings, the matrix $A \text{ and scaling }\epsilon_{std}$ are computed by Eq.\ (\ref{eq:StdN}). Top right: the $\epsilon_0$ scaling, the calculation of $\epsilon_0$ is described in Alg. (\ref{alg:Singer}) and \cite{Singer}. Mid left: the MaxMin scaling, the value $\epsilon_{\text{MaxMin}}$ is defined by Eq.\ (\ref{eq:MaxMin}). Mid right: k-NN based scaling \cite{zelnik}. Bottom left: the proposed scaling $\hat{\epsilon},\hat{A}$  which is described in Alg. (\ref{alg:solv_global}). Bottom right: scaling based on $\epsilon_0$ as described in Algorithm (\ref{alg:Singer}) and \cite{Singer} applied to the clean Swiss roll $Y$ that is defined by Eq.\ (\ref{EQSwissYML}).}	
	\label{fig:Swiss}
\end{figure}
Different high-dimensional datasets $\myvec{X}$ were generated using various values of $\sigma_{N}$, $\sigma_{N_T}$, $D_1$ and $D_2$. DM is applied to each $\myvec{X}$ using:
\begin{itemize}
	\item The standard deviation normalization as defined in Eq.\ (\ref{eq:StdN}).
	\item The $\epsilon_0$ scale, which is described in \ref{alg:Singer} and in \cite{Singer}.
	\item The MaxMin scale, as defined in Eq.\ (\ref{eq:MaxMin}) and in \cite{Keller}.
	\item The k-NN based scaling \cite{zelnik}.
	\item The proposed scale parameters $\mymat{A},\epsilon$, obtained using Algorithm \ref{alg:solv_global}. 
\end{itemize}  
The extracted embedding is compared to the embedding extracted from the clean Swiss roll $\myvec{Y}$ defined in Eq.\ (\ref{EQSwissYML}). 
\begin{figure}[H]
	\centering
	\includegraphics[scale = 0.4]{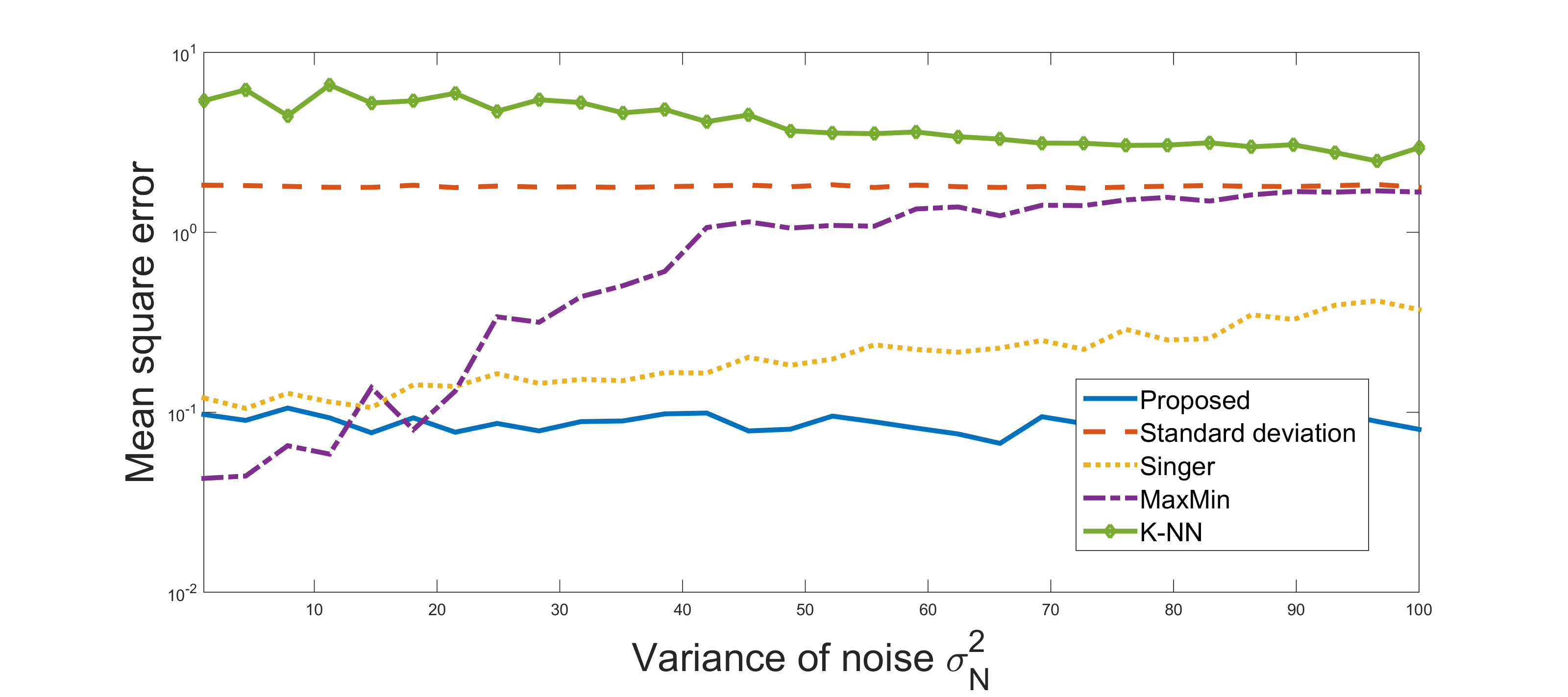}
	\caption{The mean square error of the extracted embedding. A comparison between the proposed normalization and alternative methods which are detailed in Section \ref{sec:EpsEst}.}
		\label{fig:MSE1}
\end{figure}
Each embedding is computed using an eigendecomposition, therefore, the embedding's coordinates could be the same up to scaling and rotation. To overcome this ambiguity, we search for an optimal translation and rotation matrix of the following form
\begin{equation}
\bar{\myvec{\Psi}}_{\epsilon}(\myvec{X})=\myvec{R} \cdot \myvec{\Psi}_{\epsilon}(\myvec{X})+\myvec{T},
\end{equation} where $\myvec{R}$ is the rotation matrix and $\myvec{T}$ is the translation matrix, which minimizes the mis-match error
\begin{equation}
\text{err}=||\bar{\myvec{\Psi}}_{\epsilon}(\myvec{X})-\myvec{\Psi}_{\epsilon}(\myvec{Y})||_F^2,
\end{equation} 
defined as the sum of square distances between values of the clean mapping $\myvec{\Psi}_{\epsilon}(\myvec{Y})$ and the ``aligned'' mapping $\bar{\myvec{\Psi}}_{\epsilon}(\myvec{X})$.
We repeat the experiment 40 times and compute the empirical Mean Square error in the embedding space defined as
\begin{equation}
\text{MSE}=\frac{1}{N}\sum_{i=1}^{N}(\myvec{\Psi}_{\epsilon}(\myvec{y}_i)-\bar{\myvec{\Psi}}_{\epsilon}(\myvec{x}_i))^2.
\end{equation}
An example of the extracted embedding based on all the different methods is presented in Fig.\ \ref{fig:Swiss}, followed by the MSE in Fig.\ \ref{fig:MSE1}. It is evident that Algorithm \ref{alg:solv_global} is able to extract a more precise embedding than the alternative scaling schemes. The strength of Algorithm \ref{alg:solv_global} is that it emphasizes the coordinates which are essential for the embedding and neglects the coordinates which were contaminated by noise.

\subsubsection{MNIST Manifold}
In the following experiment, we create an artificial low-dimensional manifold by rotating a handwritten image of a digit. First, we rotate the handwritten digit `6' from MNIST dataset by $N=320$ angles that are uniformly sampled over $[0,2\pi]$. Next, we add random zero-mean Gaussian noise with a variance of $\sigma_N^2$ independently to each pixel. An example of the original and noisy version of the handwritten `6' are shown in Fig.\ \ref{fig:mnist}. Note that the values of the original image are in the range $[0,1]$. In order to capture the circular structure of the manifold we apply DM to the rotated images. An example of the expected circular structure extracted by DM is depicted in Fig.\ \ref{fig:mnist}. 

We apply the different scaling schemes to the noisy images and extract a 2-dimensional DM-based embedding. For the vectorized scaling schemes (proposed and standard deviation approach) we apply the scalings to the top 50 principle components. This allows us to reduce the computational complexity, as the dimension of the feature space is reduced from $784$ to $50$.

To evaluate the performance of the different scaling schemes, we propose the following metric to compare the extracted embedding to a perfect circle. Given a $2$-dimensional representation $\myvec{\Psi}_{\epsilon}(\myvec{X})$, we use a polar transformation to evaluate the implied radius at each point. The squared radius is defined by $r^2_i=(\myvec{\Psi}^1_{\epsilon}(\myvec{x}_i))^2+(\myvec{\Psi}^2_{\epsilon}(\myvec{x}_i))^2$. Next, we normalize the radius values by their empirical mean, such that $\hat{r}_i=\frac{r_i}{\mu_r}$, and $\mu_r$ is the empirical mean of the radii $\mu_r=\sum_i\frac{r_i}{N}$. Finally, we compute the empirical variance of the normalized radius $\hat{r}$. Explicitly, this value is computed by
$ \sigma^2_{\hat{r}}= \frac{\sum_i(\hat{r}_i -1)^2}{N}  $. A scatter plot of $\sigma^2_{\hat{r}}$ vs.\ the variance of the additive noise $\sigma_N^2$ is presented in Fig.\ \ref{fig:mnist_error}. As evident in this figure, up to a certain variance of the noise, the proposed scaling scheme suppresses the noise and captures the correct circular structure of the data. At some level of noise our method breaks. It seems that the standard deviation and Singer's approach also break at a similar noise level. An explanation for this phenomenon could be that at the lower SNRs all these methods start to ``amplify'' the noise, rather than the signal.
	\begin{figure}[H]
		\centering
		\includegraphics[scale = 0.26]{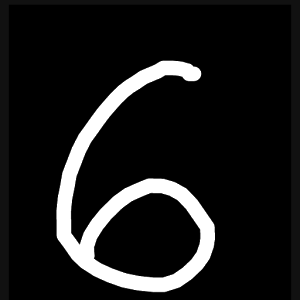}
		\hspace{5cm}
		\includegraphics[scale = 0.26]{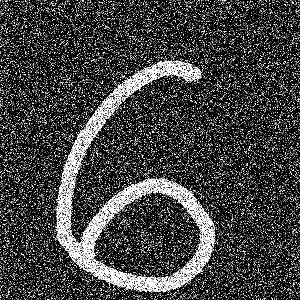}
		
		\includegraphics[scale = 0.15]{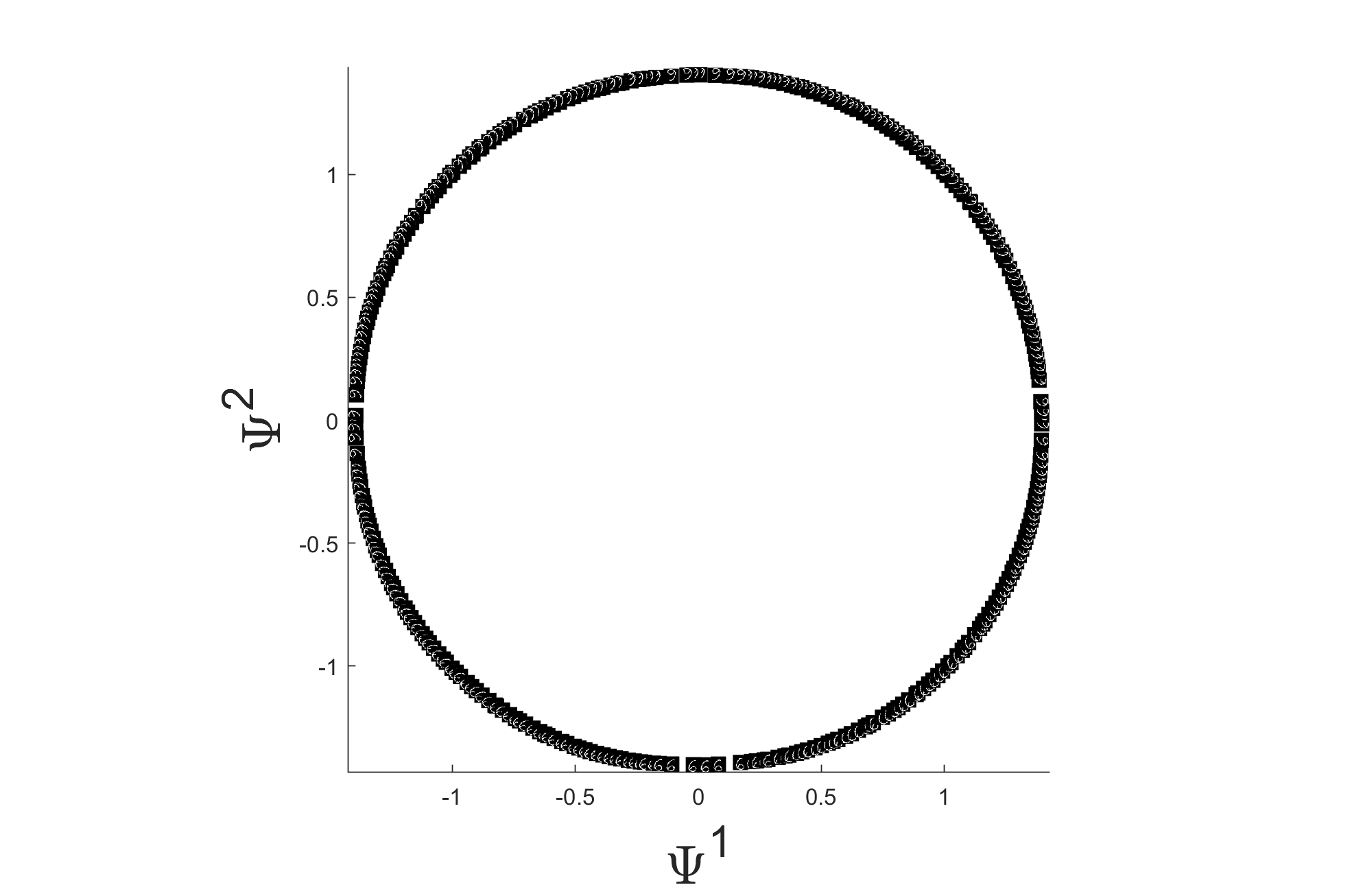}
		\includegraphics[scale = 0.15]{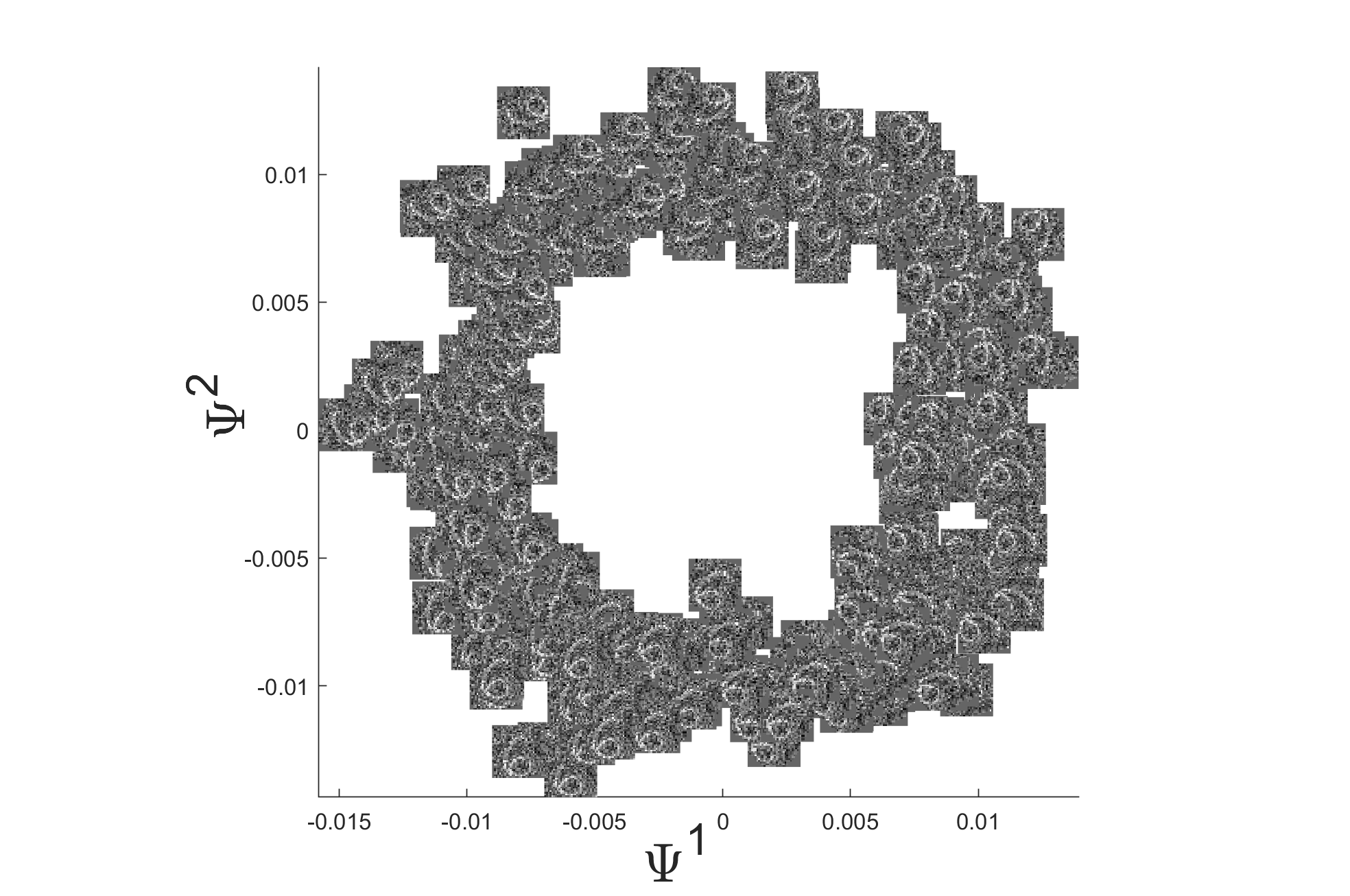}
		
		\caption{Top left: an example of a clean handwritten digit of `6'. Top right: a noisy example of the digit `6'. Each pixel is added by a Gaussian noise. The noise is i.i.d. drawn from $N(0,0.5)$. Bottom left: extracted DM-based embedding from 320 rotated images of the ``clean'' handwritten digit. Bottom right:  extracted DM-based embedding from 320 rotated images of the ``noisy'' handwritten digit.}
		\label{fig:mnist}
	\end{figure}

		\begin{figure}[H]
		\centering

		\includegraphics[scale = 0.3]{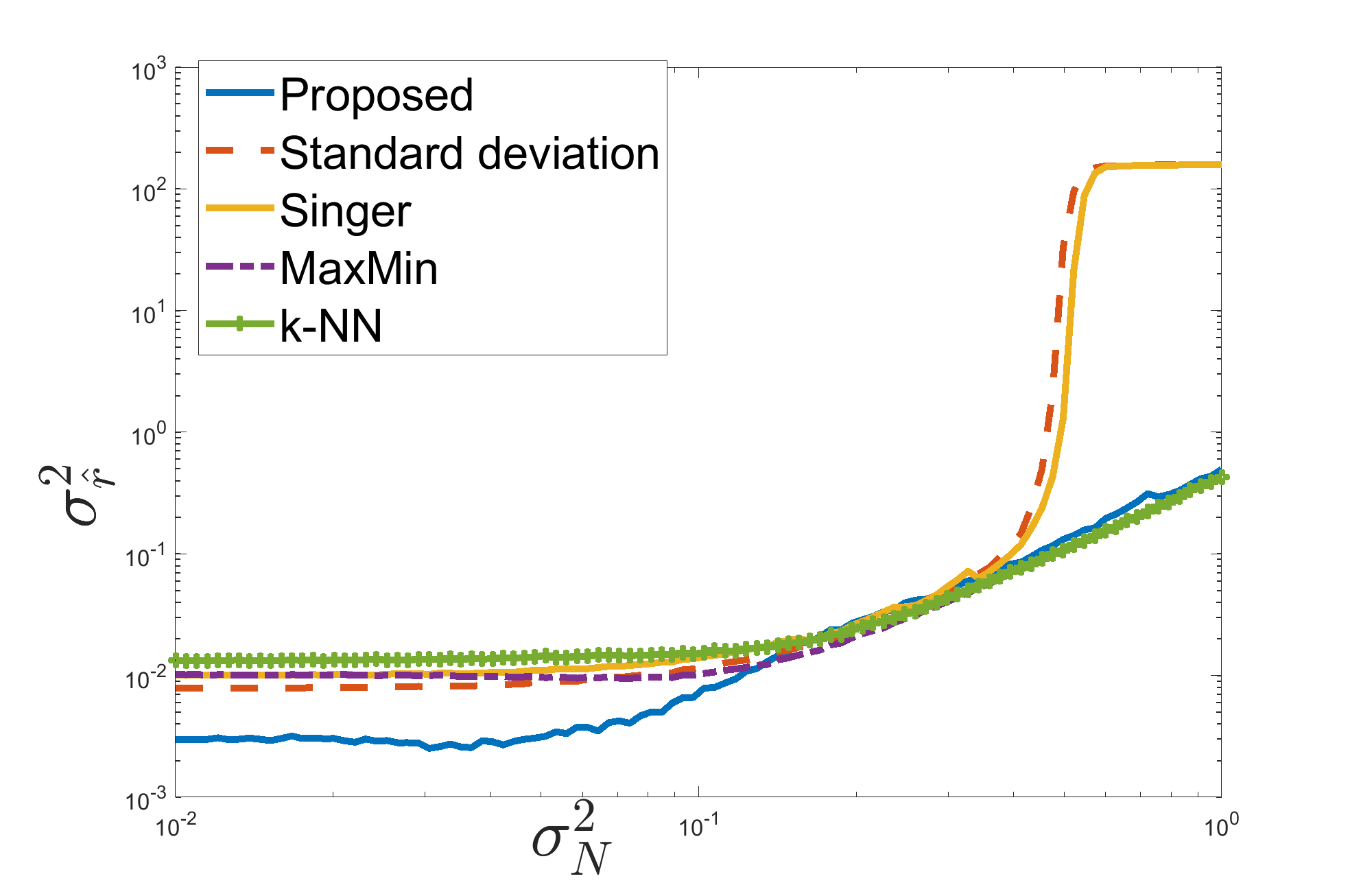}
		
		\caption{The normalized radius variance (NRV) of the extracted embedding from the noisy rotated digit manifold. A comparison between the proposed normalization and alternative methods which are detailed in Section \ref{sec:EpsEst}}
		\label{fig:mnist_error}
	\end{figure}
	
\subsection{Classification}
\label{sec:ExpClass}
In this subsection we provide empirical support for the theoretical analysis from Section (\ref{sec:OptClass}). We evaluate the influence of $\epsilon$ on the classification results using four datasets: a mixture of Gaussians, artificial classes lying on a manifold, handwritten digits and seismic recordings. We focus on evaluating how the proposed measures $\rho_P,\rho_{\myvec{\Psi}},\text{Ge}$ (Eqs.\ (\ref{eq:rhop}), (\ref{eq:rhopsi}), (\ref{eq:GE}), resp.) are correlated with the quality of the classification. 

\subsubsection{{Classification of a Gaussian Mixture}}
In the following experiment we focus on a simple classification test using a mixture of Gaussians. We generate two classes using two Gaussians, based on the following steps:
\begin{enumerate}
	\item{Two vectors $  \myvec{\mu}_1\text{ and }\myvec{\mu}_2 \in {\mathbb{R}}^{6}$ were drawn from a Gaussian distribution $N({0},\sigma_M\cdot {\myvec{I}_{6\times 6}})$. These vectors are the centers of masses for the generated classes $\myvec{C}_1$ and $\myvec{C}_2$. (resp.).}
	\item{$N=100$ data points were drawn for each class $\myvec{C}_1\text{ and }\myvec{C}_2$ with a Gaussian distribution
		$N(\myvec{\mu}_1,\sigma_V\cdot {\myvec{I}_{6\times 6}})\text{ and } N(\myvec{\mu}_2,\sigma_V\cdot {\myvec{I}_{6\times 6}})$, respectively. Denote these $2N$ data points by ${\myvec{C}_1 \cup \myvec{C}_2=\myvec{T}\subset \mathbb{R}^{6\times 200}}$.}
	\end{enumerate} 
	\begin{figure}[H]
		\includegraphics[scale = 0.26]{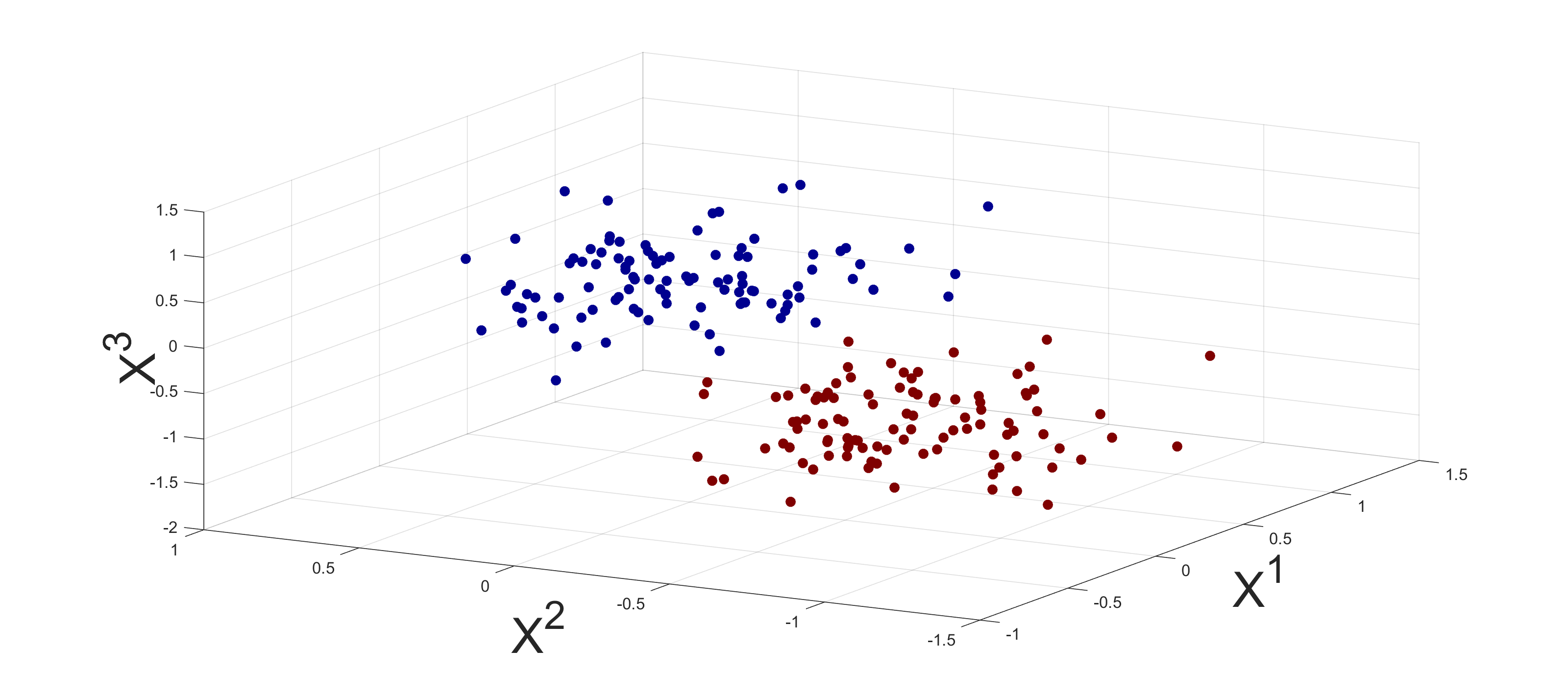}
		\includegraphics[scale = 0.11]{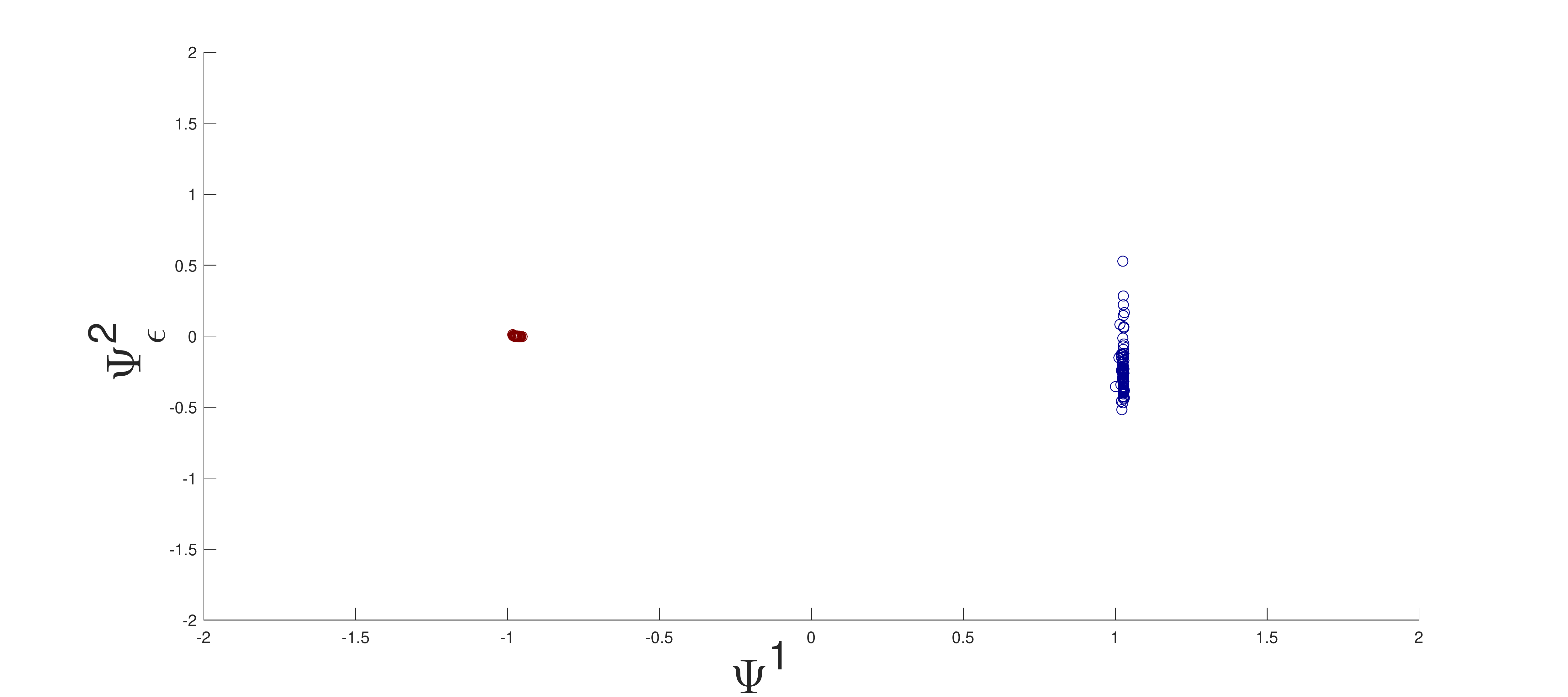}
		
		\caption{Left: an example of the Gaussian distributed data points. Right: a 2-dimensional mapping of the data points.}
		\label{fig:Gauss}
	\end{figure}
	
The first experiment evaluates the Spectral Approach (Section \ref{sec:spectral}). Therefore, we set $\sigma_v<\sigma_M$ such that the class variance is smaller than the variance of the center of mass. Then, we apply DM using a scale parameter $\epsilon$ such that $\epsilon \sim \sigma_v^2<\sigma_M^2$. In Fig \ref{fig:Ideal} (left), we present the first extracted diffusion coordinate using various values of $\epsilon$. It is evident that the separation between classes is highly influenced by $\epsilon$. A comparison between $\rho_P$,$\rho_{\Psi}$ and $Ge$ is presented in Fig.\ \ref{fig:Ideal} (right). This comparison provides evidence of the high correlation between $\rho_P$ (Eq.\ (\ref{eq:rhop})), $\rho_{\Psi}$ (Eq.\ (\ref{eq:rhopsi})) and the generalized eigengap (Eq.\ (\ref{eq:GE})).
\begin{figure}[H]
	\includegraphics[scale = 0.3]{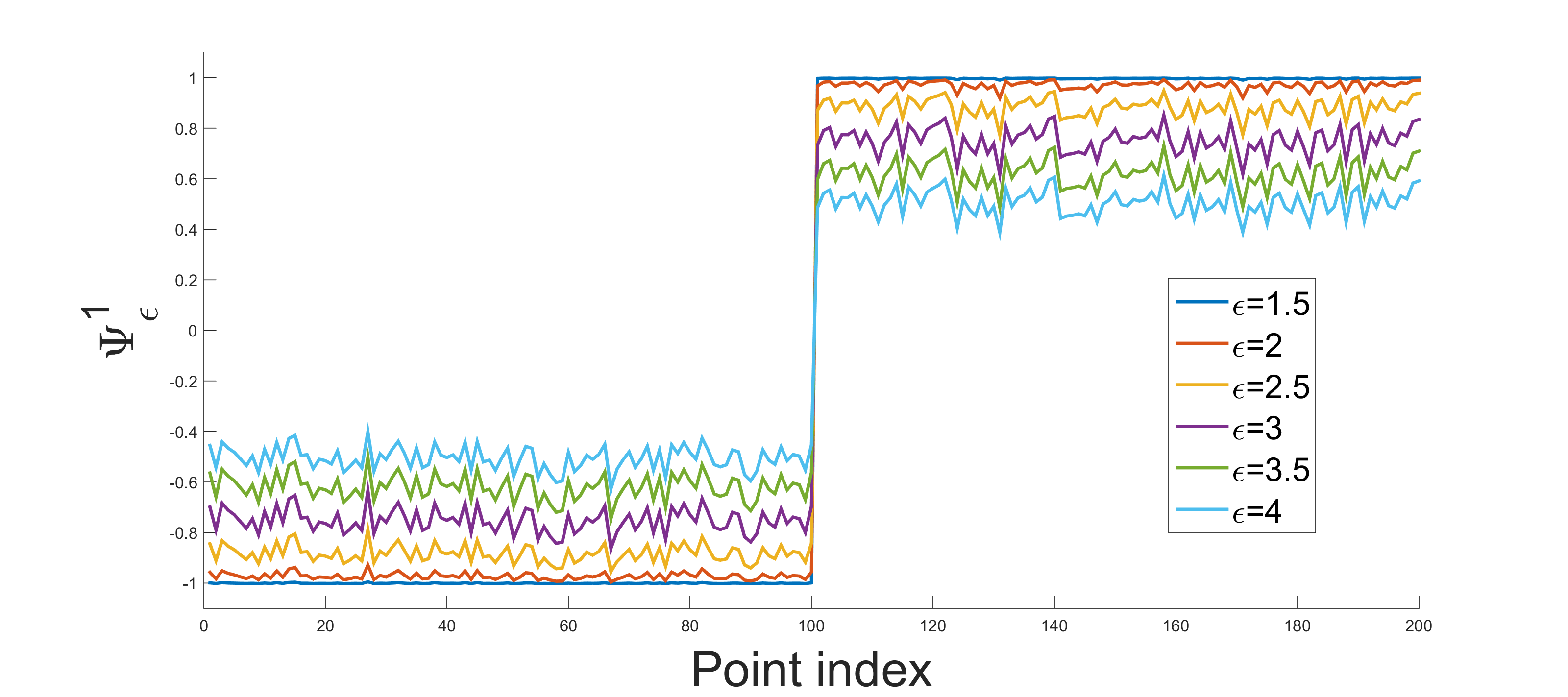}
	\includegraphics[scale = 0.32]{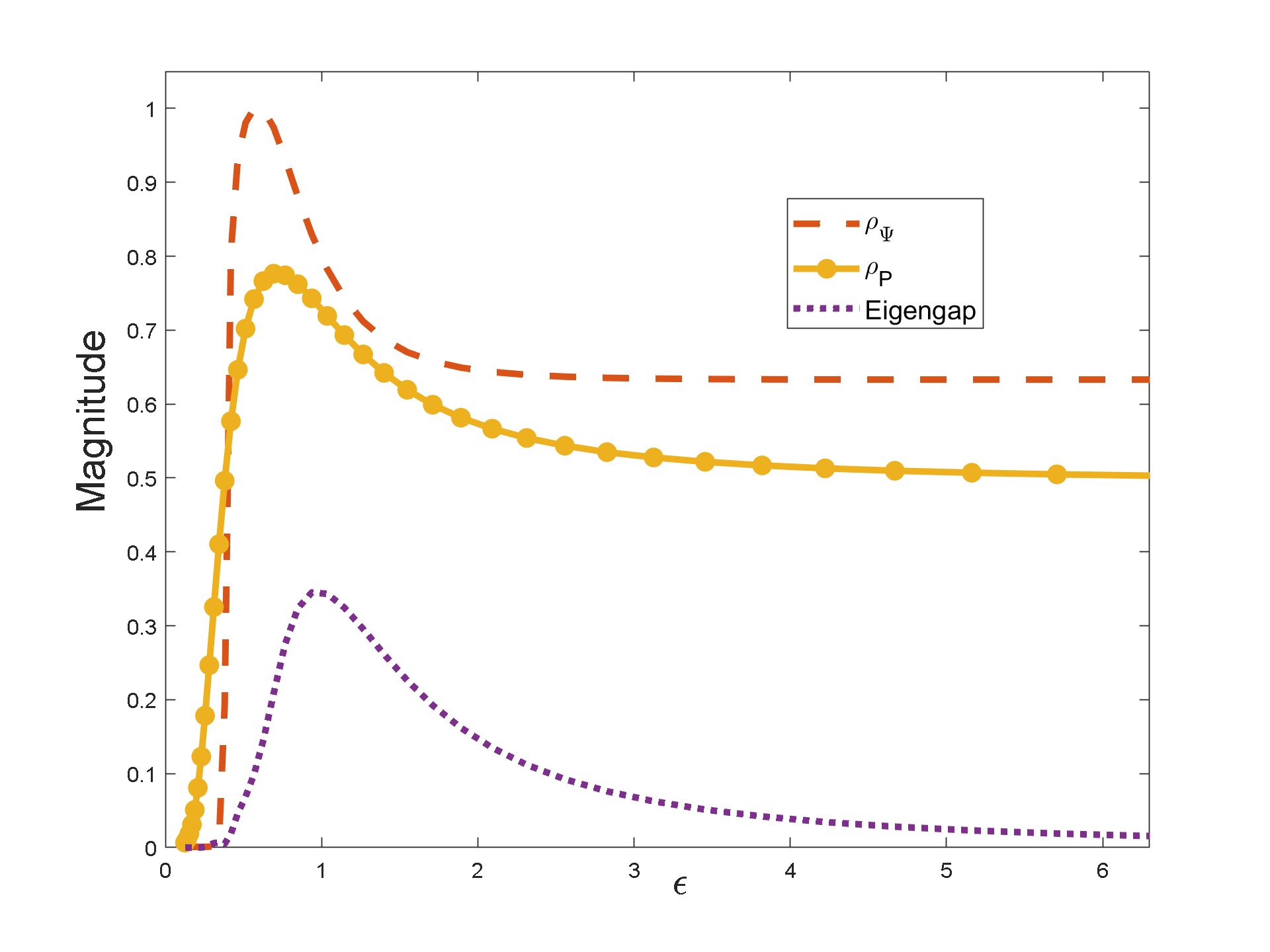}

	\caption{Left: the first eigenvector $\myvec{\Psi}^1$ computed for various values of $\epsilon$. Right: a comparison between $\rho_P,\rho_{\Psi}$ and $Ge$.}
	\label{fig:Ideal}
\end{figure}

To evaluate the validity of Assumption \ref{assump1}, we calculate the Frobenius norm of the perturbation matrix $\widehat{W}$ for various values of $\epsilon$. The results with the approximated $\epsilon_{Ge}$ are presented in Fig.\ \ref{fig:Wpert}. Indeed, as evident from Fig.\ \ref{fig:Wpert}, the value of $||\widehat{W}||_F$ is nearly constant for a small range of values around $\epsilon_{Ge}$.
\begin{figure}[H]
	\centering
	\includegraphics[scale = 0.25]{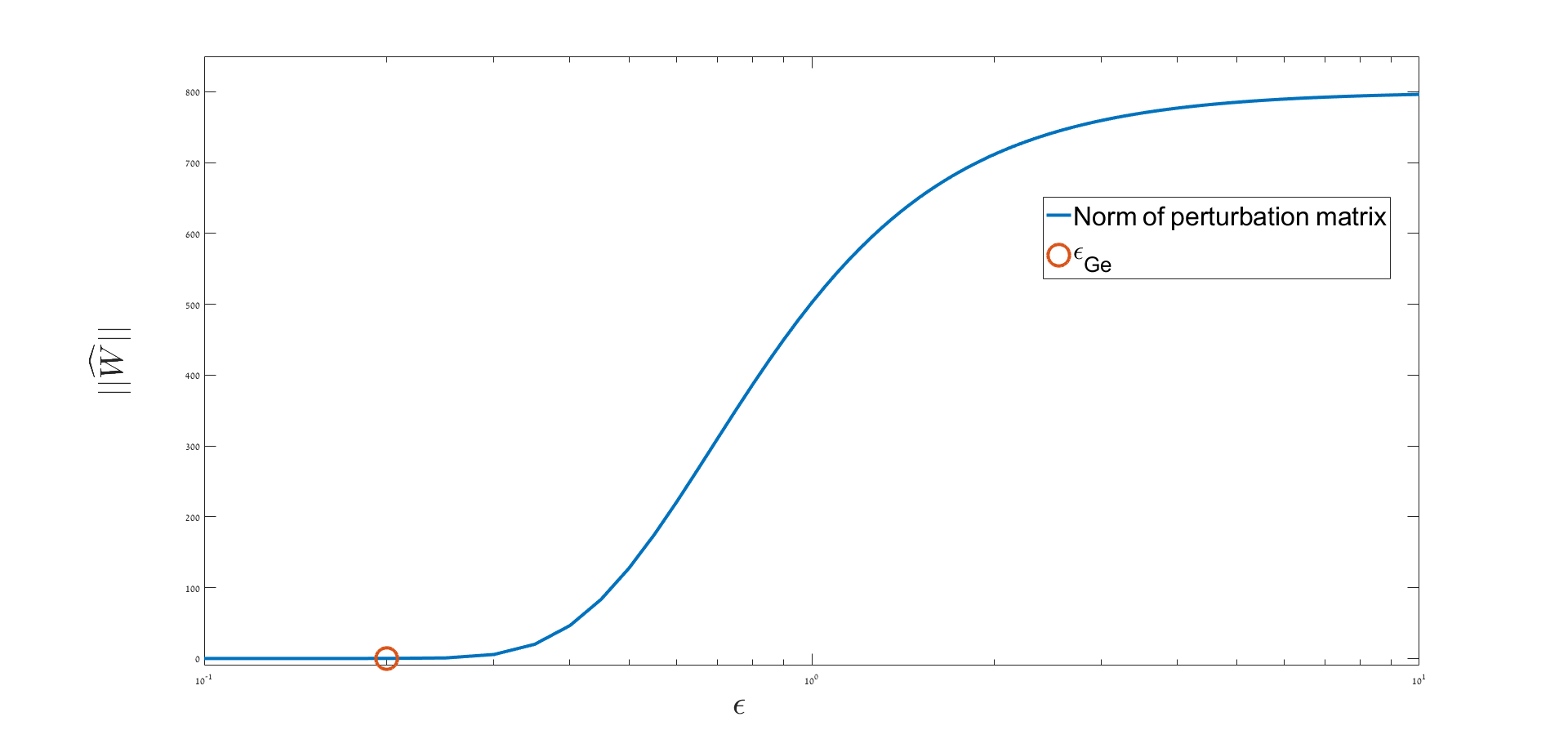}
	
	\caption{The Frobenius norm of the perturbation matrix $\widehat{W}$. The annotated point is the approximated scale $\epsilon_{Ge}$}
	\label{fig:Wpert}
\end{figure}

\subsubsection{Classes based on an artificial physical process}
For the non-ideal case, we generate classes using a non-linear function. This non-linear function is designed to model an unknown underlying nonlinear physical process governed by a small number of parameters. Consequently, the classification task is essentially expected to provide an estimate of these hidden parameters. An example for such a problem is studied, e.g., in \cite{lindenbaum2015musical}, where a musical key is estimated by applying a classifier to a low-dimensional representation extracted from the raw audio signals. In the following steps we describe how we generate classes from a Spiral structure:
\begin{enumerate}
	\item{Set the number of classes $N_C$ and a gap parameter $G$. Each class $\myvec{C}_\ell$,  $\ell=1,...,N_C,$ consists of $N_P$ data points drawn from a uniformly dense distribution within the line  $[(\ell-1)\cdot L_C, \ell\cdot L_C-G ]$, $\ell=1,...,N_C$. $L_C$ is the class-length, set as $L_C=\frac{1}{N_C}$. Let $N=N_CN_P$ denote the total number of points.}
	\item{Denote $\{r_i\}_{i=1}^N$ as the set of all points from all classes.}
	\item{Project each $r_i$ into the ambient space using the following spiral-like function
		\begin{equation}
		\label{eq:Spiral}
		\bar{\myvec{x}}_i=
		\begin{bmatrix}
		\bar{x}_i(1)\\
		\bar{x}_i(2)\\
		\bar{x}_i(3)\\
		
		\end{bmatrix}
		=
		\begin{bmatrix}
		{(6\pi r_i)\cos(6\pi r_i)}\\
		{(6\pi r_i)\sin(6\pi r_i)}\\
		{r_i^3-r_i^2}\\
		
		\end{bmatrix}
		+\myvec{n}^2_i,
		\end{equation}}
\end{enumerate} 
where $\myvec{n}^2_i\in \mathbb{R}^3$ are drawn independently from a zero-mean Gaussian distribution with covariance $\myvec{\Lambda}=\sigma_S\cdot \myvec{I}_3$.
Two examples of the spiral-based classes are shown in Fig.\ \ref{fig:Spiral1}. For both examples, we use $N_C=4,N_P=100,\sigma_S=0.4$ with different values for the gap parameter $G$.
\begin{figure}[H]
   \centering
   \begin{subfigure}
   	\centering
   	\includegraphics[scale = 0.09]{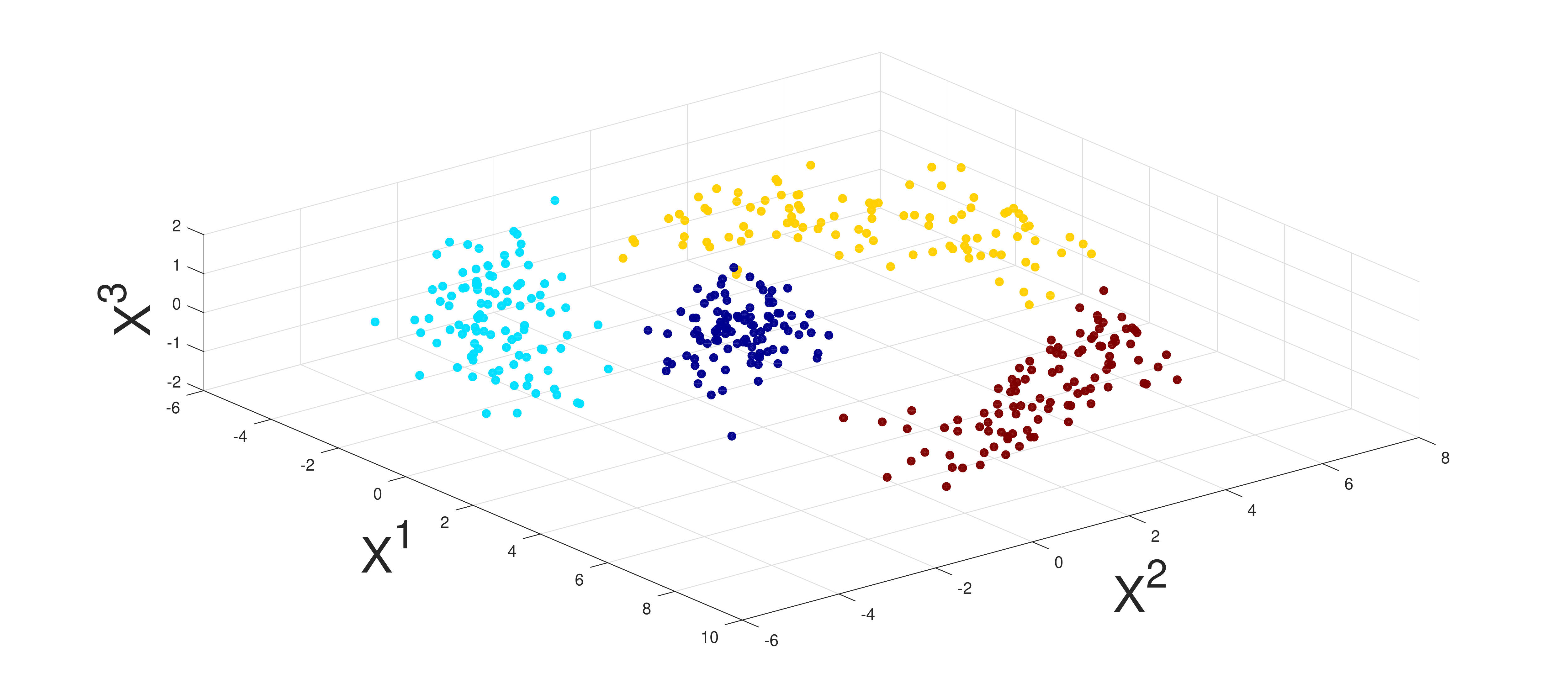}
   	\label{fig:gull}
   \end{subfigure}%
        \begin{subfigure}
        	\centering
	\includegraphics[scale = 0.09]{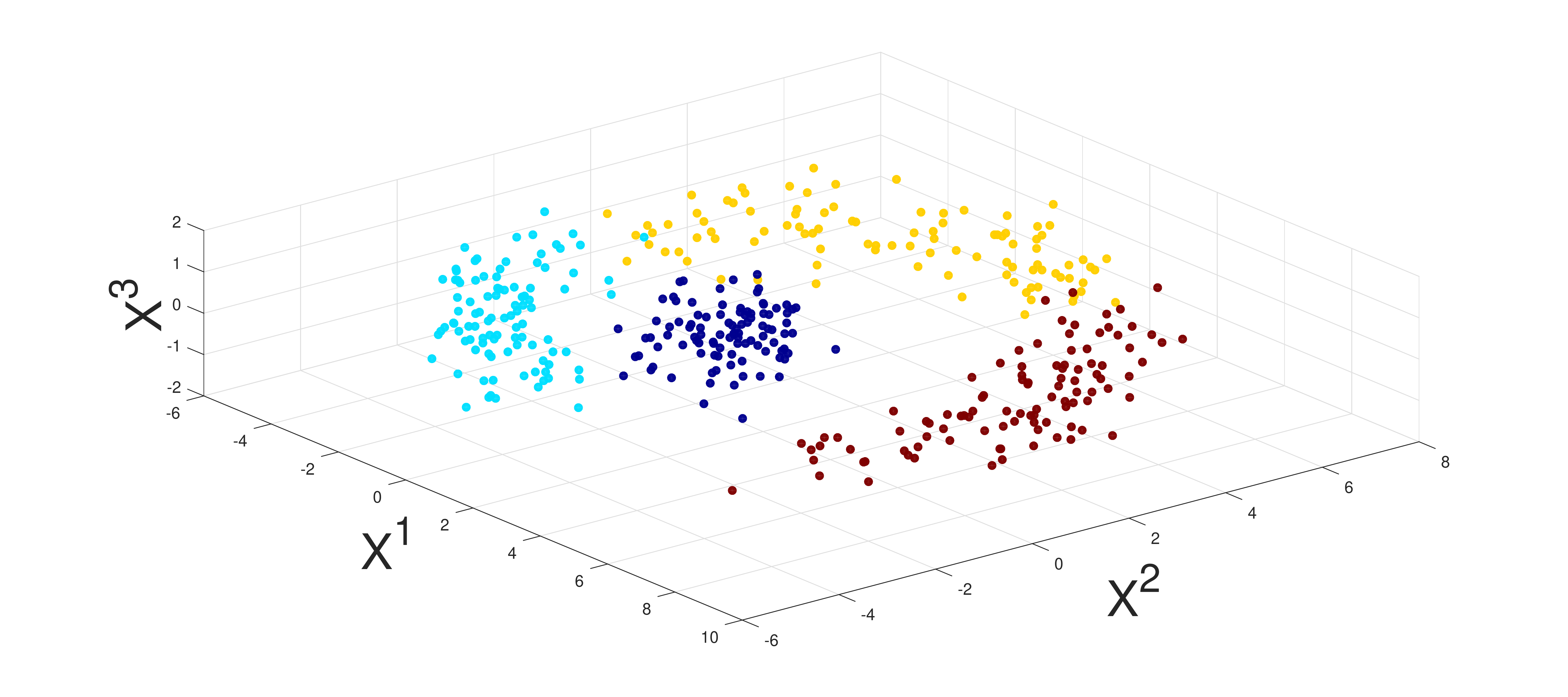}
        	\label{fig:gull2}
        \end{subfigure}

	\caption{Two examples of the generated three-dimensional spiral that are based on Eq.\ (\ref{eq:Spiral}) using $N_C=4$ classes with $N_P=100$ data points within each class. The gaps are set to be $G=0.02,0.04$ left and right, respectively.}
	\label{fig:Spiral1}
\end{figure}

\begin{figure}[H]
	\includegraphics[scale = 0.24]{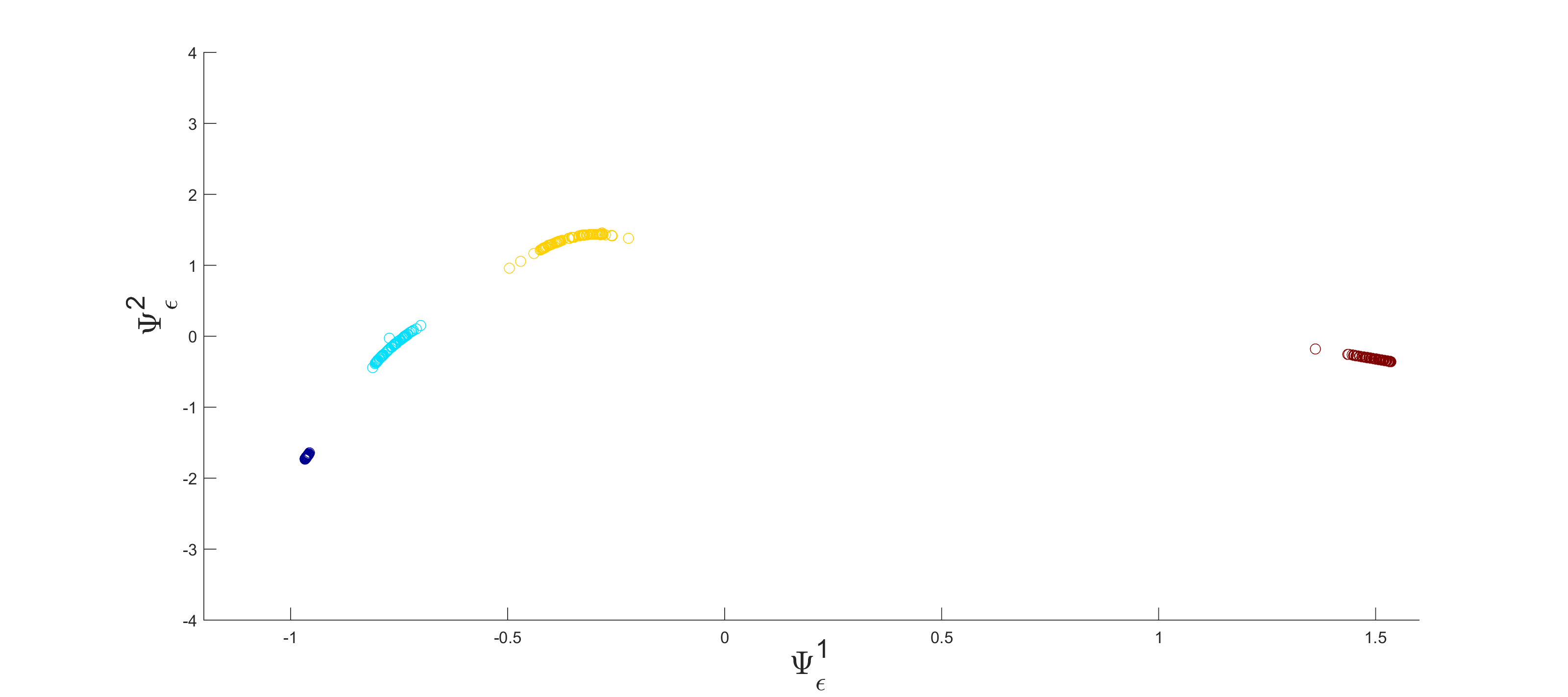}
	\includegraphics[scale = 0.24]{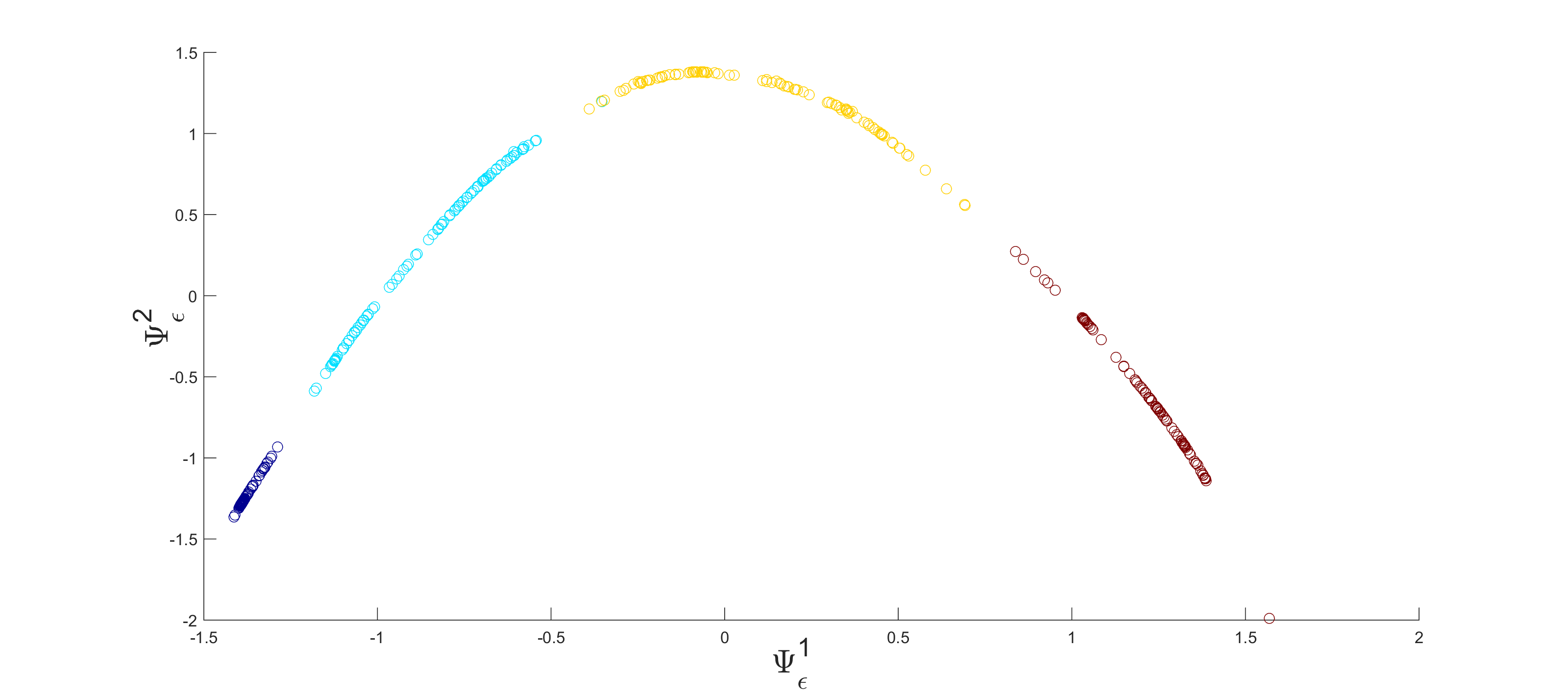}
	
	\caption{A 2-dimensional mapping extracted from both spirals presented in Fig.\ \ref{fig:Spiral1}.}
	\label{fig:Spiral2}
\end{figure}
To evaluate the advantage of the proposed scale parameters $\epsilon_{{\Psi}}$ and $\epsilon_{P}$ (Eqs.\ (\ref{eq:rhopsi}) and (\ref{eq:rhop}), resp.) for classification tasks, we calculate the ratios $\rho_P \text{ and }\rho_{\Psi}$ for various values of $\epsilon$, and then we evaluate the resulting classification (which is based on the low-dimensional embedding). Examples of embeddings of the two spirals from Fig.\ \ref{fig:Spiral1} are shown in Fig.\ \ref{fig:Spiral2}. This merely demonstrates the effect of $\epsilon$ on the quality of separation. 

We apply classification in the low-dimensional space using k-NN ($k=1$). The k-NN classifier is evaluated based on Leave-One-Out cross validation. The results are shown in Fig.\ \ref{fig:Spiral3}, where it is evident that the classification results in the ambient space are highly influenced by the scale parameter $\epsilon$. Furthermore, peak classification results occur at a value of $\epsilon$ corresponding to the maximal values of $\rho_{Ge}$ and $\rho_{\Psi}$. The value of $\rho_P$ did not indicate the peak classification scale, however, its computation complexity is lighter compared to $\rho_{Ge}$ and $\rho_{\Psi}$.

\begin{figure}[H]
	\includegraphics[scale = 0.26]{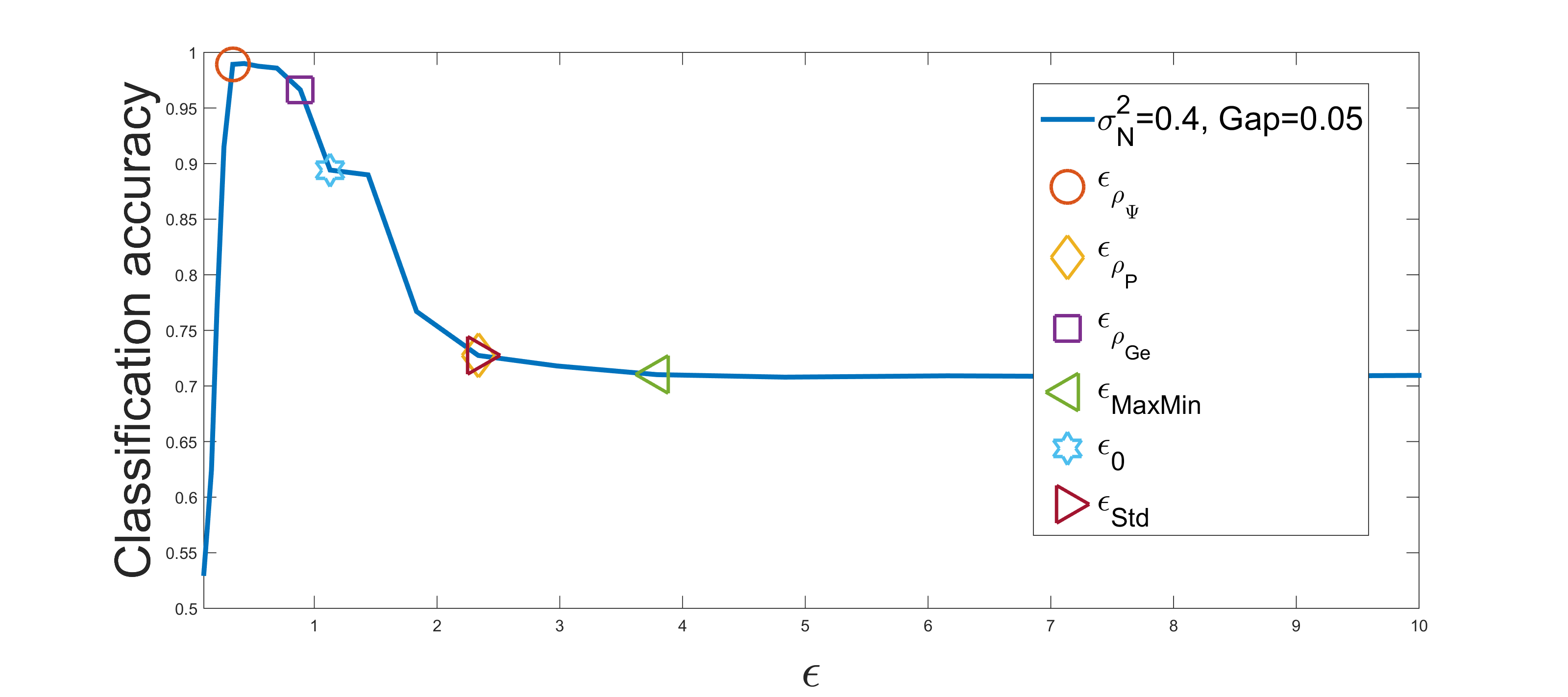}
	\includegraphics[scale = 0.26]{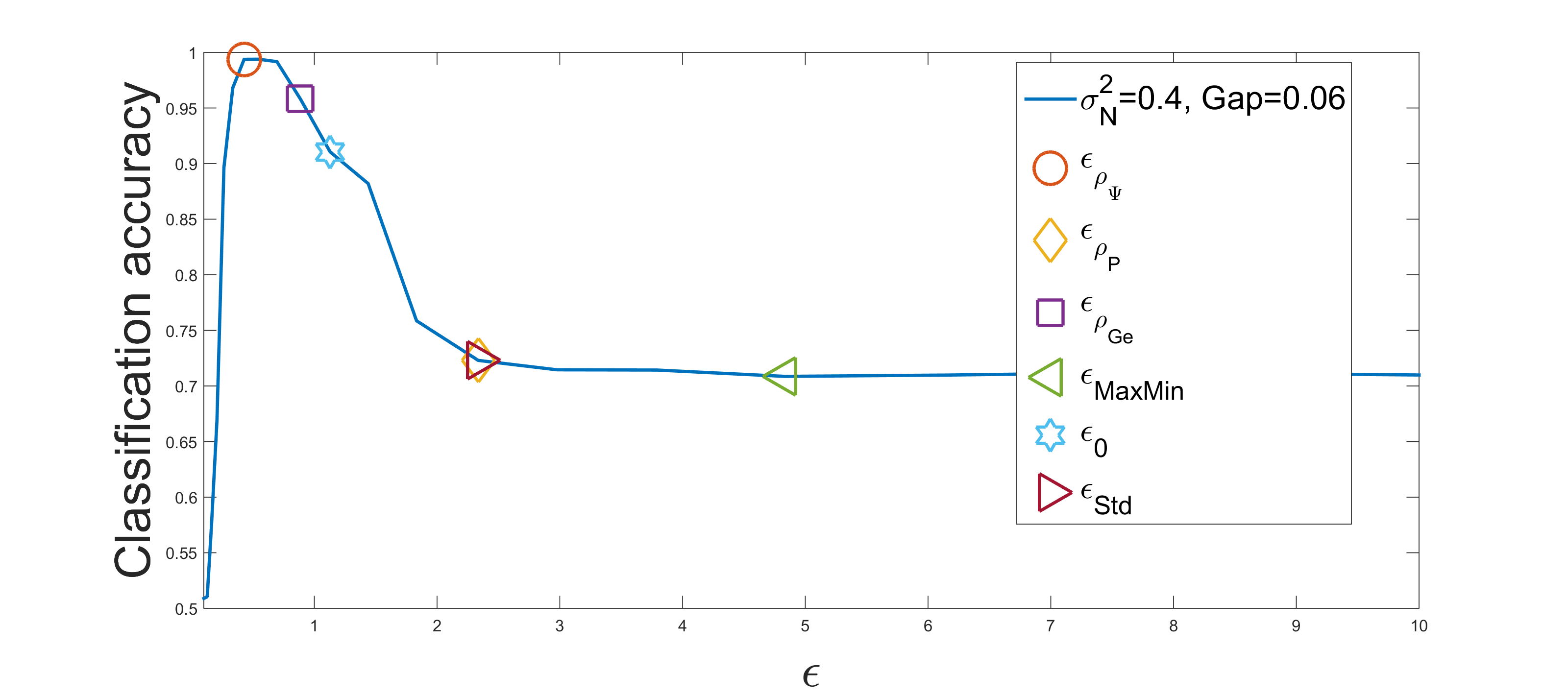}
	\includegraphics[scale = 0.26]{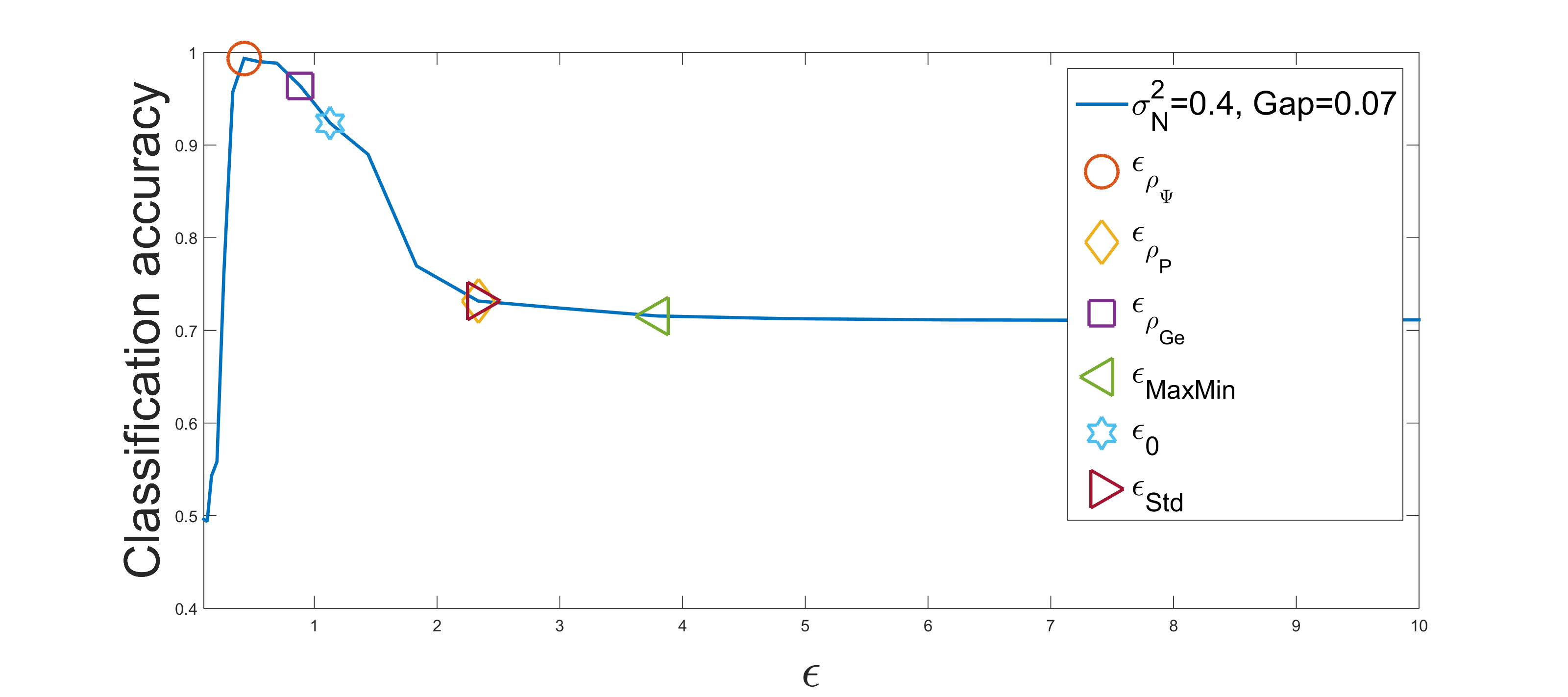}
	\includegraphics[scale = 0.26]{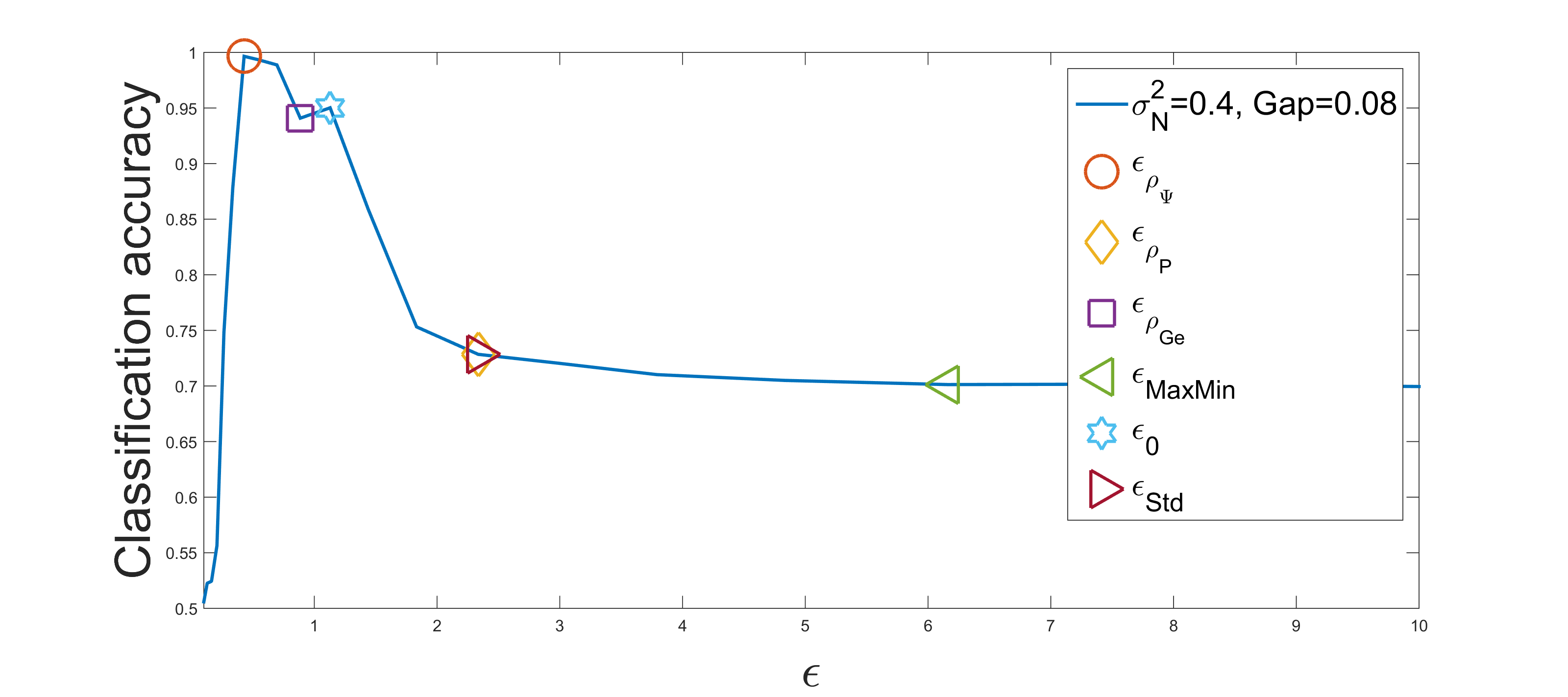}
	\caption{Accuracy of classification in the spiral artificial dataset for different values of the gap parameter $G$. The data is generated based on Eq.\ \ref{eq:Spiral}. The proposed scales ($\epsilon_{{\Psi}},\epsilon_{Ge},\epsilon_{P}$) and existing methods ($\epsilon_0,\epsilon_{\text{MaxMin}},\epsilon_{std}$) are annotated on the plots.}
	\label{fig:Spiral3}
\end{figure}
\subsubsection{{Classification of Handwritten Digits}}
In the following experiment, we use the dataset from the UCI machine learning repository \cite{Lichman:2013}. The dataset consists of $2000$ data points describing $200$ instances of each digit from $0$ to $9$, extracted from a collection of Dutch utility maps. The dataset consists of multiple features of different dimensions. We use a concatenation of the Zerkine moment (ZER), morphological (MOR), profile correlations
(FAC) and the Karhunen-lo\'{e}ve coefficients (KAR) as our features space.

We compute the proposed ratios $\rho_P$ and $\rho_{\Psi}$ for various values of $\epsilon$, and estimate the optimal scale based on Eqs.\ (\ref{eq:rhopsi}), (\ref{eq:rhop}). We evaluate the extracted embedding using $20$-fold cross validation ($5\%$ left out as a test set). The classification is done by applying k-NN (with $k=1$) in the $d$-dimensional embedding. In Fig.\ \ref{fig:Mnist}, we present the classification results and the proposed optimal scales $\epsilon$ for classification. Our proposed scale concurs with the scale that provides maximal classification rate. 
\begin{figure}[H]
	\centering
	\includegraphics[scale = 0.25]{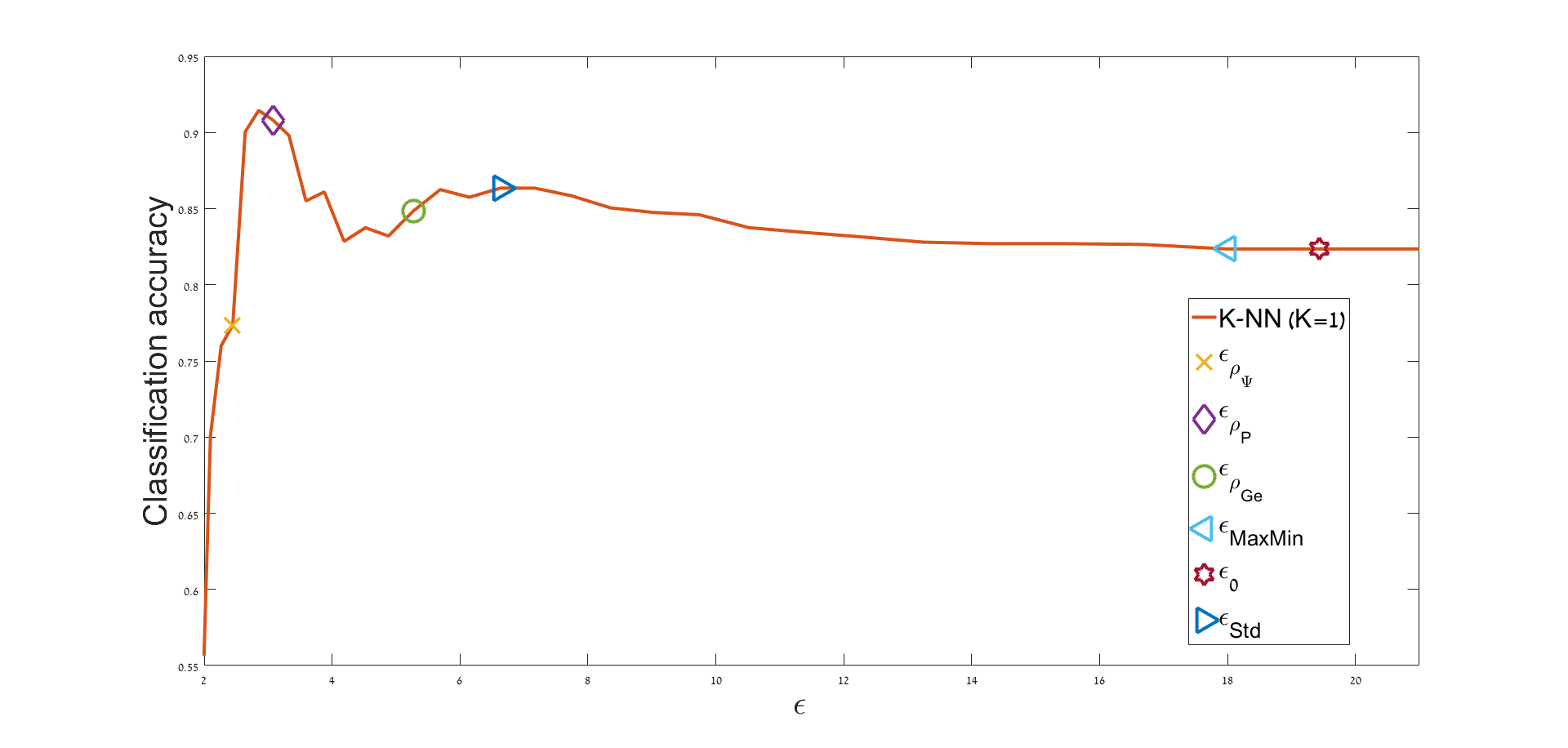}
	\caption{Accuracy of classification in the multiple features dataset. k-NN ($k=1$) is applied in a $d=4$ dimensional diffusion based representation. The proposed scales ($\epsilon_{{\Psi}},\epsilon_{Ge},\epsilon_{P}$) and existing methods ($\epsilon_0,\epsilon_{\text{MaxMin}},\epsilon_{std}$) are annotated on the plots.}
	\label{fig:Mnist}
\end{figure}

\section{Application: Learning Seismic Parameters}
\label{sec:application}
In this section we demonstrate the capabilities of the proposed approach for extracting meaningful parameters from raw seismic recordings.
Extracting reliable seismic parameters is a challenging task. Such parameters could help discriminate earthquakes from explosions, moreover, they can enable automatic monitoring of nuclear experiments. Traditional methods such as \cite{rodgers1997comparison,Blandford} use signal processing to try to analyze seismic recordings. More recent methods, such as \cite{Kortstrom,ruano2014seismic,Ohrnberger, Beyreuther,Hammer,DelPezzo,Tiira} use machine learning to construct a classifier for a variety of seismic events. Here, we extend our result from \cite{rabin2016earthquake,lindenbaum2017Seismic}, in which we have demonstrated the strength of DM for extracting seismic parameters. Our proposed method performs a vector scaling for manifold learning. Thus, essentially if the data lies on a manifold, our scaling combined with DM will extract the manifold from high-dimensional seismic recordings. Moreover, it will provide a natural feature selection procedure, thus if some features are corrupt, the proposed scaling may be able to reduce their influence. 

As a test case, we use a dataset from \cite{rabin2016earthquake,lindenbaum2017Seismic}, which was recorded in Israel and Jordan between $2005$-$2015$. All recordings were collected in HRFI (Harif) station located in the south of Israel. The station collects three signals from north (N), east (E) and vertical (Z). Each signal is sampled using a broadband seismometer at $40$[Hz] and consists of $10,\!000$ samples.

\subsection{Feature Extraction}
Seismic events usually generate two waves, primary-waves (P) and secondary waves (S). The primary wave arrives directly from the source of the event to the recorder, while the secondary wave is a shear wave and thus arrives at some time delay. Both waves pass through different material thus have different spectral properties. This motivates the use of a time-frequency representation as the feature vector for each seismic event. The time-frequency representation used in this study is a Sonogram \cite{joswig1990pattern}, which offers computation simplicity while retaining the sufficient spectral resolution for the task in hand. The Sonogram is basically a spectrogram, renormalized and rearranged in a logarithmic manner. Given a seismic signal $\myvec{z}(n)$, the Sonogram is extracted using the following steps:
\begin{enumerate}

\item {Compute a discrete-time Short Term Fourier Transform:
\begin{equation}
\label{eq:STFT}
\myvec{\widehat{Z}}(f,t) = \sum\limits_{n = 1}^{N} {w(n-\ell(t)) \cdot {z}(n)}  \cdot {e^{ - j2\pi f n}} ,
\end{equation}
where $w(n-\ell(t))$ is a window function of length $N_0=512$, with a shift value of $\ell(t)=\lfloor (1-s)\cdot N_0\cdot t\rfloor$ time steps. We use overlap of $s=0.85$ and compute for $N_0$ frequency bins, such as the values of $f$ are spread uniformly on a logaritmic scale. }
\item{Normalize the energy by the number of frequency bins
\begin{equation}
\label{eq:NSPEC}
\myvec{\widetilde{{Z}}}(f,t) = \frac{|\myvec{\widehat{Z}}(f,t)|^2 }{N_0}.
\end{equation} }
\item{Reshape the time frequency representation into a vector, by concatenating columns. The resulted vector representation for a signal $\myvec{z}(n)$ is denoted by $\myvec{x}$.}


\end{enumerate}
These steps are applied to each of the channels separately. This results in three sets $\myvec{X}_E$ for the east channel, $\myvec{X}_N$ for the north channel and $\myvec{X}_Z$ for the horizontal. Examples for seismic recording of an explosion and of an earthquake are presented on Fig.\ \ref{fig:seis_exp} and \ref{fig:seis_eq}. Examples of corresponding Sonograms are presented in Fig.\ \ref{fig:seis_exp} and \ref{fig:seis_eq}.
\begin{figure}[H]
	
	\centering
	\subfigure[Seismic recording of an explosion ] {\label{fig:seis_exp} \includegraphics[width=0.45\textwidth] {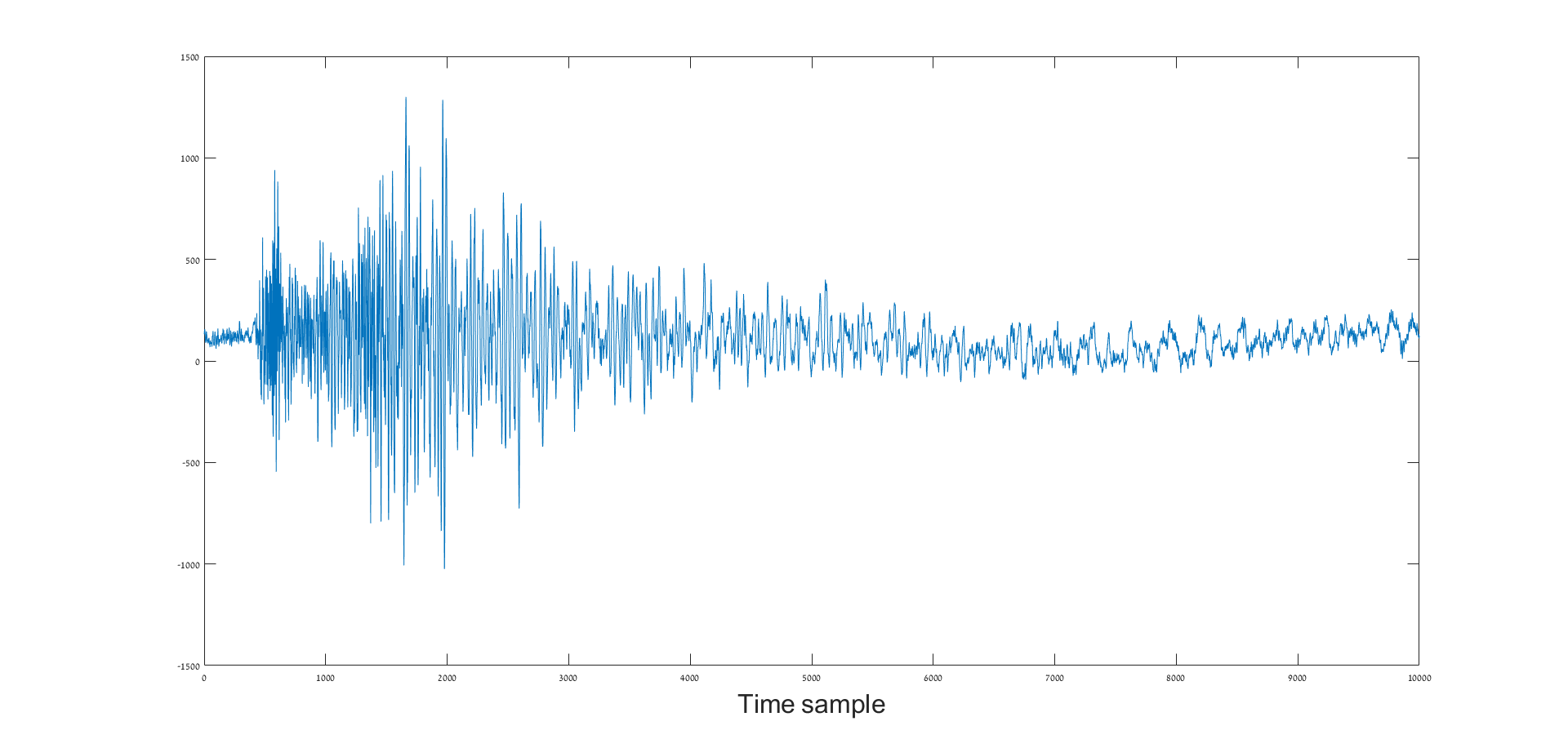}}
	\subfigure[Seismic recording of an earthquake] {\label{fig:seis_eq} \includegraphics[width=0.45\textwidth] {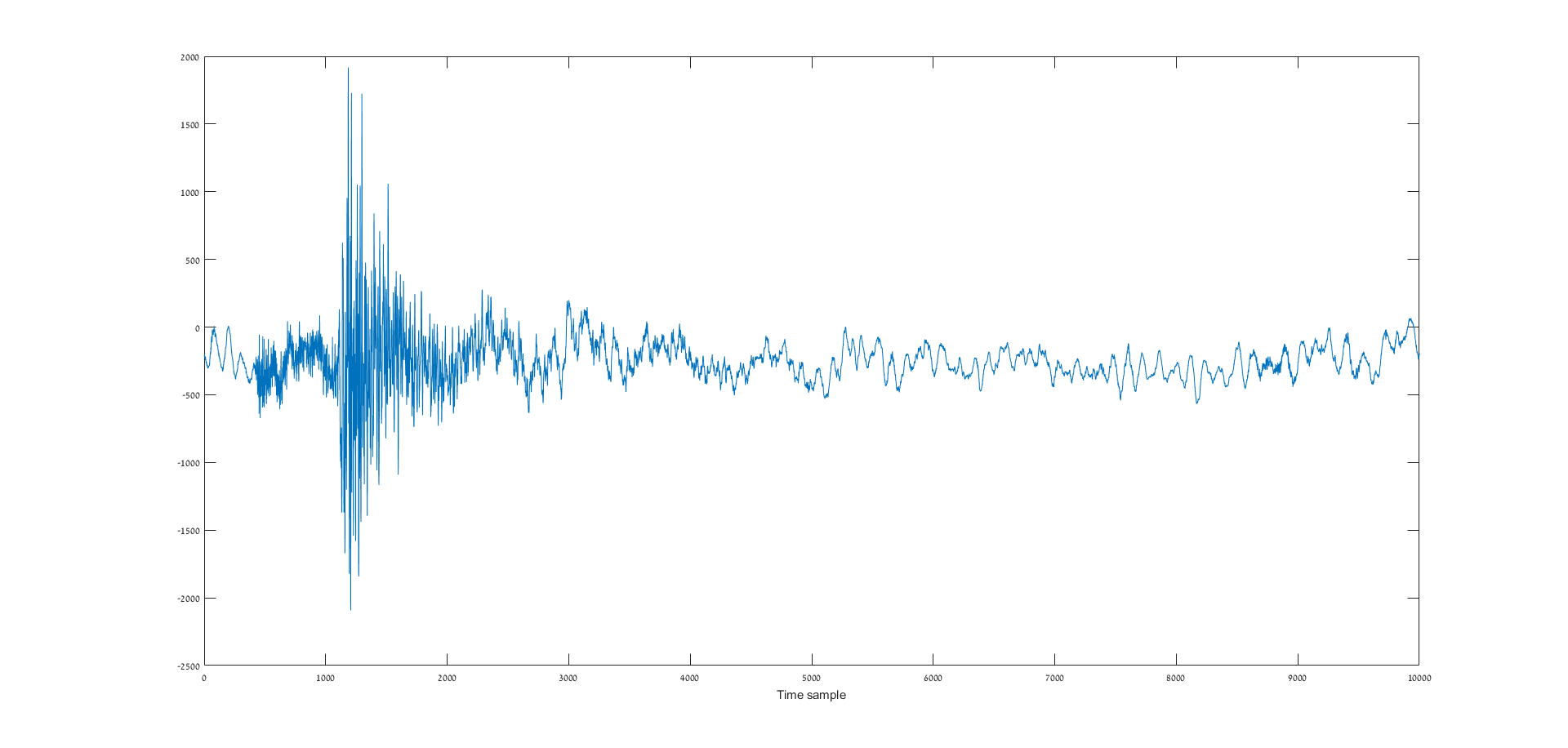}}
	\subfigure[Sonogram of an explosion] {\label{fig:sono_exp} \includegraphics[width=0.45\textwidth] {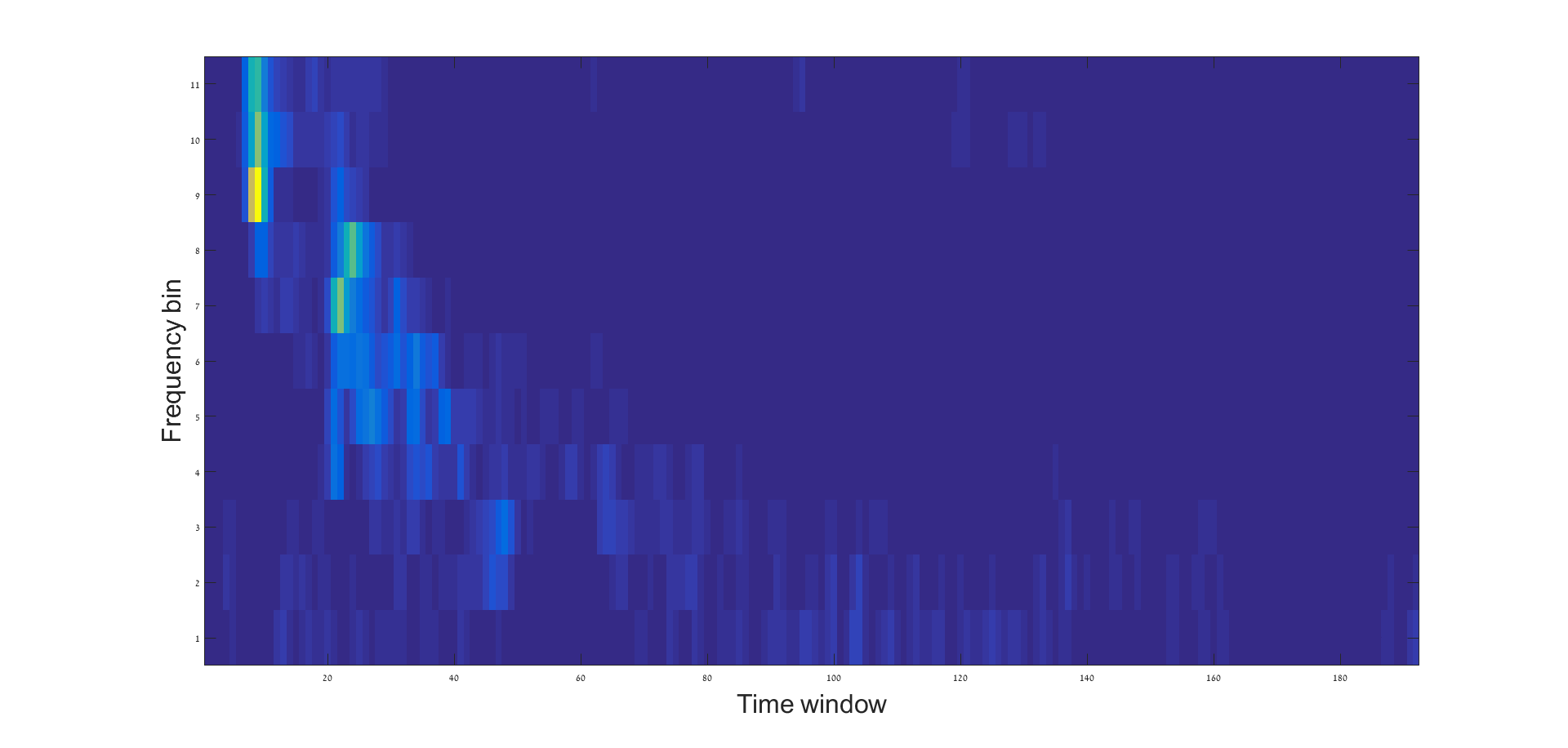}}
		\subfigure[Sonogram of an earthquake] {\label{fig:sono_eq} \includegraphics[width=0.45\textwidth] {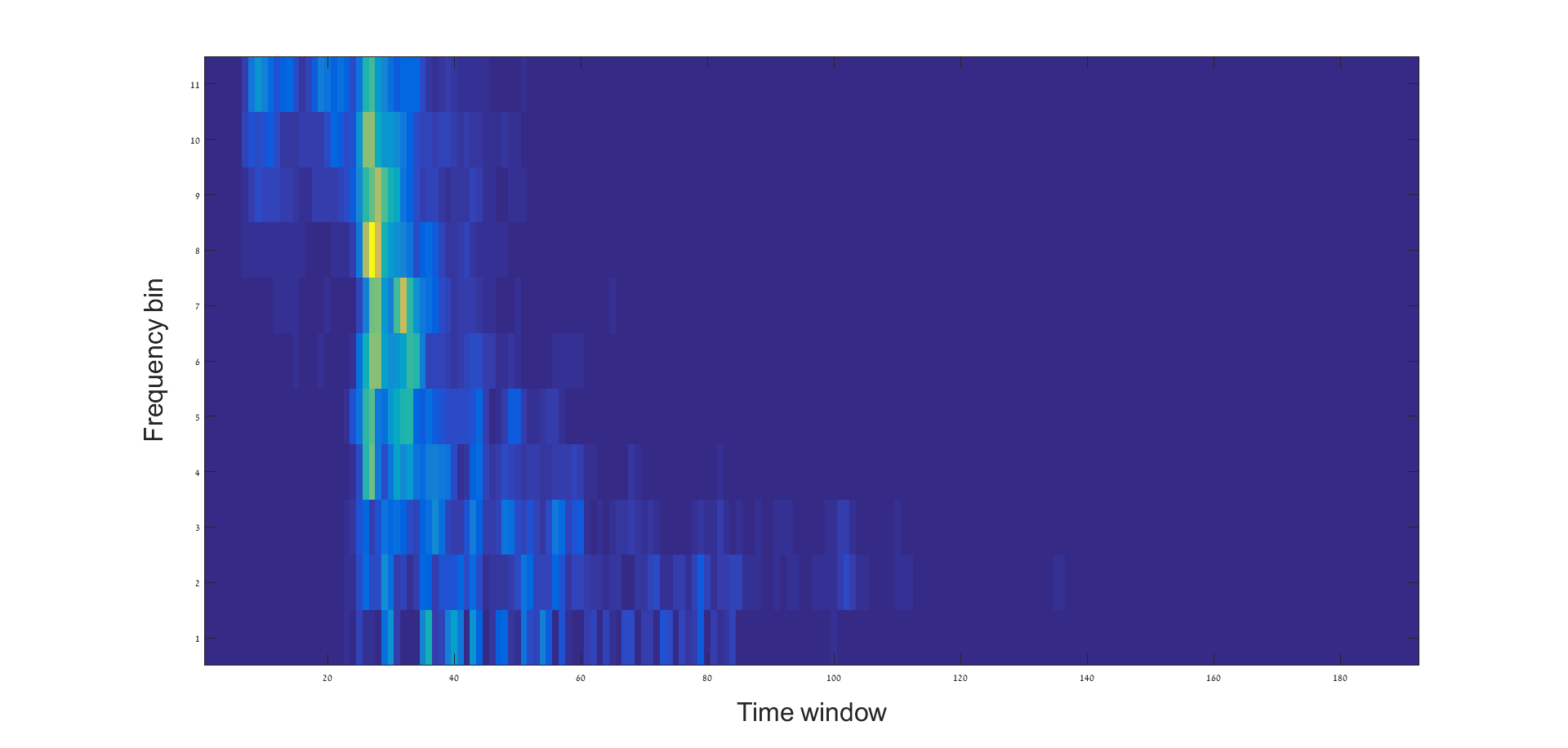}}
	\caption{Top: Example of a raw signal recorded from (a) an explosion  and (b) earthquake. Bottom: The Sonogram matrix extracted from (c) an explosion and (d) earthquake.}
	\label{fig:SonoExample1}
\end{figure}
\subsection{Seismic Manifold Learning}
To evaluate the proposed scaling for manifold learning, we use a subset of the seismic recording with 352 quarry blasts. The explosions have occurred at 4 known quarries surrounding the recording station HRFI. Our study in \cite{rabin2016earthquake,lindenbaum2017Seismic}, has demonstrated that most of the variability of quarry blasts stems from the source location of each quarry, therefor, we assume that the 352 blasts lie on some low-dimensional manifold. Where the parameters of the manifold should correlate with location parameters. Our approach for setting the scale parameter provides a natural feature selection procedure. To evaluate the capabilities of this procedure we ``destroy'' the information in some of the features. We do this by applying a deformation function to one channel out of the three seismometer recordings. We define the input for Algorithm \ref{alg:solv_global} as 
\begin{equation}
\myvec{X}=[\myvec{X}_N,\myvec{X}_E,g(\myvec{X}_Z)],
\end{equation} where $g(\cdot)$ is an element-wise deformation function. In the first test case the deformation function is defined by $g(y)=y^{0.1}$.
Then, we apply DM with various scaling schemes and examine the extracted representation. In Fig.\ref{fig:seismicloc}  the two leading DM coordinates of different scaling methods are presented.  

\begin{figure}[H]
	\centering
\subfigure[ ]{\label{fig:Seis_Std} \includegraphics[scale = 0.15]{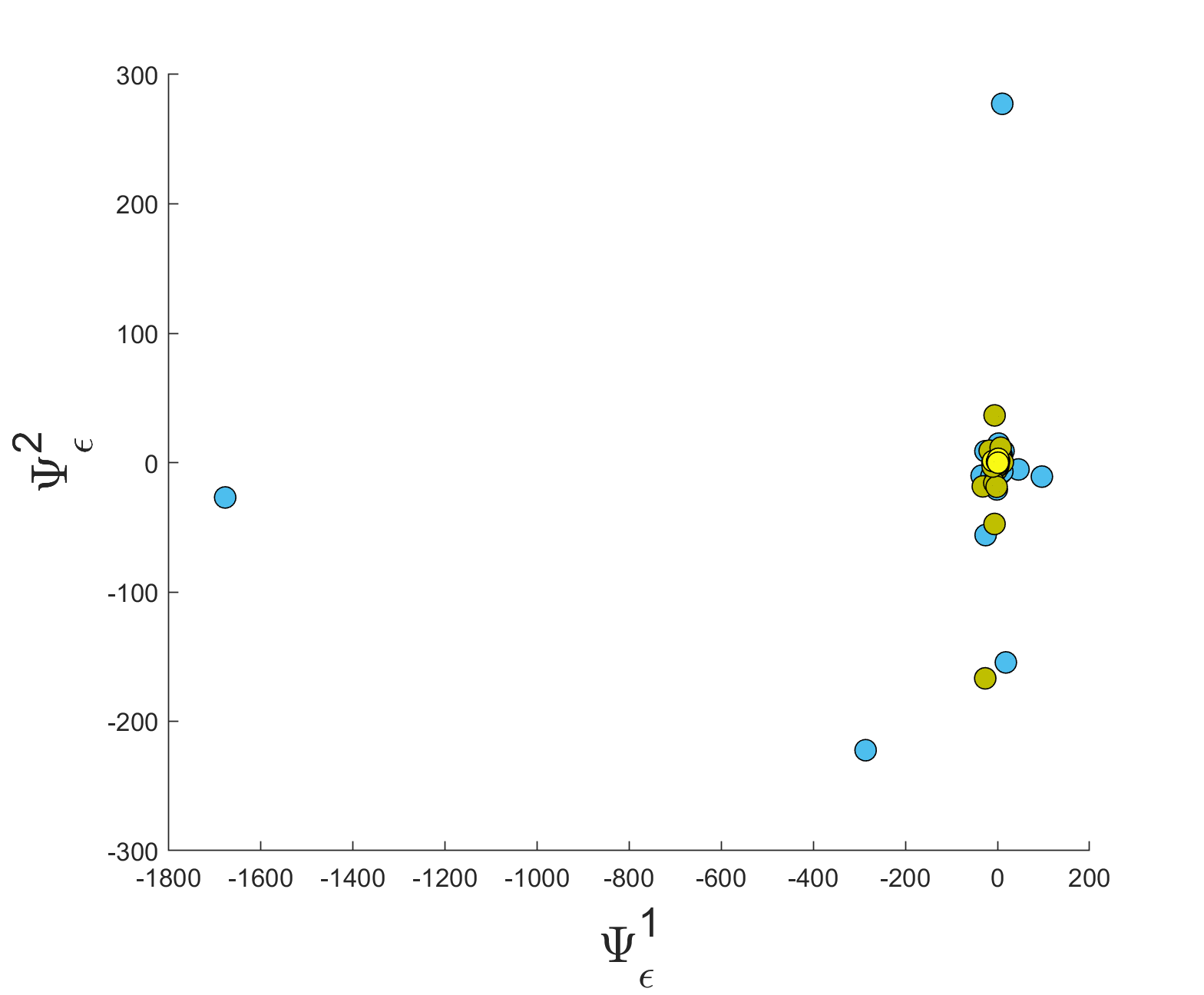}}
\subfigure[ ]{\label{fig:Seis_Singer} \includegraphics[scale = 0.15]{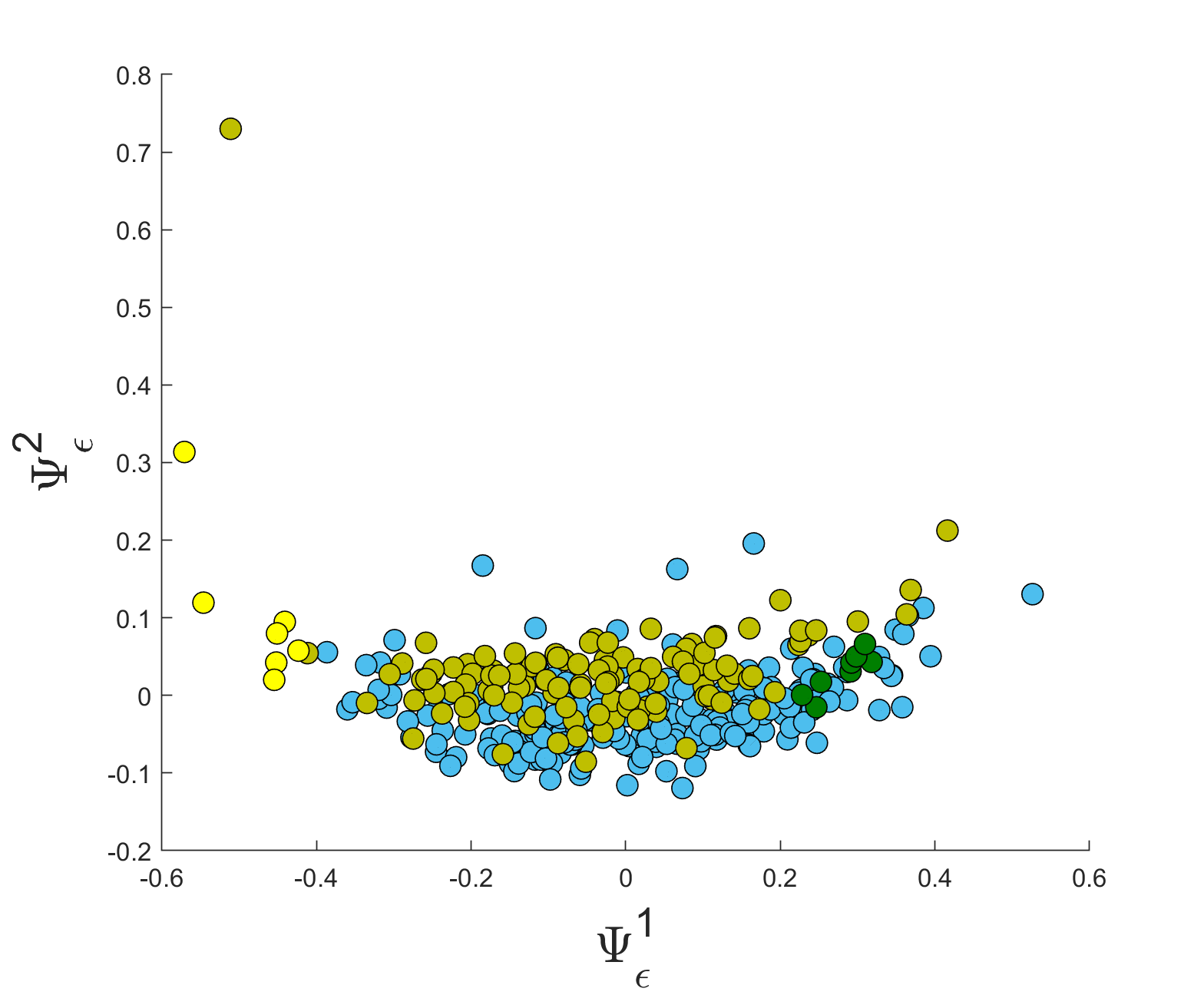}}
\subfigure[ ]{\label{fig:Seis_MinMax} \includegraphics[scale = 0.15]{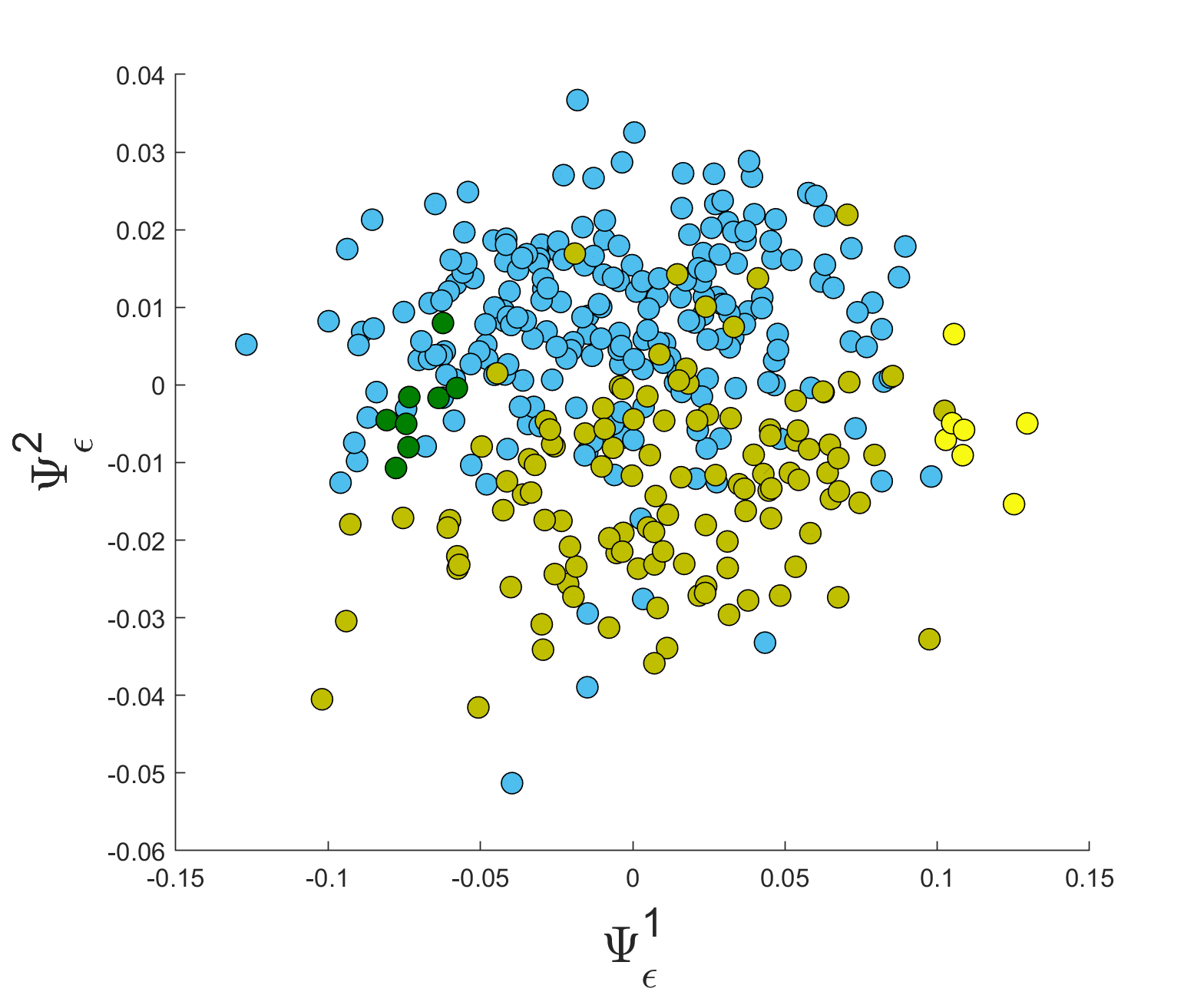}}
\subfigure[ ]{\label{fig:Seis_Prop} \includegraphics[scale = 0.15]{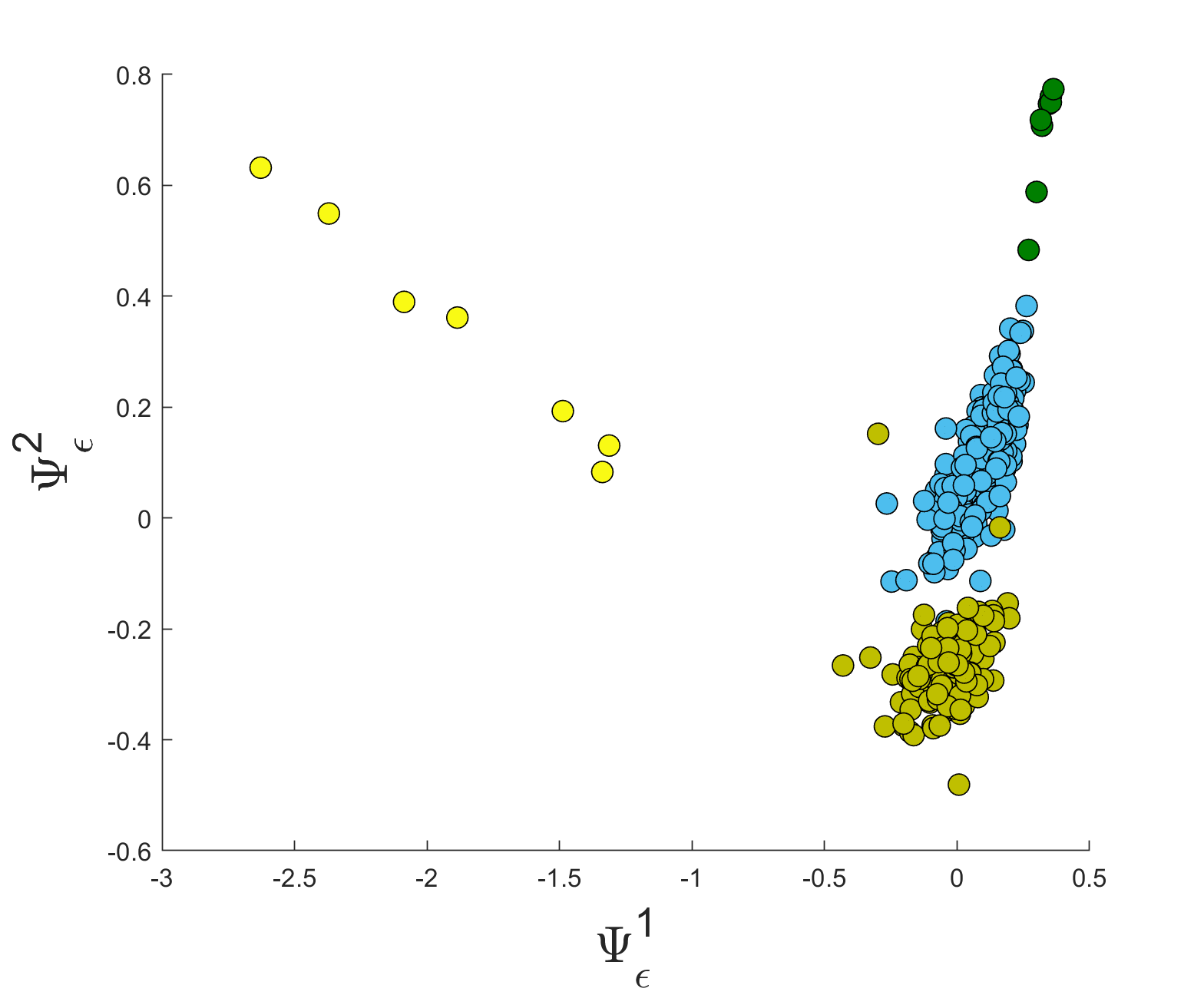}}
	\caption{The two leading DM coordinates of the 352 quarry blasts, colored by source quarry cluster. Scaling method based on: (a) The standard deviation of the data. (b) Singer's \cite{Singer} approach $\epsilon_0$ (detailed in Algorithm \ref{alg:Singer}. (c) The max-min methods $\epsilon_{MaxMin}$ (Eq.\ \ref{eq:MaxMin}). (d) Proposed scaling for manifold learning (detailed in Algorithm \ref{alg:solv_global}). }
	\label{fig:seismicloc}
\end{figure}
The quarry cluster separation is clearly evident in Fig.\ \ref{fig:Seis_Prop}. To further evaluate how well the low-dimensional representation correlates with the source location we use a list of source locations. A list  with the explosions locations is provided to us based on manual calculations, performed by an analyst by considering the phase difference between the signals' arrival times to different stations. We note that this estimation is accurate up to a few kilometers. A map of the location estimates colored by source quarry is presented in Fig.\ \ref{fig:map}. Then, we apply Canonical Correlation Analysis (CCA) to find the most correlated representation. The transformed representations $\myvec{U}$ and $\myvec{V}$ are presented in Fig.\ \ref{fig:CCAU} and \ref{fig:CCAV} respectively. The two correlation coefficients between coordinates of $\myvec{U}$ and $\myvec{V}$ are $0.88$ and $0.72$.  

\begin{figure}[H]
	\centering
	\subfigure[ ]{\label{fig:map}\includegraphics[scale = 0.72]{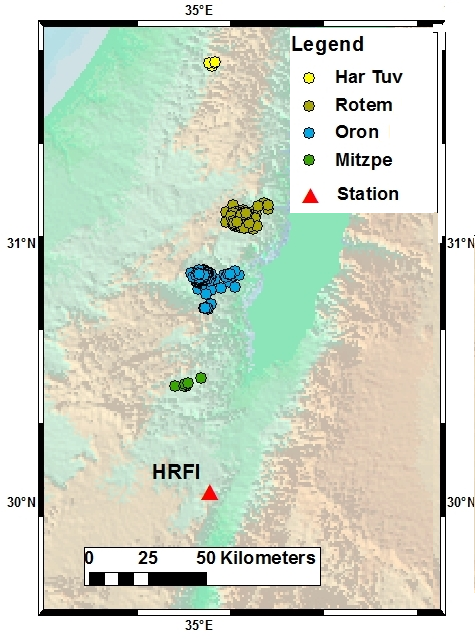}}
    \subfigure[ ]{\label{fig:CCAU}\includegraphics[scale = 0.14]{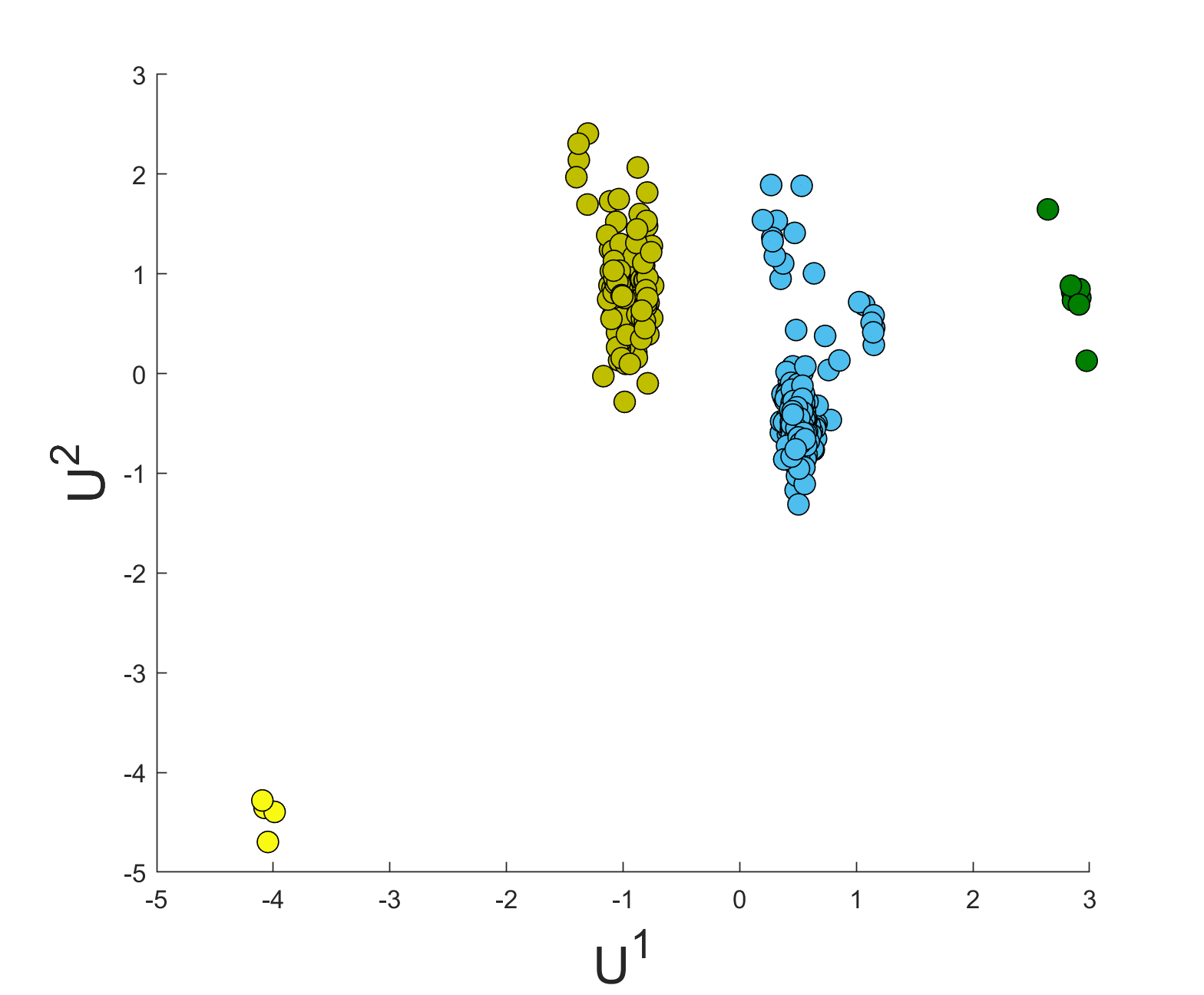}}
    \subfigure[ ]{\label{fig:CCAV}\includegraphics[scale = 0.14]{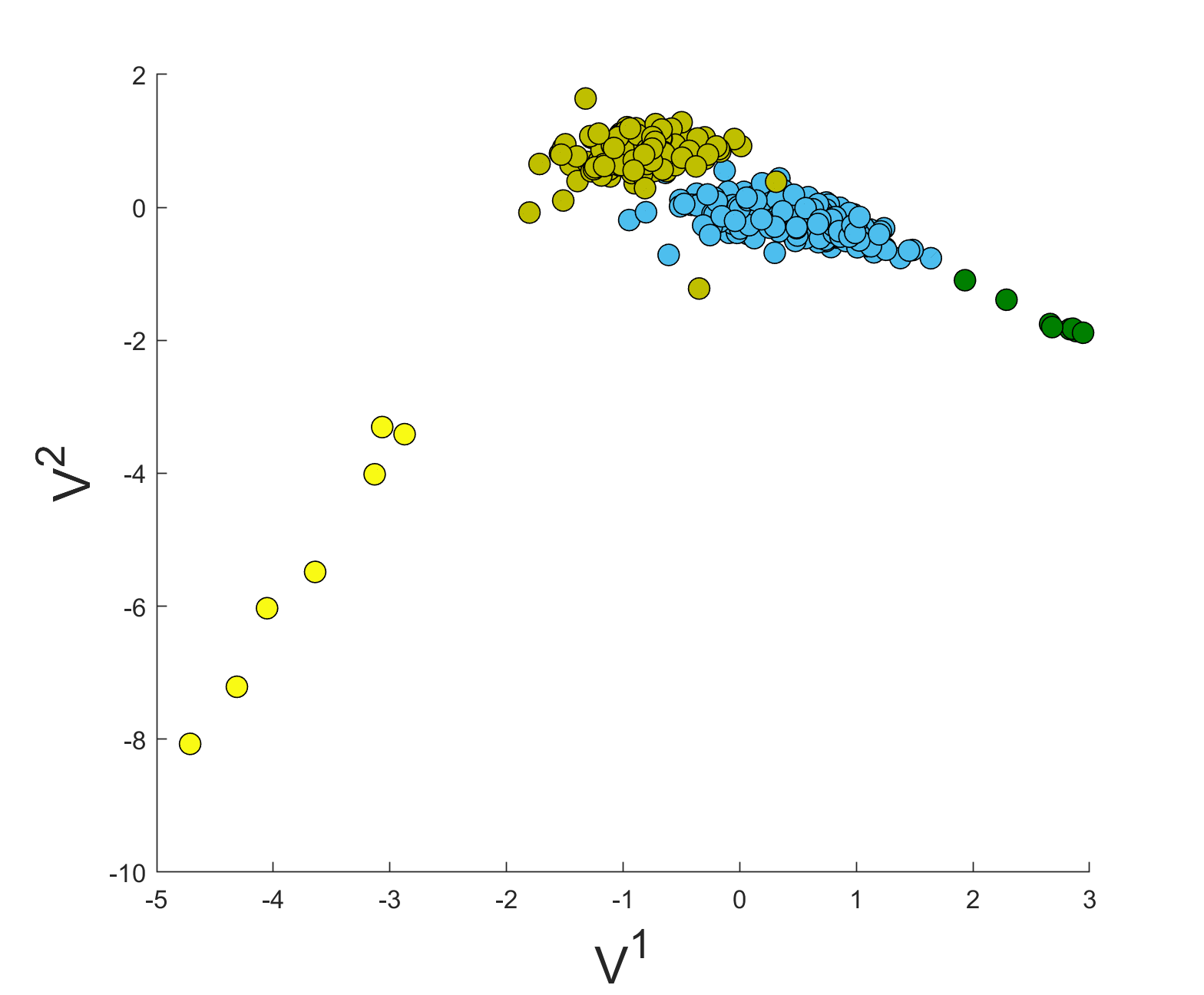}}

	\caption{(a) A map with source locations of 352 explosions. Points are colored by quarry cluster. (b) A CCA based representation of the latitude and longitude of the explosions. (c) A CCA based representation of the two leading DM coordinates extracted based on the proposed scaling (appear in Fig.\ \ref{fig:Seis_Prop}).  }
	\label{fig:CCA}
\end{figure}
In the second test case, we use additive Gaussian noise to ``degrade'' the signal. We use $g(y)=y+n_1$ as the defoemation function, where $n_1$ is drawn from a zero-mean Gaussian distribution with variance of $\sigma^2_N$. We estimate the scaling $\epsilon$ based on the proposed and alternative methods. Then, we apply CCA to the two leading DM coordinated and the estimated source locations. The top correlation coefficients for various values of $\sigma^2_N$ are presented in Fig.\ \ref{fig:CorrGauss}. Both the max-min method $\epsilon_{MaxMin}$ and Singer's $\epsilon_0$\cite{Singer} scheme seem to break at the same noise level. The standard deviation approach is robust to the noise level, this is because it essentially performs whitening of the data. However, this also obscures some of the information content when the noise is of low power. The proposed approach seems to outperform all alternative schemes for this test case.
\begin{figure}[H]
	\centering
	\includegraphics[scale = 0.25]{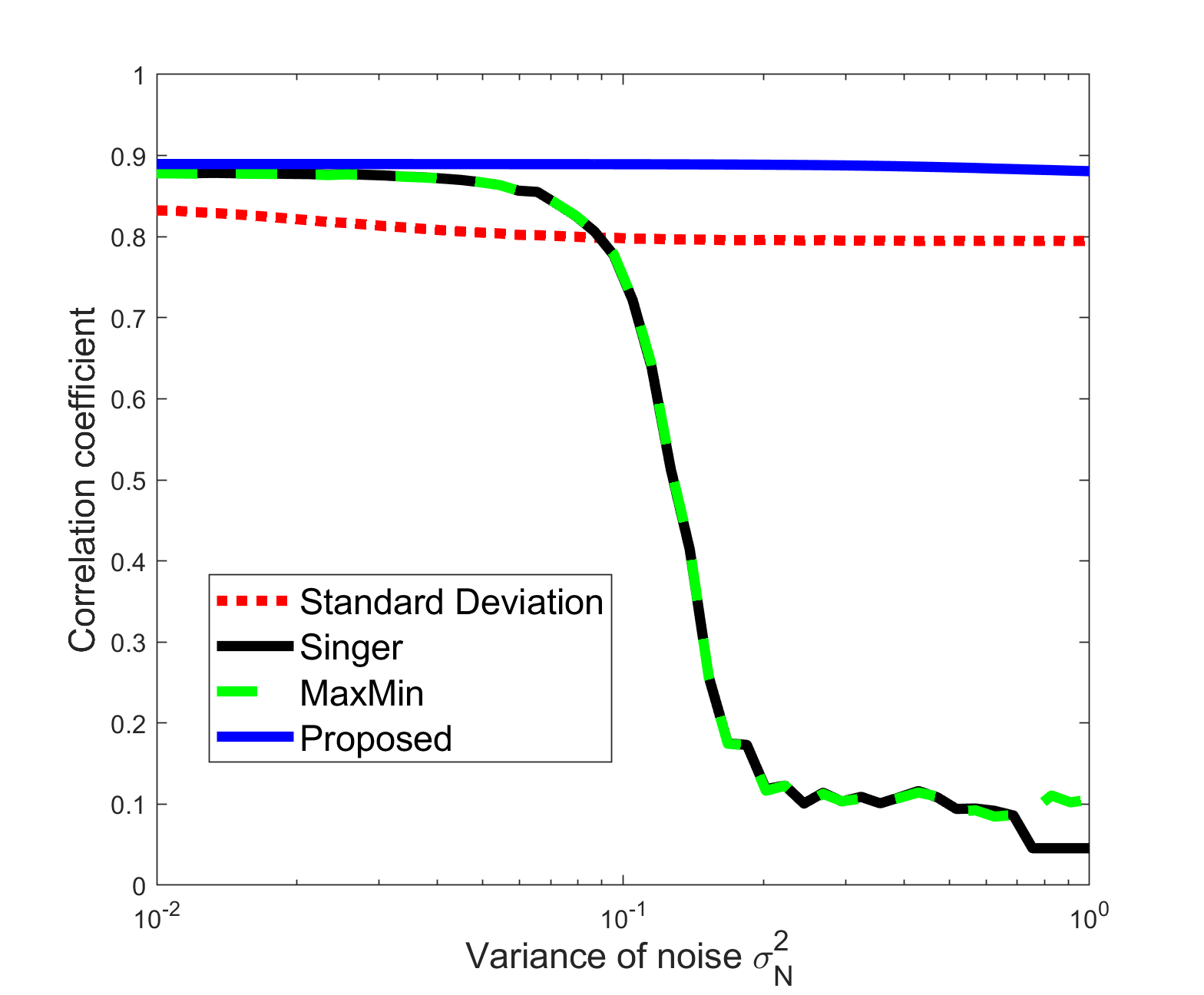}

	\caption{Highest correlation coefficient between the DM representation extracted using various scaling schemes. The x-axis corresponds to the variance of the additive Gaussian noise.}
	\label{fig:CorrGauss}
\end{figure}

\subsection{Classification of seismic events}
Automatic classification of seismic events is useful as it may reduce false alarm warnings on one hand, and enable monitoring nuclear events on the other hand. To evaluate the proposed scaling for classification of seismic events, we use a set with $46$ earthquakes and $62$ explosions all of which were recorded in Israel. A low-dimensional mapping is extracted by using DM with various values of $\epsilon$, and binary classification was applied using k-NN ($k=5$) in a leave-one-out fashion. The accuracy of the classification for each value of $\epsilon$ is presented in Fig.\ \ref{fig:SonoClass}. The estimated values of $\epsilon_{Ge}$, $\epsilon_{\rho_{P}}$ and $\epsilon_{\rho_{\Psi}}$ were annotated. It is evident that for classification the estimated values are indeed close to the optimal values, although they do not fully coincide. Nevertheless, they all achieve high classification accuracy.

\begin{figure}[H]
	\centering
	\includegraphics[scale = 0.46]{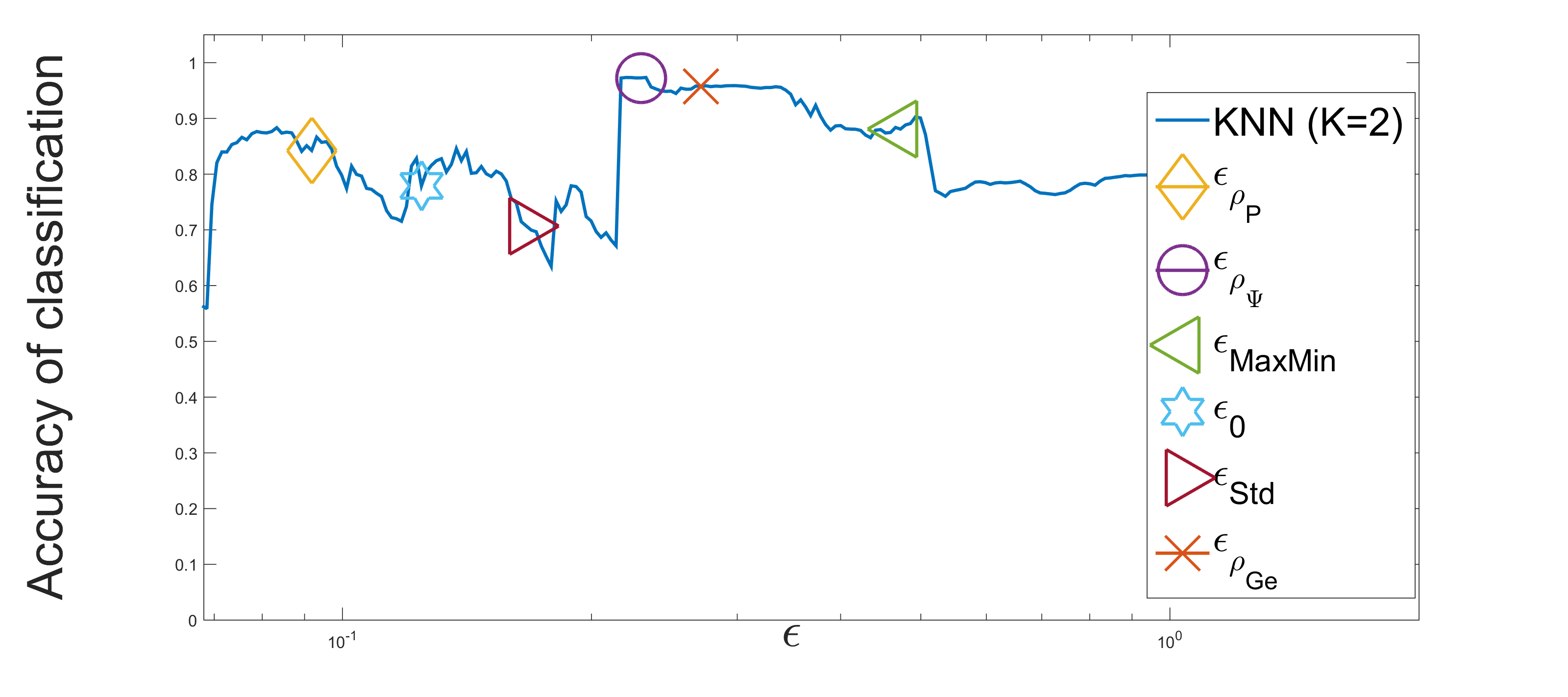}

	\caption{Classification accuracy vs. value of $\epsilon$. The proposed scales ($\epsilon_{{\Psi}},\epsilon_{Ge},\epsilon_{P}$) and existing methods ($\epsilon_0,\epsilon_{\text{MaxMin}},\epsilon_{std}$) are annotated on the plots. }
	\label{fig:SonoClass}
\end{figure}

\section{Conclusions}
\label{sec:Conclusions}
The scaling parameter $\epsilon$ of the widely used Gaussian kernel is often crucial for machine learning algorithms. As happens in many tasks in the field, there does not seem to be one global scheme that is optimal for all applications. For this reason, we propose two new frameworks for setting a kernel's scale parameter tailored for two specific tasks. The first approach is useful when the high-dimensional data points lie on some lower dimensional manifold. By exploiting the properties of the Gaussian kernel, we extract a vectorized scaling factor that provides a natural feature selection procedure. Theoretical justification and simulations on artificial data demonstrate the strength of the scheme over alternatives. The second approach could improve the performance of a wide range of kernel based classifiers. The capabilities of the proposed methods are demonstrated using artificial and real datasets. Finally, we present an application for the proposed approach that helps learn meaningful seismic parameters in an automated manner. In the future, we intend to generalize the approach for the multi-view setting recently studied in \cite{lindenbaum2015multiview,salhov2016multi,lederman2014common}.

\section{Appendix}
{\bf{Dimensionality from Angle and Norm Concentration (DANCo) \cite{danco}}}
DANCo is a recent method for estimating the intrinsic dimension based on high-dimensional measurements. The estimate is based on the following steps:
\begin{enumerate}
\item For each point $\myvec{x}_i,i=1,...,N$, find the set of $\ell+1$ nearest neighbors ${\cal{S}}^{\ell+1}(\myvec{x}_i)=\{\myvec{x}_{s_j} \}^{\ell+1}_{j=1}$. Denote the farthest neighbor of $\myvec{x}_i$ by $\widehat{{S}}(\myvec{x}_i)$. The value of $\ell$ depends on the density of the dataset, and is usually not higher than $10$. 
\item{Calculate the normalized closest distance for $\myvec{x}_i$ as $\rho(\myvec{x}_i)=\underset{\myvec{x}_j\in {\cal{S}^{\ell+1}}(\myvec{x}_i)}{\min} \frac{||\myvec{x}_i-\myvec{x}_j||}{||\myvec{x}_i-\widehat{{S}}(\myvec{x}_i)||} $. }
\item{Use Maximum Likelihood (ML) to estimate $\hat{d}_{ML}= \arg\max {\cal{L}}(d)$, where the log likelihood is
	\begin{equation}{\cal{L}}(d)= N \log \ell d +(d-1)\sum_{\myvec{x}_i\in \mymat{X}}\log \rho(\myvec{x}_i)+ (\ell-1)\sum_{\myvec{x}_i\in \mymat{X}} \log(1-\rho^d(\myvec{x}_i)) .\end{equation}	
}
\item{For each point $\myvec{x}_i$, find the $\ell$ nearest neighbors and center them relative to $\myvec{x}_i$. The translated points are denoted as $\tilde{\myvec{x}}_{s_j}\defeq\myvec{x}_{s_j}-\myvec{x}_i$. The set of $\ell$ nearest neighbors for point $\myvec{x}_i$ is denoted by $\tilde{\cal{S}}^{\ell}(\myvec{x}_i)= \{ \tilde{\myvec{x}}_{s_j} \}^{\ell}_{j=1}$. The distribution model is explained in \cite{camastra2003data}.}
\item{Calculate the ${\ell}\choose{2}$} angles for all pairs of vectors within $\tilde{\cal{S}}^{\ell}(\myvec{x}_i)$. The angles are calculated using
\begin{equation}\label{eq:angles}\theta(\myvec{x}_{s_j},\myvec{x}_{s_m})= \arccos \frac{\myvec{\tilde{x}}_{s_j} \cdot \myvec{\tilde{x}}_{s_m}  }{\myvec{||\tilde{x}}_{s_j}|||| \myvec{\tilde{x}}_{s_m}||}. \end{equation} For each point $\myvec{x}_i$ concatenate all angles from Eq.\ \ref{eq:angles} into a vector  $\myvec{\bar{\theta}}_i$ and the set of vectors by $\myvec{\widehat{\theta}}\defeq \{\myvec{\bar{\theta}}_i\}^N_{i=1}$. 
\item{Estimate the set of parameters $\myvec{\hat{\nu}}=\{\hat{\nu}_i\}^N_{i=1} $ and $\hat{\myvec{\tau}}=\{\hat{\tau}_i\}^N_{i=1}$ based on a ML estimation using the von Mises (VM) distribution with respect to $\myvec{X}$. The VM pdf describes the probability for $\theta$ given the mean direction $\nu$ and the concentration parameter $\tau\geq0$. The VM pdf, as well as the ML solution, are presented in \cite{danco}. The means of $\myvec{\hat{\nu}}$ and $\hat{\myvec{\tau}}$ are denoted as $\hat{\mu}_\nu \text{ and } \hat{\mu}_\tau,$ respectively.}  
\item{For each hypothesis of $d=1,...,D$, draw a set of $N$ data points $\myvec{Y}^d= \{ \myvec{y}^d_i\}^N_{i=1}$ from a $d$-dimensional  unit hypersphere.  }
\item{Repeat steps 1-6 for the artificial dataset $\myvec{Y}^d$. Denote the maximum likelihood estimated set of parameters as $\tilde{d}_{ML},\tilde{\myvec{\nu}},\tilde{\myvec{\tau}},\tilde{\mu}_{\nu},\tilde{\mu}_{\tau}$.}
\item{Obtain $\hat{d}$ by minimizing the Kullback-Leibler (KL) divergence between the distribution based on $\myvec{X}$ and $\myvec{Y}^d$. The estimator takes the following form
	\[ \hat{d}=\underset{d=1,...,D}{\arg \min} \myvec{{KL}}(g(\cdot;\ell,\hat{d}_{ML}),g(\cdot;\ell,\tilde{d}_{ML})) +\myvec{{KL}}(q(\cdot;\hat{\mu}_{\nu},\hat{\mu}_{\tau}),q(\cdot;\tilde{\mu}_{\nu},\tilde{\mu}_{\tau})),\] where $g$ is the pdf of the normalized distances and $q$ is the VM pdf. Both $g$ and $q$ are described in \cite{danco}.
	
}
\end{enumerate}

\section*{Acknowledgment}

This research was partially supported by the US-Israel Binational Science Foundation (BSF 2012282), Blavatnik Computer Science Research Fund , Blavatink ICRC Funds and Pazy Foundation.



\noindent {\bf Ofir Lindenbaum} {recived B.Sc. degrees in 2010, in electrical engineering and in physics (both summa cum laude), from the Technion --— Israel Institute of Technology. In 2018, he received his Ph.D in electrical engineering at the School of Electrical Engineering at Tel-Aviv University. His areas of interest include machine learning, applied and computational harmonic analysis, musical signals analysis.

	{\bf Moshe Salhov} received his B.Sc. (magna cum laude) in electrical engineering from Ben-Gurion University, Beer-Sheva, Israel and  M.Sc. in electrical engineering from the Technion, Institute of Technology, Haifa, Israel, in 1998 and 2006, respectively.
	In 2017, he received his Ph.D  in computer science from Tel Aviv University, Tel Aviv, Israel. His research interests include wide area of machine learning, deep learning, big data analysis, scientific computing, optimization, computer vision and wireless communication.
	
	{\bf Arie Yeredor} received the B.Sc. (summa cum laude) and Ph.D. degrees in electrical engineering from Tel-Aviv University, where he is currently an Associate Professor. His research and teaching areas are in statistical signal processing and estimation theory. He served as an Associate Editor and Guest Editor for several journals, and currently serves as a Senior Area Editor for IEEE TRANSACTIONS ON SIGNAL PROCESSING.

	{\bf Amir Averbuch} received the B.Sc and M.Sc degrees in Mathematics from the Hebrew University in Jerusalem, Israel in 1971 and 1975, respectively. He received the Ph.D degree in Computer Science from Columbia University, New York, in 1983. In 1987, he joined the School of Mathematical Sciences (later split to School of Computer Science), Tel Aviv University, where he is now Professor emerita of Computer Science.}
	
\end{document}